\documentclass[11pt,fullpage,letterpaper]{article} 

\usepackage{natbib}

\setcitestyle{authoryear,round,citesep={;},aysep={,},yysep={;}}
\usepackage[english]{babel}

\usepackage{fancyhdr}
\usepackage[utf8]{inputenc} 
\usepackage[T1]{fontenc}    

\usepackage{titletoc}
\usepackage[toc, page, header]{appendix} 

\usepackage[colorlinks=true, linkcolor=blue, citecolor=blue,urlcolor=black]{hyperref}
\usepackage{url}            
\usepackage{booktabs}       
\usepackage{amsfonts}       
\usepackage{nicefrac}       
\usepackage{microtype}      
\usepackage{xcolor}         

\usepackage{indentfirst}

\title{Refined Sample Complexity for Markov Games with \\ Independent Linear Function Approximation}

\author{%
    Yan Dai~\thanks{IIIS, Tsinghua University. Email: \texttt{yan-dai20@mails.tsinghua.edu.cn}.}\and Qiwen Cui~\thanks{University of Washington. Email: \texttt{qwcui@cs.washington.edu}.}\and Simon S. Du~\thanks{University of Washington. Email: \texttt{ssdu@cs.washington.edu}.}
}
\date{}

\usepackage{amsmath}
\usepackage{amsthm}
\usepackage{amssymb}
\usepackage{mathtools}

\usepackage{cleveref} 
\newtheorem{theorem}{Theorem}[section]
\newtheorem{lemma}[theorem]{Lemma}

\newtheorem{assumption}[theorem]{Assumption}
\newtheorem{corollary}[theorem]{Corollary}

\usepackage{algorithm}
\usepackage{algorithmic}

\numberwithin{equation}{section}
\newcommand{\mathsc}[1]{{\normalfont\textsc{#1}}}
\newcommand{\E}{\operatornamewithlimits{\mathbb{E}}}

\newcommand{\argmin}{\operatornamewithlimits{\mathrm{argmin}}}

\renewcommand{\O}{\operatorname{\mathcal O}}
\newcommand{\Otil}{\operatorname{\tilde{\mathcal O}}}
\renewcommand{\P}{\operatorname{\mathbb P}}

\newcommand{\trans}{\mathsf{T}}
\newcommand{\mS}{\mathcal{S}}
\newcommand{\mA}{\mathcal{A}}

\newcommand{\mD}{\mathcal{D}}

\newcommand{\mI}{\mathcal{I}}

\newcommand{\Gap}{\mathsc{Gap}}
\newcommand{\Unif}{\mathrm{Unif}}
\newcommand{\poly}{\mathrm{poly}}
\newcommand{\Cov}{\mathrm{Cov}}

\renewcommand{\tilde}{\widetilde}
\renewcommand{\hat}{\widehat}
\renewcommand{\bar}{\overline}
\newcommand{\happa}{\mathfrak h}

\usepackage{nicefrac}
\usepackage{multirow}
\usepackage{footnote}
\usepackage{enumitem}
\setitemize{leftmargin=*, itemsep=2pt}
\setenumerate{leftmargin=*, itemsep=2pt}
\usepackage{comment}
\usepackage{bbm}
\usepackage{bm}
\allowdisplaybreaks

\Crefname{ALC@unique}{Line}{Lines}
\Crefname{assumption}{Assumption}{Assumptions}
\Crefformat{equation}{Eq. #2(#1)#3}
\Crefrangeformat{equation}{Eqs. #3(#1)#4 to #5(#2)#6}
\Crefmultiformat{equation}{Eqs. #2(#1)#3}{ and #2(#1)#3}{, #2(#1)#3}{ and #2(#1)#3}
\Crefrangemultiformat{equation}{Eqs. #3(#1)#4 to #5(#2)#6}{ and #3(#1)#4 to #5(#2)#6}{, #3(#1)#4 to #5(#2)#6}{ and #3(#1)#4 to #5(#2)#6}

\begin{document}
\maketitle

\begin{abstract}

Markov Games (MG) is an important model for Multi-Agent Reinforcement Learning (MARL).
It was long believed that the ``curse of multi-agents'' (\textit{i.e.}, the algorithmic performance drops exponentially with the number of agents) is unavoidable until several recent works \citep{daskalakis2023complexity,cui2023breaking,wang2023breaking}.
While these works resolved the curse of multi-agents, when the state spaces are prohibitively large and (linear) function approximations are deployed, they either had a slower convergence rate of $\O(T^{-1/4})$ or brought a polynomial dependency on the number of actions $A_{\max}$ -- which is avoidable in single-agent cases even when the loss functions can arbitrarily vary with time.
This paper first refines the \texttt{AVLPR} framework by \citet{wang2023breaking}, with an insight of designing \textit{data-dependent} (i.e., stochastic) pessimistic estimation of the sub-optimality gap, allowing a broader choice of plug-in algorithms.
When specialized to MGs with independent linear function approximations, we propose novel \textit{action-dependent bonuses} to cover occasionally extreme estimation errors. With the help of state-of-the-art techniques from the single-agent RL literature, we give the first algorithm that tackles the curse of multi-agents, attains the optimal $\O(T^{-1/2})$ convergence rate, and avoids $\text{poly}(A_{\max})$ dependency simultaneously.%
\footnote{Accepted for presentation at the Conference on Learning Theory (COLT) 2024.}
\end{abstract}

\section{Introduction}

Multi-Agent Reinforcement Learning (MARL) studies decision-making under uncertainty in a multi-agent system.
Many practical MARL systems demonstrate impressive performance in games like Go \citep{silver2017mastering}, Poker \citep{brown2019superhuman}, Starcraft II \citep{vinyals2019grandmaster}, Hide-and-Seek \citep{Baker2020Emergent}, and Autonomous Driving \citep{shalev2016safe}.
While MARL exhibits huge success in practice, it is still far from being understood theoretically.

A core challenge in theoretically studying MARL is the ``curse of multi-agents'': When many agents are involved in the game, the joint state and action space is prohibitively large.
Thus, early algorithms for multi-agent games \citep{liu2021sharp} usually have a sample complexity (the number of samples the algorithm requires to attain a given accuracy) exponentially depending on the number of agents $m$, for example, scaling with $\prod_{i\in [m]}\lvert \mA_i\rvert$ where $\mA_i$ is the action space of the $i$-th agent.

Later, many efforts were made to resolve this issue.
\citet{jin2021v} gave the first algorithm avoiding the curse of multi-agents, and \citet{daskalakis2023complexity} made the output policy \textit{Markovian} (\textit{i.e.}, non-history-dependent, which we focus in our paper). Such results only depend on $A_{\max}\triangleq \max_{i\in [m]}\lvert \mA_i\rvert$ but not $\prod_{i\in [m]}\lvert \mA_i\rvert$.
While these algorithms work well in the \textit{tabular} Markov Games \citep{shapley1953stochastic},\footnote{Tabular Markov Games refer to the Markov Games where the number of states and actions are finite and small.} they cannot handle the case where the state space is prohibitively large.

However, in many real-world applications, tabular models are insufficient. For example, Go has $3^{361}$ possible states.
In single-agent RL, people use function approximation to model the state space \citep{jiang2017contextual,wen2017efficient}.
While this idea naturally generalizes to MARL \citep{xie2020learning,chen2022almost}, unfortunately, it usually induces the curse of multi-agents.

Two recent works by \citet{cui2023breaking} and \citet{wang2023breaking} investigated this issue.
They concurrently proposed that instead of the \textit{global} function approximations previously used in the literature, \textit{independent} function approximations should be developed.
By assuming independent linear function approximations, they designed the first algorithms that can avoid the curse of multi-agents when the state spaces are prohibitively large and function approximations are deployed.

While their algorithms succeeded in yielding polynomial dependencies on $m$, they were sub-optimal in other terms.
The sample complexity of \citet{cui2023breaking} was $\Otil(\poly(m,H,d)\epsilon^{-4})$, while that of \citet{wang2023breaking} was $\Otil(\poly(m,H,d)A_{\max}^5\epsilon^{-2})$.%
\footnote{We use $\Otil$ to hide any logarithmic factors. Here, $m$ is the number of agents, $H$ is the length of each episode, $d$ is the dimension of the feature space, $A_{\max}=\max_{i\in [m]}\lvert \mA_i\rvert$ is the largest size of action sets, and $\epsilon$ is the desired accuracy.}
The former has a sub-optimal  convergence rate of $\epsilon^{-4}$, whereas the latter has a polynomial dependency on the number of actions.
However, no polynomial dependency on the number of actions is necessary for single-agent RL with linear function approximations, even when the losses can arbitrarily vary with time, \textit{i.e.}, in the so-called adversarial regime \citep{dai2023refined}.

As Linear MGs are generalizations of Linear MDPs, we aim to generalize such a property to Linear MGs.
In this paper, we propose an algorithm for multi-player general-sum Markov Games with independent linear function approximations that \textit{i)} retains a polynomial dependency on $m$, \textit{ii)} ensures the optimal $\epsilon^{-2}$ convergence rate, and \textit{iii)} only has \textit{logarithmic} dependency on $A_{\max}$.

\subsection{Key Insights and Technical Overview of This Paper}
The key insight in our paper is developing \textit{data-dependent} (\textit{i.e.}, stochastic) pessimistic sub-optimality gap estimators instead of deterministic ones.
For more context, the \texttt{AVLPR} framework designed and used by \citet{wang2023breaking} required a \textit{deterministic} gap estimation regarding the current policy $\tilde \pi$ during execution -- so that the agents can collaborate to further improve $\tilde \pi$.
Unfortunately, yielding such a deterministic estimation corresponds to an open problem called ``high-probability regret bounds for adversarial contextual linear bandits'' in the literature \citep{olkhovskaya2023first}. Thus, \citet{wang2023breaking} used a uniform exploration strategy to avoid proving regret bounds; however, this approach unavoidably brings $\poly(A)$ factors.
On the other hand, the framework by \citet{cui2023breaking} was intrinsically incapable of $\epsilon^{-2}$ convergence rate as it uses the \textit{epoching} technique (\textit{i.e.}, fixing the policy for many episodes so the environment is almost stationary).

Hence, existing ideas in the literature cannot give favorable sample complexities, and we thus propose the usage of stochastic sub-optimality estimators.
To fully deploy our insight, we make the following technical contributions:
\begin{enumerate}
\item Based on the \texttt{AVLPR} framework by \citet{wang2023breaking} which required deterministic sub-optimality gap estimations, we propose a refined framework capable of data-dependent (\textit{i.e.}, non-deterministic) estimators in \Cref{alg:framework}. As we show in \Cref{lem:new main theorem}, a \textit{stochastic gap estimation with bounded expectation} already suffices for an $\epsilon$-CCE. This innovation, as we shall see shortly, gives more flexibility in algorithms and allowing techniques from expected-regret-minimization literature.

Slightly more formally, suppose that we would like to evaluate a joint policy $\tilde \pi_t$. However, its actual sub-optimality gap, denoted by $\Gap_{\tilde \pi_t}$, cannot be accurately calculated during runtime since the ``optimal'' policy is unknown. The original approach requires a \textit{deterministic} constant $G_t$ such that $\Gap_{\tilde \pi_t}\le G_t$ \textit{w.h.p.} However, the approach we use to generate $\tilde \pi_t$ (which used the famous \textit{regret-to-sample-complexity} reduction and translates the problem to regret-minimization in an adversarial contextual linear bandit; see \Cref{sec:cce-approx} for more) does not allow such a deterministic $G$. Instead of crafting $\tilde \pi$ in another way like \citet{wang2023breaking}, we propose that calculating a \textit{random variable} $\Gap_t$ such that \textit{i)} $\Gap_{\tilde \pi_t}\le \Gap_t$ \textit{w.h.p.}, and \textit{ii)} $\E[\Gap_{\tilde \pi_t}]\le G_t$ for some deterministic constant $G_t$. More details can be found in \Cref{lem:new main theorem}.

\item Existing expected-regret-minimization algorithms cannot be directly used, as they only guarantees the second condition of $\E[\Gap_{\tilde \pi_t}]\le G_t$ -- but in addition to this, we also want $\Gap_t$ to be a high-probability pessimistic estimation of $\Gap_{\tilde \pi_t}$.
Meanwhile, previous algorithms in high-probability RL also do not directly work due to the open problem. Technically, this is because existing bonus-design mechanism cannot cover estimation errors which can occasionally have extreme magnitudes albeit with well-bounded expectations; see \Cref{sec:action-dependent bonus} for a more formal description.

To tackle this issue, we propose a novel technique called \textit{action-dependent bonuses} which was partially inspired by the Adaptive Freedman Inequality proposed by \citet{lee2020bias} and improved by \citet{zimmert2022return}. As we detail in \Cref{sec:action-dependent bonus}, such a technique can be applied when \textit{a)} we want to use the bonus technique to cancel some error in the form of $\sum_t v_t(a^\ast)$, where $a^\ast$ is the optimal action in hindsight; but \textit{b)} $\sup_{t,a} \lvert v_t(a)\rvert$ can be prohibitively large so Freedman fails, while \textit{c)} $\E_{a\sim \pi_t}[\sup_t \lvert v_t(a)\rvert]$ is small where $\pi_t$ is the player's policy.
Going beyond this paper, we expect this \textit{action-dependent bonuses} technique to be also applicable elsewhere when high-probability bounds are desired, but the error to be covered can sometimes be prohibitively large -- though its expectation \textit{w.r.t.} the player's policy can be well controlled.
\item Finally, because of unknown transitions and multiple agents, it is also non-trivial to attain an $\Otil(\sqrt K)$-style expected regret guarantee.
Towards this, we incorporate several state-of-the-art techniques from the recent adversarial RL literature, \textit{e.g.}, the Magnitude-Reduced Estimators proposed by \citet{dai2023refined}, the Adaptive Freedman Inequality by \citet{zimmert2022return}, and a new covariance matrix estimation technique introduced by \citet{liu2023bypassing}.
\end{enumerate}

\subsection{Related Work}
\label{sec:related work}

\paragraph{Tabular Markov Games.}
Markov Games date back to \citet{shapley1953stochastic}.
For tabular Markov Games where the number of states and actions is finite and small, \citet{bai2020provable} gave the first sample-efficient algorithm in two-player zero-sum games.
For the harder multi-player general-sum case, the first provably sample-efficient algorithm was developed by \citet{liu2021sharp}, albeit depending on $\prod_{i\in [m]}\lvert \mA_i\rvert$ (\textit{i.e.}, the ``curse of multiagents'').
When non-Markovian policies were allowed, based on the V-learning algorithm \citep{bai2020near}, various algorithms were proposed \citep{jin2021v,song2022when,mao2023provably}; otherwise, the first algorithm for multi-player general-sum games avoiding the curse of multiagents only dates back to \citep{daskalakis2023complexity}, although with a sub-optimal $\epsilon^{-3}$ convergence rate.
\citet{cui2023breaking} and \citet{wang2023breaking} recently yielded the optimal $\epsilon^{-2}$ convergence rate, but their dependencies on $S$ remained improvable.

\paragraph{Markov Games with Function Approximation.}
When the state space can be prohibitively large, as in the single-agent case, the state space is often modeled via function approximations.
Early works in this line \citep{xie2020learning,chen2022almost} considered \textit{global} function approximations where the function class captures the joint value functions of all the agents, making it hard to avoid the curse of multiagents.
The idea of \textit{independent} function approximation, \textit{i.e.}, the function class of each agent only encodes its own value function, was concurrently proposed by \citet{cui2023breaking} and \citet{wang2023breaking}. However, as mentioned, their sample complexities were sub-optimal in either $\epsilon$ or $A$ while ours is optimal in both $\epsilon$ and $A$.
Notably, while this paper only focuses on linear function approximations, more general approximation schemes were already studied in the literature, both globally (see, \textit{e.g.}, \citep{huang2022towards,jin2022power,xiong2022self,chen2022unified,ni2022representation,zhan2023decentralized}) or independently \citep{wang2023breaking}.

\paragraph{Markov Decision Processes with Linear Function Approximation.}
With only one agent, Markov Games became Markov Decision Processes (MDPs), whose linear function approximation schemes were extensively studied.
When losses are fixed across episodes, the problem was solved by \citet{jin2020provably} and \citet{yang2020reinforcement}. When losses are \textit{adversarial} (\textit{i.e.}, can arbitrarily vary with time), some types of linear approximations were recently tackled, \textit{e.g.}, linear mixture MDPs \citep{zhao2023learning} or linear-Q MDPs equipped with simulators \citep{dai2023refined}. In contrast, other approximation methods, \textit{e.g.}, linear MDPs, remain open. More detailed discussions can be found in recent papers like \citep{dai2023refined,kong2023improved,sherman2023improved,sherman2023rate,liu2023towards}.

\paragraph{Concurrent Work by \citet{fan2024rl}.}
After the submission of this paper, we are aware of a concurrent and independent work by \citet{fan2024rl}, which also studies the sample complexity of finding a CCE in Markov Games with Independent Linear Function Approximations.
Different from the online model studied in this paper and in previous works by \citet{cui2023breaking} and \citet{wang2023breaking}, they assumed a \textit{local access model} (\textit{i.e.}, there exists a simulator where the learner can query for samples $s'\sim \P(\cdot \mid s,\bm a)$ whenever $s$ is a previously visited state).
Under this stronger assumption, \citet{fan2024rl} achieved $\Otil(m^2 d^2 H^6 \epsilon^{-2})$ sample complexity, which resolving the curse of multi-agents, having optimal dependency on $\epsilon$, and avoiding polynomial dependency on $A_{\max}$.
Technically, by maintaining core set of \textit{well-covered} state-action pairs, each agent can independently perform policy learning (via a FTRL-based subroutine) and thus avoiding the curse of multi-agents. Moreover, as core sets have sizes independent to $A$ \citep{yin2022efficient}, their approach avoids $\text{poly}(A_{\max})$ factors as well.
Further making our results and that of \citet{fan2024rl} completely independent of $S$ or enjoy better dependency on $m,d,H$ remains a valuable direction for future research.\footnote{Throughout this paper, we omit the $\O(\log S)$ term (due to \Cref{lem:Gap concentration}) in the sample complexity bound as it's logarithmic. However, as discussed by \citet{fan2024rl}, it would be more favorable to have the $\log S$ factor removed.}

\section{Preliminaries}
\label{sec:setup}

\paragraph{Markov Games.}
In (multi-agent general-sum) Markov Games, there are $m$ agents sharing a common state space $\mS$, but each agent $i\in [m]$ has its own action space $\{\mA^i\}_{i=1}^m$.
The game repeats for several episodes, each with length $H$. Without loss of generality, assume that $\mS$ is layered as $\mS_1,\mS_2,\ldots,\mS_{H+1}$ such that transition only occurs from one layer to the next.
At the beginning of each episode, the state resets to an initial state $s_1\in \mS_1$. For the $h$-th step, each agent $i\in [m]$ observes the current state $s_h$ and makes its action $a_h^i\in \mA^i$.

Let the \textit{joint action} be $\bm a_h=(a_h^1,a_h^2,\ldots,a_h^m)\in \mA^1\times \mA^2\times \cdots \times \mA^m\triangleq \mA$. The new state $s_{h+1}\in \mS_{h+1}$ is independently sampled from a distribution $\P(\cdot \mid s_h,\bm a_h)$ (hidden from the agents).
Meanwhile, each agent $i\in [m]$ observes and suffers a \textit{loss} $\ell^i(s_h,\bm a_h)\in [0,1]$.\footnote{Adopting notations from the adversarial MDP literature, this paper focuses on losses instead of rewards (\textit{i.e.}, agents minimize total loss instead of maximizing total reward). Following the convention \citep{cui2023breaking,wang2023breaking}, we assume the loss functions $\ell^i$ are deterministic (though kept as secret from the agents).} The objective of each agent is to minimize the expectation of its total loss, \textit{i.e.}, agent $i\in [m]$ minimizes $\E\big [\sum_{h=1}^H \ell^i(s_h,\bm a_h)\big ]$.

\paragraph{Policies and Value Functions.}
A \textit{(Markov joint) policy} $\pi$ is a joint strategy of each agent, formally defined as $\pi\colon \mS\to \triangle(\mA)$. Note that this allows the policies of different agents to be correlated.
For a Markov joint policy $\pi$, let $\pi^i$ be the policy induced by agent $i$, and let $\pi^{-i}$ be the joint policy induced by all agents except $i$. Define $\Pi=\{\pi\colon \mS\to \triangle(\mA)\}$ as the set of all Markov joint policies. Similarly, define $\Pi^i=\{\pi^i\mid \pi \in \Pi\}$ and $\Pi^{-i}=\{\pi^{-i}\mid \pi \in \Pi\}$.

Given a joint policy $\pi\in \Pi$, one can define the following state-value function (\textit{V-function} in short) induced by $\pi$ for each agent $i\in [m]$ and state $s\in \mS$ (where $h\in [H]$ is the layer $s$ lies in):
\begin{equation*}
V^i_\pi(s)=\E_{(s_1,\bm a_1,s_2,\bm a_2,\ldots,s_H,\bm a_H)\sim \pi}\left [\sum_{\happa=h}^H \ell^i(s_{\happa},\bm a_{\happa})\middle \vert s_h=s\right ],~ \forall i\in [m],s\in \mS_h.
\end{equation*}
Here, $(s_1,\bm a_1,s_2,\bm a_2,\ldots,s_H,\bm a_H)\sim \pi$ denotes a trajectory generated by following $\pi$, \textit{i.e.}, $s_1=s_1$, $\bm a_h\sim \pi_h(\cdot \mid s_h)$, and $s_{h+1}\sim \P(\cdot \mid s_h,\bm a_h)$. If we are only interested in the $h$-th state in such a trajectory, we use $s\sim_h \pi$ to denote an $s_h\in \mS_h$ generated by following $\pi$.

Similar to V-function, the state-action-value function (\textit{Q-function}) can be defined as follows:
\begin{equation*}
Q^i_\pi(s,a^i)=\E_{(s_1,\bm a_1,s_2,\bm a_2,\ldots,s_H,\bm a_H)\sim \pi}\left [\sum_{\happa=h}^H \ell^i(s_{\happa},\bm a_{\happa})\middle \vert (s_h,\bm a_h^i)=(s,a^i)\right ],~ \forall i\in [m],s\in \mS_h,a^i\in \mA^i.
\end{equation*}

Slightly abusing the notation $\P$, we define an operator $\P$ as $[\P V](s,\bm a)=\E_{s'\sim \P(\cdot \mid s,\bm a)}[V(s')]$. Then we can simply rewrite the Q-function as $Q_\pi^i(s,a^i)=\E_{\bm a^{-i}\sim \pi^{-i}(s)}\left [(\ell^i+\P V_\pi^i)(s,\bm a)\right ]$.

\paragraph{Coarse Correlated Equilibrium.}
As all agents have different objectives, we can only find an equilibrium policy where no agent can gain much more by voluntarily deviating. In general, calculating the well-known Nash equilibrium is intractable even in normal-form general-sum games, \textit{i.e.}, MGs with $H=1$ and $S=1$ \citep{daskalakis2009complexity}. 
Hence, people usually consider (Markov) Coarse Correlated Equilibrium \citep{daskalakis2023complexity,cui2023breaking,wang2023breaking} instead.

Formally, for each agent $i\in [m]$, we fix the strategy of all remaining agents as $\pi^{-i}\in \Pi^{-i}$ and consider its \textit{best-response}. We define its \textit{best-response V-function} against $\pi^{-i}$ as
\begin{equation*}
V^i_{\dagger,\pi^{-i}}(s)=\min_{\pi^i\in \Pi^i} V^i_{\pi^i,\pi^{-i}}(s),\quad \forall i\in [m],s\in \mS.
\end{equation*}
A Markov joint policy $\pi$ is a \textit{(Markov) $\epsilon$-CCE} if $\max\limits_{i\in [m]}\left \{V^i_\pi(s_1)-V^i_{\dagger,\pi^{-i}}(s_1)\right \}\le \epsilon$.
We measure an algorithm's performance by the number of samples needed for learning an $\epsilon$-CCE, namely \textit{sample complexity}.
When the state space $\mS$ is finite and small (which we call a \textit{tabular} MG), the best-known sample complexity for finding an $\epsilon$-CCE is $\Otil(H^6S^2A_{\max}\epsilon^{-2})$ \citep{wang2023breaking,cui2023breaking}.

\paragraph{MGs with Independent Linear Function Approximation.}
When $\mS$ is infinite, \citet{cui2023breaking} and \citet{wang2023breaking} concurrently propose to use \textit{independent linear function approximation} (though with different Bellman completeness assumptions; see Appendix D of \citet{wang2023breaking}).
Their results are $\Otil(d^4H^6m^2A_{\max}^5\epsilon^{-2})$ \citep{wang2023breaking} and $\Otil(d^4H^{10}m^4\epsilon^{-4})$ \citep{cui2023breaking}.

Our paper also considers Markov Games with independent linear function approximations. Inspired by \textit{linear MDPs} in single-agent RL \citep{jin2020provably}, we assume the transitions and losses to be linear.
For a detailed discussion on the connection between linear MDP and Bellman completeness, we refer the readers to Section 5 of \citep{jin2020provably}.
Formally, we assume the following:
\begin{assumption}\label{assumption}
For any agent $i\in [m]$, there exists a known $d$-dimensional feature mapping ${\bm \phi}\colon \mS\times \mA^i\to \mathbb R^d$, such that for any state $s\in \mS$, action $a^i\in \mA^i$, and any policy $\pi\in \Pi^{\text{est}}$,
\footnote{$\Pi^{\text{est}}$ refers to the set of all policies that the algorithm may give, similar to \citep{cui2023breaking,wang2023breaking}.}
\begin{alignat*}{2}
&\E\big [\P(s'\mid s,\bm a)&&\big \vert \bm a^{-i}\sim \pi^{-i}(\cdot \mid s),\bm a^i=a^i\big ]={\bm \phi}(s,a^i)^\trans {\bm \mu}_{\pi^{-i}}^i(s'),\\
&\E\big [\ell^i(s,\bm a)&&\big \vert \bm a^{-i}\sim \pi^{-i}(\cdot \mid s),\bm a^i=a^i\big ]={\bm \phi}(s,a^i)^\trans {\bm \nu}_{\pi^{-i}}^i(h),
\end{alignat*}
where ${\bm \mu}^i\colon \Pi^{-i}\times \mS\to \mathbb R^d$ and ${\bm \nu}^i\colon \Pi^{-i}\times [H]\to \mathbb R^d$ are both \textit{unknown} to the agent.
Following the convention \citep{jin2020provably,luo2021policy}, we assume $\lVert {\bm \phi}(s,a^i)\rVert_2\le 1$ for all $s\in \mS$, $a^i\in \mA^i$, $i\in [m]$ and that $\max\{\lVert {\bm \mu}_{\pi^{-i}}^i(\mS_h)\rVert_2,\lVert {\bm \nu}_{\pi^{-i}}^i(h)\rVert_2\}\le \sqrt d$ for all $h\in [H]$, $\pi^{-i}\in \Pi^{-i}$, and $i\in [m]$.

This assumption also implies the Q-functions for all the agents are linear, \textit{i.e.}, there exists some unknown $d$-dimensional feature mapping ${\bm \theta}^i\colon \Pi^{\text{est}}\times [H]\to \mathbb R^d$ with $\lVert \bm \theta_\pi^i(h)\rVert_2\le \sqrt dH$ such that
\begin{equation*}
Q_{\pi}^i(s,a^i)= {\bm \phi}(s,a^i)^\trans {\bm \theta}_\pi^i(h),\quad \forall s\in \mS,a^i\in \mA^i,i\in [m],\pi \in \Pi.
\end{equation*}
\end{assumption}

\section{Improved \texttt{AVLPR} Framework}
\label{sec:framework}


\begin{algorithm}[!t]
\caption{Improved \texttt{AVLPR} Framework}
\label{alg:framework}
\begin{algorithmic}[1]
\REQUIRE{\#epochs $T$, potentials $\{\Psi_{t,h}^i\}_{t\in [T],h\in [H],i\in [m]}$, subroutines \textsc{CCE-Approx} and \textsc{V-Approx}.}
\STATE Set policy $\tilde \pi_0$ as the uniform policy for all $i\in [m]$, \textit{i.e.}, $\tilde \pi_0^i(s)\gets \Unif(\mA^i)$. Set $t_0\gets 0$.
\FOR{$t=1,2,\ldots,T$}
\STATE All agents play according to $\tilde \pi_{t-1}$. Each agent $i\in [m]$ records trajectory $\{(\tilde s_{t,h},\tilde a_{t,h}^i)\}_{h=1}^H$.
\STATE Each agent $i\in [m]$ calculates its potential function $\Psi_{t,h}^i$ according to the definition.
\IF{$t_0\ne 0$ and $\Psi_{t,h}^i\le \Psi_{t_0,h}^i+1$ for all $t\in [T]$, $h\in [H]$, and $i\in [m]$}\label{line:condition}
\STATE Do a ``lazy'' update by directly setting $\tilde \pi_{t}\gets \tilde \pi_{t-1}$ and $\Gap_t\gets \Gap_{t-1}$; \textbf{continue}.
\ENDIF
\STATE Define $\bar \pi_t\colon \mS\to \triangle(\mA)$ as the uniform mixture of all previous policies $\tilde \pi_0\colon \mS\to \triangle(\mA),\tilde \pi_1\colon \mS\to \triangle(\mA),\ldots,\tilde \pi_{t-1}\colon \mS\to \triangle(\mA)$, or denoted by $\bar \pi_t\gets \frac 1t \sum_{s=0}^{t-1} \tilde \pi_s$ for simplicity.
\STATE We will then fill $\tilde \pi_t \colon \mS \to \triangle(\mA)$, $\Gap\colon \mS\to \mathbb R^m$, and $\bar V_t\colon \mS \to \mathbb R^m$ layer-by-layer, in the order of $\mS_H,\mS_{H-1},\ldots,\mS_1$. Before that, we first initialize $\bar V_{t}^i(s_{H+1})=0$ for all $i\in [m]$.
\FOR{$h=H,H-1,\ldots,1$}
\STATE Execute subroutine $\mathsc{CCE-Approx}_h(\bar \pi_t, \bar V_{t}, t)$ {\color{violet}independently for $R\triangleq \O(\log \frac 1\delta)$ times}. For the $r$-th execution of $\mathsc{CCE-Approx}_h(\bar \pi_t, \bar V_{t}, t)$, record the return value $(\tilde \pi_r\colon \mS_h\to \triangle(\mA),\Gap_r\colon \mS_h \to \mathbb R^m)$.
\STATE {\color{violet}For each current-layer state $s\in \mS_h$, set $(\tilde \pi_t(s),\Gap_t(s))\gets (\tilde \pi_{r^\ast(s)}(s),\Gap_{r^\ast(s)}(s))$ where
\begin{equation}\label{eq:definition of r*}
r^\ast(s)\triangleq \argmin_{r\in [R]}\sum_{i=1}^m \Gap_r^i(s).
\end{equation}}
\STATE Update the current-layer V-function $\{\bar V_t(s)\}_{s\in \mS_h}$ (abbreviated as $\bar V(\mS_{h+1})$) from the next-layer $\{\bar V_t(s)\}_{s\in \mS_{h+1}}$ (or simply $\bar V_t(\mS_{h+1})$) by calling $\bar V_t(\mS_h)\gets \mathsc{V-Approx}_h(\bar \pi_t, \tilde \pi_{t}, \bar V_{t}(\mS_{h+1}), \Gap_t,t)$.
\ENDFOR
\STATE Update the ``last update time'' $t_0\gets t$.
\ENDFOR
\ENSURE{The uniform mixture of all policies $\tilde \pi_0,\tilde \pi_1,\tilde \pi_2,\ldots,\tilde \pi_T$, \textit{i.e.}, $\pi_{\text{out}}\gets \frac{1}{T+1} \sum_{t=0}^T \tilde \pi_t$.}
\end{algorithmic}
\end{algorithm}

Our framework, presented in \Cref{alg:framework}, is based on the \texttt{AVLPR} framework proposed by \citet{wang2023breaking}. The main differences are marked in violet.
Before introducing these differences, we first overview the original \texttt{AVLPR} framework, which is almost the same as \Cref{alg:framework} except that $R=1$ --- one of the most crucial innovations of our framework which we will describe later.

\paragraph{Overview of the \texttt{AVLPR} Framework by \citet{wang2023breaking}.}
The original framework starts from an arbitrary policy $\tilde \pi_0$ and then gradually improves it: In the $t$-th epoch, all agents together make the next policy $\tilde \pi_t$ an $\Otil(1/\sqrt t)$-CCE. Thus, the number of epochs for an $\epsilon$-CCE is $T=\Otil(\epsilon^{-2})$.

For each epoch $t\in [T]$, the agents determine their new policies $\tilde \pi_t$ layer-by-layer in the reversed order, \textit{i.e.}, $h=H,H-1,\ldots,1$.
Suppose that we are at layer $h\in [H]$ and want to find $\{\tilde \pi_t(s)\}_{s\in \mS_h}$ (abbreviated as $\tilde \pi_t(\mS_h)$ for simplicity).
As $\{\tilde \pi_t(\mS_\happa)\}_{\happa=h+1}^H$ are already calculated, the next-layer V-functions $V_{\tilde \pi_t}^i(\mS_{h+1})$ can be estimated (denoted by $\bar V_t^i$ in \Cref{alg:framework}).
Thus, the problem of deciding $\tilde \pi_t^i(\mS_h)$ becomes a \textit{contextual bandit} problem: The context $s$ is sampled from a fixed policy $\bar \pi_t$ (which only depends on $\tilde \pi_1,\tilde \pi_2,\ldots,\tilde \pi_{t-1}$), and the loss of every $(s,a^i)\in \mS_h\times \mA^i$ is the Q-function induced by the current policy, namely $Q_{\tilde \pi_t}^i(s,a^i)$, which can be estimated via $\ell^i(s,a^i)$ and $\bar V^i(\mS_{h+1})$.

Hence, \citet{wang2023breaking} propose to deploy a contextual bandit algorithm on this layer $h$. This is abstracted as a plug-in subroutine $\textsc{CCE-Approx}_h$ in \Cref{alg:framework}.\footnote{In the \textsc{CCE-Approx} subroutine, an iterative approach is also used, which means the  $\tilde \pi_t(\mS_h)$ is the average of a few policies $\pi_1,\pi_2,\ldots,\pi_K$. Thus, as $\pi_k^{-i}$ is varying with $k$, each Q-function $Q_{\pi_k}^i(s,a^i)$ is also varying with time, which means that each agent actually faces an \textit{adversarial} (\textit{i.e.}, non-stationary) contextual bandit problem. We shall see more details in Section \ref{sec:cce-approx}.\label{footnote:adversarial contextual bandit}} As we briefly mention in the Technical Overview, we should ensure that \textit{i)} the joint policy $\tilde \pi\colon \mS_h\to \triangle(\mA)$ has a calculable sub-optimality gap on all $s\in \mS_h$ for all $i\in [m]$ \textit{w.h.p.} (see \Cref{eq:original gap requirement}), and \textit{ii)} this sub-optimality gap estimation $\Gap\colon \mS_h\to \mathbb R^m$ has a bounded expectation \textit{w.r.t.} $\tilde \pi_t$ (see \Cref{eq:expectation of Gap}).

Afterward, \citet{wang2023breaking} estimates the V-function for the current layer (\textit{i.e.}, $V_{\tilde \pi}^i(\mS_h)$), which is useful when we move on to the previous layer and invoke $\textsc{CCE-Approx}_{h-1}$. This is done by another subroutine called $\textsc{V-Approx}_h$, which must ensure an ``optimistic'' estimation (see \Cref{eq:optimistic V-function}).

By repeating this process for all $h$, the current epoch $t$ terminates. It can be inferred from \Cref{eq:original gap requirement,eq:optimistic V-function,eq:expectation of Gap} that $\tilde \pi_t$ is an $\Otil(1/\sqrt t)$-CCE.
To further reduce the sample complexity, \citet{wang2023breaking} propose another trick to ``lazily'' update the policies. Informally, if the $\Gap$ functions remain similar (measured by the increments in the potential function $\Psi_{t,h}^i$, \textit{\textit{c.f.}} \Cref{line:condition}), then we directly adopt the previous $\tilde \pi_t$. By ensuring such updates only happen $\O(\log T)$ times by properly choosing $\Psi^i$, $\Otil(\epsilon^{-2})$ sample complexity is enjoyed.
The main theorem of the original \texttt{AVLPR} framework can be summarized as follows.
\begin{theorem}[Main Theorem of \texttt{AVLPR} {\citep[Theorem 18]{wang2023breaking}}; Informal]\label{lem:original main theorem}
Suppose that
\begin{enumerate}
    \item (Per-state no-regret) If we call the subroutine $\mathsc{CCE-Approx}_h(\bar \pi, \bar V, K)$, then the returned policy $\tilde \pi \colon \mS_h\to \triangle(\mA)$ and gap estimation $\Gap\colon \mS_h\to \mathbb R^m$ shall ensure the following \textit{w.p.} $1-\delta$:
    \begin{align}
    \max_{\pi_\ast^i\in \Pi_h^i}\left \{\left (\E_{a\sim \tilde \pi}-\E_{a\sim \pi_\ast^i\times \tilde \pi^{-i}}\right )\left [\big (\ell^i+\P_{h+1}\bar V^i\big )(s,a)\right ]\right \} \le \Gap^i(s),\quad \forall i\in [m], s\in \mS_h,\label{eq:original gap requirement}
    \end{align}
    where $\Gap^i(s)$ must be a \textbf{\color{violet}deterministic} (\textit{i.e.}, non-stochastic) function in the form $G_h^i(s,\bar \pi,K,\delta)$.
    \item (Optimistic V-function) If we call the subroutine $\mathsc{V-Approx}_h(\bar \pi,\tilde \pi,\bar V,\Gap,K)$, then the returned V-function estimation $\hat V\colon \mS_h\to \mathbb R^m$ should ensure the following \textit{w.p.} $1-\delta$:
    \begin{align}
    \hat V^i(s)\in \bigg [\min\bigg \{&\E_{a\sim \tilde \pi}\left [\big (\ell^i+\P_{h+1}\bar V^i\big )(s,a)\right ]+\phantom{2}\Gap^i(s),H-h+1\bigg \},\nonumber\\
    &\E_{a\sim \tilde \pi}\left [\big (\ell^i+\P_{h+1}\bar V^i\big )(s,a)\right ]+2\Gap^i(s)\bigg ],\quad \forall i\in [m], s\in \mS_h.\label{eq:optimistic V-function}
    \end{align}
    \item (Pigeon-hole condition) There exists a deterministic complexity measure $L$ such that \textit{w.p.} $1-\delta$
    \begin{equation}\label{eq:expectation of Gap}
    \sum_{t=1}^T \E_{s\sim_h \tilde \pi_t} \left [G_h^i(s,\Unif(\{\tilde \pi_\tau\}_{\tau\in [t]}),t,\delta)\right ]\le \sqrt{LT \log^2 \frac T\delta},\quad \forall i\in [m],h\in [H],
    \end{equation}
    where $G_h^i$ is the deterministic \textsc{Gap} function defined in \Cref{eq:original gap requirement}. Informally, if we execute the whole $\{\tilde \pi_t(\mS_h)\}_{h\in [H]}$, it must have a expected sub-optimality gap of order $\Otil(\sqrt{L/t})$ in each.
    \item (Potential function) 
    The potential functions $\{\Psi_{t,h}^i\}_{t\in [T],h\in [H],i\in [m]}$ is chosen \textit{s.t.} there exists a constant $d_{\text{replay}}$ ensuring that \Cref{line:condition} is violated for at most $d_{\text{replay}}\log T$ times. Meanwhile,
    \begin{equation}\label{eq:potential function}
    \Psi_{t,h}^i\le \Psi_{t_0,h}^i+1\Longrightarrow G_h^i(s,\bar \pi_{t_0},t_0,\delta)\le 8 G_h^i(s,\bar \pi_{t},t,\delta),~~ \forall i\in [m],h\in [H].
    \end{equation}
\end{enumerate}
Then, by setting $T=\Otil(H^2L\epsilon^{-2})$, an $\epsilon$-CCE can be yielded within $\Otil(H^3 L d_{\text{replay}} \epsilon^{-2})$ samples.
\end{theorem}

One can see that \citet{wang2023breaking} require the $\Gap$ function to be \textbf{\textit{\color{violet}deterministic}} (highlighted in \Cref{lem:original main theorem}).
However, as we mentioned in \Cref{footnote:adversarial contextual bandit}, this corresponds to crafting high-probability regret bounds for adversarial linear contextual bandits, which is still open in the literature \citep{olkhovskaya2023first}.
Hence, when facing MGs with (independent) linear function approximation, it is highly non-trivial to construct a \textit{non-stochastic} high-probability upper bound like \Cref{eq:original gap requirement} via deploying regret-minimization algorithm at each state $s\in \mS_h$.
Consequently, \citet{wang2023breaking} adopt a pure exploration algorithm (\textit{i.e.}, using uniform policies in subroutine \textsc{CCE-Approx}), which brings undesirable $\poly(A_{\max})$ dependencies.

\paragraph{Loosened High-Probability Bound Requirement.}
To bypass this situation, instead of forcing each agent to do pure exploration like \citep{wang2023breaking}, our Improved \texttt{AVLPR} framework allows the $\Gap(s)$ to be a \textit{\textbf{\color{violet}stochastic}} (\textit{i.e.}, data-dependent) upper bound of the actual sub-optimality gap, as shown in \Cref{eq:new gap requirement}. The differences between \Cref{lem:original main theorem,lem:new main theorem} are highlighted in violet. 
\begin{theorem}[Main Theorem of Improved \texttt{AVLPR}; Informal]\label{lem:new main theorem}
Suppose that
\begin{enumerate}
    \item (Per-state no-regret) If we call the subroutine $\mathsc{CCE-Approx}_h(\bar \pi, \bar V, K)$, then the returned policy $\tilde \pi \colon \mS_h\to \triangle(\mA)$ and gap estimation $\Gap\colon \mS_h\to \mathbb R^m$ shall ensure the following \textit{w.p.} $1-\delta$:
    \begin{align}
    \max_{\pi_\ast^i\in \Pi_h^i}\left \{\left (\E_{a\sim \tilde \pi}-\E_{a\sim \pi_\ast^i\times \tilde \pi^{-i}}\right )\left [\big (\ell^i+\P_{h+1}\bar V^i\big )(s,a)\right ]\right \} \le \Gap^i(s),\quad \forall i\in [m], s\in \mS_h,\label{eq:new gap requirement}
    \end{align}
    where $\Gap^i(s)$ is a \textbf{\color{violet}random variable} whose randomness comes from the environment (when generating the trajectories), the agents (when playing the policies), and internal randomness.
    \item (Optimistic V-function)
    If we call the subroutine $\mathsc{V-Approx}_h(\bar \pi,\tilde \pi,\bar V,\Gap,K)$, then the returned V-function estimation $\hat V\colon \mS_h\to \mathbb R^m$ should ensure \Cref{eq:optimistic V-function} \textit{w.p.} $1-\delta$.
    \item (Pigeon-hole condition \& Potential Function)
    The potential functions $\{\Psi_{t,h}^i\}_{t\in [T],h\in [H],i\in [m]}$ is chosen \textit{s.t.} there exists a constant $d_{\text{replay}}$ ensuring that \Cref{line:condition} is violated for at most $d_{\text{replay}}\log T$ times. Meanwhile, there also exists a deterministic $L$ ensuring the following \textit{w.p.} $1-\delta$:
    \begin{equation}\label{eq:new expectation of Gap}
    \sum_{t=1}^T \sum_{i=1}^m \E_{s\sim_h \tilde \pi_t} \left [{\color{violet}\E_{\Gap}[\Gap_t^i(s)]}\right ]\le m \sqrt{LT \log^2 \frac T\delta},\quad \forall h\in [H],
    \end{equation}
    where the expectation is taken \textit{w.r.t.} the randomness in calculating the random variable \textsc{Gap}.
\end{enumerate}
Then we have a similar conclusion as \Cref{lem:original main theorem} that $\Otil(m^2 H^3 L d_{\text{replay}} \epsilon^{-2})$ samples give an $\epsilon$-CCE.
\end{theorem}

We defer the formal proof of this theorem to \Cref{sec:appendix main theorem} and only sketch the idea here.
\begin{proof}[Proof Sketch of \Cref{lem:new main theorem}]
The idea is to show that our picked policy-gap pair nearly satisfies the conditions in \Cref{lem:original main theorem}. We can pick different $r$'s for different states $s\in \mS_h$ (\textit{i.e.}, $r^\ast(s)$) because they are in the same layer. However, on a single state $s\in \mS_h$, all agents $i\in [m]$ must share the same $r^\ast(s)$. This is because the expectation in \Cref{eq:new gap requirement} is w.r.t. the opponents' policies. So if any agent deviates from the current $r$, all other agents will observe a different sequence of losses and thus break \Cref{eq:original gap requirement}.

Now we focus on a single state $s\in \mS_h$.
By \Cref{eq:new gap requirement}, any construction $r\in [R]$ ensures \textit{w.h.p.} that $V_{\tilde \pi_r}^i(s)\ge V_\ast^i(s)-\Gap_r^i(s)$.
Moreover, \Cref{eq:optimistic V-function} ensures $\hat V_r^i(s)\in [V_\ast^i(s),V_{\tilde \pi_r}^i(s)+2\Gap_r^i(s)]$.
So we only need to find a $r$ whose $\Gap_r^i(s)$ is ``small'' for all $i\in [m]$ -- more preciously, we want a $r^\ast(s)$ such that
\begin{equation*}
\sum_{i=1}^m \Gap_{r^\ast(s)}^i(s)\le 2\sum_{i=1}^m\E_{\textsc{Gap}}[\Gap^i(s)],\quad \forall s\in \mS_h,i\in [m].
\end{equation*}

From Markov inequality, for each $r\in [R]$, $\Pr\left \{\sum_{i=1}^m \Gap_r^i(s)>2\sum_{i=1}^m \E_{\textsc{Gap}}[\Gap^i(s)]\right \}\le \frac{1}{2}$. Thus, from the choice of $R=\O(\log \frac 1\delta)$, the $r$ minimizing $\sum_{i=1}^m \Gap_r^i(s)$ (\textit{i.e.}, the $r^\ast(s)$ defined in \Cref{eq:definition of r*}) ensures that $\sum_{i=1}^m \Gap_{r^\ast(s)}^i(s)\le 2\sum_{i=1}^m \E_{\textsc{Gap}}[\Gap^i(s)]$ with probability $1-\delta$.

As states in the same layer are independent in the sense that the policy on $s'$ doesn't affect the Q-function of $s$ if $s,s'\in \mS_h$, we can combine all $\{(\tilde \pi_{r^\ast(s)}(s),\Gap_{r^\ast(s)}(s))\}_{s\in \mS_h}$'s into $(\tilde \pi,\tilde \Gap)$. \Cref{eq:new gap requirement} then ensures Condition 1 of \Cref{lem:original main theorem}, and Condition 3 is also closely related with our \Cref{eq:new expectation of Gap}.
Besides the different \textsc{Gap}, the potential part is also slightly different from \Cref{eq:potential function}. This is because the original version is mainly tailored for uniform exploration policies -- as we will see in \Cref{sec:expectation of Gap main text}, we design a potential function similar to that of \citet{cui2023breaking} as we are both using $\bar \pi_t$ as the exploration policy in \textsc{CCE-Approx}.
\end{proof}

\section{Improved \textsc{CCE-Approx} Subroutine}
\label{sec:cce-approx}

\begin{algorithm}[!t]
\caption{Improved \textsc{CCE-Approx} Subroutine for Independent Linear Markov Games}
\label{alg:linear case}
\begin{algorithmic}[1]
\REQUIRE{Policy mixture $\bar \pi$, next-layer V-function $\bar V$, epoch length $K$, failure probability $\delta$.}
\ENSURE{A mixed policy $\tilde \pi\colon \mS_h\to \triangle(\mA)$ and a data-dependent \Gap estimation $\Gap\colon \mS_h \times \mathbb R^m$.}
\STATE Set learning rate $\eta$, bonus parameters $\beta_1,\beta_2$, and regularization parameter $\gamma=\frac{5d}{K}\log \frac{6d}{\delta}$.
\STATE All agents play $\bar \pi$ for $K$ times. For the $i$-th of them, it memorizes the state-action pair in the $h$-th layer as $\{(s_k^{\text{cov}},a_{k,i}^{\text{cov}})\}_{k=1}^K$ and calculates the estimated (inverse) covariance matrix as
\begin{equation}
\hat \Sigma_{t,i}^\dagger=\left (\frac 1K\sum_{\kappa=1}^K \bm \phi(s_{\kappa}^{\text{cov}},a_{\kappa,i}^{\text{cov}})\bm \phi(s_{\kappa}^{\text{cov}},a_{\kappa,i}^{\text{cov}})^\trans+\gamma I\right )^{-1},\label{eq:definition of Sigma hat}
\end{equation}
\STATE All agents play $\bar \pi$ for $K$ times. For the $i$-th of them, it memorizes the state-action pair in the $h$-th layer as $\{(s_k^{\text{mag}},a_{k,i}^{\text{mag}})\}_{k=1}^K$ and calculates the magnitude-reduced estimator \citep{dai2023refined} as
\begin{equation}\label{eq:definition of m hat}
\hat m_k^i(s,a)=\frac HK\sum_{\kappa=1}^K \left ({\bm \phi}(s,a)^\trans \hat \Sigma_{t,i}^\dagger {\bm \phi}(s_\kappa^\text{mag},\tilde a_{\kappa,i}^\text{mag})\right )_-,\quad \forall s\in \mS_h,a\in \mA^i.
\end{equation}
\FOR{$k=1,2,\ldots,K$}
\STATE Each agent $i\in [m]$ uses EXP3 to calculate its policy $\pi_k^i(a \mid s)$ for all $s\in \mS_h$ and $a\in \mA^i$:
\begin{align}
&{\pi_k^i(a \mid s)\propto \exp \left (-\eta \left (\sum_{\kappa=1}^{k-1} (\hat Q_\kappa^i(s,a)-B_k^i(s,a))\right )\right ),} \label{eq:exp3}
\end{align}
where $B_k^i(s,a)=\beta_1 \lVert \bm \phi(s,a)\rVert_{\hat \Sigma_{t,i}^\dagger}^2+\beta_2 \sum_{j=1}^d \bm \phi(s,a)[j] \times \sup_{(s',a')\in \mS_h\times \mA^i} (\hat \Sigma_{t,i}^\dagger \bm \phi(s',a'))[j]$, where $\cdot[j]$ means the $j$-th coordinate. Roughly, $\eta_1\approx 1/\sqrt K$, $\beta_1\approx \frac{dH}{\sqrt K}$, and $\beta_2\approx \frac HK$.
\FOR{$i=1,2,\ldots,m$}
\STATE The $i$-th agent plays $\bar \pi$ and any other agent $j\ne i$ plays $\bar \pi$ for first $h-1$ steps and $\pi_k^j$ for the $h$-th one. Agent $i$ records its observed states, actions, and losses as $(s_{k,\happa}^i,a_{k,\happa}^i,\ell_{k,\happa}^i)_{\happa=1}^{h+1}$.
\ENDFOR
\STATE Each agent $i\in [m]$ estimates the kernel of the Q-function induced by $\bar V^i$ and $\pi_k$ as
\begin{align}
\hat {\bm \theta}_k^i&=\hat \Sigma_{t,i}^\dagger ~ {\bm \phi}(s_{k,h}^i,a_{k,h}^i) ~ \left (\ell_{k,h}^i+\bar V^i(s_{k,h+1}^i)\right ).\label{eq:definition of theta hat}
\end{align}
\STATE Each agent $i\in [m]$ adopts magnitude-reduced estimators \citep{dai2023refined} to calculate
\begin{align}
\hat Q_k^i(s,a)={\bm \phi}(s,a)^\trans \hat {\bm \theta}_k^i-H\left ({\bm \phi}(s,a)^\trans \hat \Sigma_{t,i}^\dagger {\bm \phi}(s_{k,h}^i,a_{k,h}^i)\right )_-+\hat m_k^i(s,a).\label{eq:definition of Q hat}
\end{align}
\ENDFOR
\RETURN $(\tilde \pi,\Gap)$ where $\tilde \pi = \frac 1K\sum_{k=1}^K \pi_k$ and the data-dependent $\Gap$ is defined in \Cref{eq:definition of gap}.\label{line:return of CCE-Approx}
\end{algorithmic}
\end{algorithm}

To improve the sample complexity for MGs with independent linear function approximation, it is insufficient to change the framework alone.
This section introduces our improved implementation of the \textsc{CCE-Approx} subroutine, presented in \Cref{alg:linear case}.
While our framework remains mostly the same as \texttt{AVLPR}, the subroutine \textsc{CCE-Approx} is very different from \citep{wang2023breaking}. Here, we briefly overview several main technical innovations that we make in our \Cref{alg:linear case}.

\subsection{Magnitude-Reduction Loss Estimators}
In this regret-minimization task, we deploy EXP3 on each state, similar to \citet{cui2023breaking}.
Motivated by the recent progress in linear MDPs \citep{dai2023refined}, we aggressively set the regularization parameter in covariance-estimation (the $\gamma$ in \Cref{eq:definition of Sigma hat}) as $\Otil(K^{-1})$, instead of the usual choice of $\gamma=\Otil(K^{-1/2})$ \citep{cui2023breaking}. This is because the regret analyses shall exhibit factors like $\Otil(\frac{\gamma}{\beta_1} K + \frac{\beta_1}{\gamma} + \beta_1 K)$ -- to get $\Otil(\sqrt K)$ regret (which is necessary for $\epsilon^{-2}$ sample complexity), we must set $\beta_1=\Otil(K^{-1/2})$ and $\gamma=\Otil(K^{-1})$.

However, the downside of this aggressive tuning of $\gamma$ is that the estimated Q-function, namely $\bm \phi(s,a)^\trans \hat{\bm \theta}_k^i=\bm \phi(s,a)^\trans \hat \Sigma_{t,i}^\dagger \bm \phi(s_{k,h}^i,a_{k,h}^i) (\ell_{k,h}^i + \bar V^i(s_{k,h+1}^i))$, can lie anywhere in $[-\gamma^{-1},\gamma^{-1}]=[-\O(K),\O(K)]$. To comply with the EXP3 requirement that all loss estimators are at least $-\O(1/\eta)$ (\textit{\textit{c.f.}} \Cref{lem:exp3}) while still setting the learning rate $\eta$ as $\O(1/\sqrt K)$, we adopt the Magnitude-Reduced Estimator proposed by \citet{dai2023refined} to ``move'' the Q-estimators into range $[-\O(\sqrt K), \O(K)]$ by setting $\hat Q_k^i(s,a)={\bm \phi}(s,a)^\trans \hat {\bm \theta}_k^i-H ({\bm \phi}(s,a)^\trans \hat \Sigma_{t,i}^\dagger {\bm \phi}(s_{k,h}^i,a_{k,h}^i) )_-+\hat m_k^i(s,a)$ (as we did in \Cref{eq:definition of Q hat}). More details can be found in \Cref{lem:magnitude reduced}.

\subsection{Action-Dependent Bonuses}\label{sec:action-dependent bonus}
As we are using $\hat{\bm \theta}_k^i$ instead of the actual $\bm \theta_k^i$ when defining $\hat Q_k^i$ in \Cref{eq:definition of Q hat}, this incurs an \textit{estimation error}.
Focusing on a single state $s\in  \mS_h$, this estimation error occurs twice in the analysis as $\E_{a\sim \pi_k^i(\cdot \mid s)}[\text{Est Err}(s,a)]$ and $\E_{a\sim \pi_\ast^i(\cdot \mid s)}[\text{Est Err}(s,a)]$. These two terms are called \textsc{Bias-1} and \textsc{Bias-2} in our analysis (see \Cref{eq:reg decomposition main text}).

The former term of \textsc{Bias-1} is relatively easy to control: Because $\pi_k^i$ is known, as we elaborate in \Cref{sec:Gap is pessimistic main text}, we can use some variants of the Freedman Inequality to have it bounded.
For the latter term (\textsc{Bias-2}), due to the unknown $\pi_\ast^i$, the aforementioned approach no longer works. One common technique in the literature is to design \textit{bonuses}: Suppose that we'd like to cover $\sum_{k=1}^K v_k^i(s,a_\ast^i)$ where $a_\ast^i$ is the optimal action on $s$ in hindsight and $v_k^i$ is an abstract quantity (\textit{e.g.}, the \text{Est Err}), we then design bonuses $B_k^i(s,a)$ such that
\begin{equation}\label{eq:bonus condition}
\sum_{k=1}^K B_k^i(s,a_\ast^i)\gtrsim \sum_{k=1}^K v_k^i(s,a_\ast^i)~\textit{w.h.p.},\quad \forall s\in \mS_h,i\in [m].
\end{equation}

In this way, if we feed $\hat Q_k^i-B_k^i$ instead of $\hat Q_k^i$ into the EXP3 algorithm deployed at $s\in \mS_h$, we can roughly get (see \Cref{lem:exp3} for the original EXP3 guarantee and \Cref{lem:reg-term} for this form)
\begin{equation*}
\sum_{k=1}^K \sum_{a\in \mA^i} (\pi_k^i(a\mid s)-\pi_\ast^i (a\mid s)) (\hat Q_k^i(s,a)-B_k^i(s,a))\lesssim \frac{\log \lvert \mA^i\rvert}{\eta} + \sum_{k=1}^K \E_{a\sim \pi_k^i(\cdot \mid s)}\left [\eta \hat Q_k^i(s,a)^2 + B_k^i(s,a)\right ],
\end{equation*}
and conclude that (the second term of LHS also appears in the regret decomposition in \Cref{eq:reg decomposition main text})
\begin{align}
&\quad \sum_{k=1}^K \E_{a\sim \pi_\ast^i(\cdot \mid s)}[v_k^i(s,a)]+\sum_{k=1}^K \sum_{a\in \mA^i} (\pi_k^i(a\mid s)-\pi_\ast^i (a\mid s)) \hat Q_k^i(s,a) \nonumber\\
&\lesssim \sum_{k=1}^K \E_{a\sim \pi_k^i(\cdot \mid s)}\left [B_k^i(s,a)\right ]+\frac{\log \lvert \mA^i\rvert}{\eta} + \eta \sum_{k=1}^K \E_{a\sim \pi_k^i(\cdot \mid s)}\left [\hat Q_k^i(s,a)^2\right ],\label{eq:general bonus cancellation}
\end{align}
which means that we replaces the unknown $\sum_{k=1}^K v_k^i(s,a_\ast^i)$ with the known $\sum_{k=1}^K \E_{a\sim \pi_k^i(\cdot \mid s)}[B_k^i(s,a)]$.

A tradition way of designing $B_k^i(s,a)$ is using the classical Freedman Inequality \citep{freedman1975tail}, which roughly claims that if $\lvert v_k^i(s,a)\rvert\le V^i$ \textit{a.s.} for all $s\in \mS_h,a\in \mA^i,k\in [K]$, we have\footnote{The expectation of $(v_k^i(s,a_\ast^i))^2$ here should actually condition on the filtration $\mathcal F_{k-1}$ as $v_k^i$ is a stochastic process.}
\begin{equation*}
\sum_{k=1}^K v_k^i(s,a_\ast^i)\lesssim \sqrt{\sum_{k=1}^K \E[(v_k^i(s,a_\ast^i))^2]} + V^i \log \frac 1\delta~\textit{w.h.p.},\quad \forall s\in \mS_h,i\in [m].
\end{equation*}

In problems like linear bandits, $V^i$ is typically as small as $\Otil(\sqrt K)$ (see, \textit{e.g.}, \citep{zimmert2022return}). Therefore, directly picking $B_k^i(s,a)\gtrsim \sqrt{\E[(v_k^i(s,a))^2]}$ can ensure \Cref{eq:bonus condition} as the $V^i \log \frac 1\delta$ part can directly go into the regret bound.
However, in our case, due to the aggressive choice of $\gamma=\Otil(K^{-1})$, $V^i$ can be as huge as $\Otil(K)$ and such an approach fails.

To tackle this, we observe that if we find a function $V^i\colon \mS_h\times \mA^i\to \mathbb R^d$ such that
\begin{equation*}
\lvert v_k^i(s,a)\rvert\le V^i(s,a)~\textit{a.s.},~\forall s\in \mS_h,a\in \mA^i,k\in [K],\text{ and }\quad \frac 1K \sum_{k=1}^K \E_{a\sim \pi_k^i(\cdot \mid s)} [V^i(s,a)]\text{ is well-bounded},
\end{equation*}
we can set a \textit{\textbf{action-dependent bonus}} of
\begin{equation}\label{eq:action-dependent bonus definition}
B_k^i(s,a)\gtrsim \sqrt{\E[(v_k^i(s,a))^2\mid \mathcal F_{k-1}]}+\frac 1K V^i(s,a),\quad \forall s\in \mS_h,a\in \mA^i,k\in [K].
\end{equation}
Using the Adaptive Freedman Inequality (see \Cref{lem:adaptive Freedman}) given by \citet{zimmert2022return},
\begin{equation*}
\sum_{k=1}^K v_k^i(s,a_\ast^i)\lesssim V^i(s,a_\ast^i)+\sqrt{\sum_{k=1}^K \E[(v_k^i(s,a_\ast^i))^2]}~\textit{w.h.p.},
\end{equation*}
which infers the bonus condition of $\sum_{k=1}^K B_k^i(s,a_\ast^i)\ge \sum_{k=1}^K v_k^i(s,a_\ast^i)$ in \Cref{eq:bonus condition}. 
Meanwhile, the extra cost of $\sum_{k=1}^K \E_{a\sim \pi_k^i(\cdot \mid s)}[B_k^i(s,a)]$ (see the RHS of \Cref{eq:general bonus cancellation}) is still bounded because we assume $\frac 1K \sum_{k=1}^K \E_{a\sim \pi_k^i(\cdot \mid s)} [V^i(s,a)]$ can be controlled. Finally, when $V^i(s,a)=\Otil(K)$, we can also ensure that the bonus function $B_k^i(s,a)$ is small as $\frac 1K V^i(s,a)=\Otil(1)$, which is necessary as the EXP3 regret guarantee requires $\hat Q_k^i-B_k^i\gtrsim -\O(1/\eta)$.

In a nutshell, our \textit{action-dependent bonus} technique allows the estimation errors $v_k^i(s,a)$ to have more extreme values on those rarely-visited state-action pairs $(s,a)$ -- compared to the classical approach which requires $\lvert v_k^i(s,a)\rvert\le V^i=\Otil(\sqrt K)$ for all $s\in \mS_h,a\in \mA^i,k\in [K]$ uniformly, our $V^i(s,a)$ can have values of order $\Otil(K)$ occasionally on some state-action pairs, as long as $\frac 1K \sum_{k=1}^K \E_{a\sim \pi_k^i(\cdot \mid s)} [V^i(s,a)]$ remains small. We expect this technique to be useful in other problems where high-probability bounds are required.

\subsection{Covariance Matrix Estimation}
When analyzing linear regression, one critical step is to investigate the quality of ``covariance matrix estimation''. Typically, we would like to estimate the covariance matrix $\Sigma$ of a $d$-dimensional distribution $\mD$, namely $\hat \Sigma$.
Various approaches try to control either the additive error $\lVert \hat \Sigma-\Sigma\rVert_2$ \citep{neu2020efficient,luo2021policy} or the multiplicative error $\lVert \hat \Sigma(\gamma I+\Sigma)^{-1}\rVert_2$ \citep{dai2023refined,sherman2023improved}, but the convergence rate is at most $\Otil(n^{-1/4})$ where $n$ is the number of samples from $\mathcal D$.
Recently, \citet{liu2023bypassing} bypassed this limitation by considering $\text{Tr}\big (\hat \Sigma^{-1/2}(\hat \Sigma-\Sigma)\big )$ and gave an $\Otil(n^{-1/2})$ convergence rate; this technique is also adopted into our analysis for an $\Otil(\sqrt K)$ regret. See \Cref{sec:stochastic matrix concentration} for more details regarding this.

\subsection{Main Guarantee of \Cref{alg:linear case}}
Incorporating all these techniques, our \Cref{alg:linear case} ensures \Cref{eq:new gap requirement,eq:new expectation of Gap}. We make the follow claims, whose formal statements are presented in \Cref{lem:main theorem CCE-Approx appendix,lem:main theorem CCE-Approx second half}.
\begin{theorem}[\textsc{Gap} is \textit{w.h.p.} Pessimistic; Informal]\label{lem:main theorem CCE-Approx}
When \Cref{alg:linear case} is configured properly, for each execution of $\textsc{CCE-Approx}_h$, the condition in \Cref{eq:new gap requirement} holds with probability $1-\Otil(\delta)$, \textit{i.e.},
\begin{equation*}
\max_{\pi_\ast^i\in \Pi_h^i}\left \{\left (\E_{a\sim \tilde \pi}-\E_{a\sim \pi_\ast^i\times \tilde \pi^{-i}}\right )\left [\big (\ell^i+\P_{h+1}\bar V^i\big )(s,a)\right ]\right \} \le \Gap^i(s),\quad \forall i\in [m],s\in \mS_h \text{ with probability }1-\Otil(\delta).
\end{equation*}
\end{theorem}

\begin{theorem}[\Cref{alg:linear case} Allows a Potential Function; Informal]\label{lem:main theorem CCE-Approx second half main text}
Consider the following potential:
\begin{equation}\label{eq:definition of potential}
\Psi_{t,h}^i=\left .\left (\sum_{\tau=1}^t \lVert \bm \phi(\tilde s_{\tau,h},\tilde a_{\tau,h}^i)\rVert_{\hat \Sigma_{\tau,i}^\dagger}^2\right )\middle / \left (64 \log \frac{8 m H T}{\delta}\right )\right .,\quad \forall t\in [T],h\in [H],i\in [m],
\end{equation}
where $\hat \Sigma_{t,i}^\dagger$ is defined as \Cref{eq:definition of Sigma hat} if \Cref{line:condition} is violated in epoch $t$ and as $\hat \Sigma_{t,i}^\dagger = \hat \Sigma_{t-1,i}^\dagger$ otherwise (which is a recursive definition since \Cref{line:condition} may not be violated in epoch $(t-1)$ as well).
Then, when \Cref{alg:linear case} is configured properly, the condition in \Cref{eq:new expectation of Gap} holds with $L=\Otil(d^4H^2)$, \textit{i.e.},
\begin{equation*}
\sum_{t=1}^T \sum_{i=1}^m \E_{s\sim_h \tilde \pi_t} \left [\E_{\textsc{Gap}}[\Gap_t^i(s)]\right ]=\Otil(m d^2 H \sqrt T),\quad \text{with probability }1-\delta.
\end{equation*}
\end{theorem}

Putting our improved \textsc{CCE-Approx} subroutine in \Cref{alg:linear case} and the original \textsc{V-Approx} subroutine \citep[Algorithm 3]{wang2023breaking} together, we give our final algorithm for multi-player general-sum Markov Games with independent linear function approximations. As claimed, this algorithm enjoys only polynomial dependency on $m$, the optimal convergence rate $\epsilon^{-2}$, and avoids $\text{poly}(A_{\max})$ factors. Formally, we give the following theorem.
\begin{theorem}[Main Theorem of the Overall Algorithm]\label{lem:sample complexity}
Under proper configuration of parameters, by picking $\textsc{CCE-Approx}_h$ as \Cref{alg:linear case} and $\textsc{V-Approx}_h$ as \Cref{alg:V-approx} \citep[Algorithm 3]{wang2023breaking}, under the Markov Games with Independent Linear Function Approximation assumption (\Cref{assumption}), \Cref{alg:framework} enjoys a sample complexity of $\Otil(m^4 d^5 H^6 \epsilon^{-2})=\Otil(\text{poly}(m,d,H) \epsilon^{-2})$.
\end{theorem}

The formal proofs of \Cref{lem:main theorem CCE-Approx,lem:main theorem CCE-Approx second half main text,lem:sample complexity} are in \Cref{sec:appendix CCE-Approx,sec:expectation of Gap,sec:appendix overall}, respectively.
In the following section, we give an overview of the high-level idea of proving \Cref{lem:main theorem CCE-Approx,lem:main theorem CCE-Approx second half main text}.

\section{Analysis Outline of the Improved \textsc{CCE-Approx} Subroutine}
Analyzing the improved \textsc{CCE-Approx} subroutine is highly technical, and the full proof of \Cref{lem:main theorem CCE-Approx,lem:main theorem CCE-Approx second half main text} are deferred into \Cref{sec:appendix CCE-Approx,sec:expectation of Gap}. In this section, we highlight several key steps when verifying \Cref{eq:new gap requirement,eq:new expectation of Gap} in \Cref{lem:main theorem CCE-Approx,lem:main theorem CCE-Approx second half main text}.

\subsection{\textsc{Gap} is \textit{w.h.p.} Pessimistic (Proof Sketch of \Cref{lem:main theorem CCE-Approx})}\label{sec:Gap is pessimistic main text}
In this section, we fix an agent $i\in [m]$ and a state $s\in \mS_h$.
We need to verify \Cref{eq:new gap requirement} for this specific agent-state pair.
By the construction of $\tilde \pi$ in \Cref{line:return of CCE-Approx}, \Cref{eq:new gap requirement} is equivalent to
\begin{equation}\label{eq:condition of gap}
\sum_{k=1}^K \left \langle \pi_k^i(\cdot\mid s)-\pi_\ast^i(\cdot\mid s), {\bm \phi}(s,a)^\trans {\bm \theta}_k^i\right \rangle\le K\cdot \Gap^i(s),\quad \forall \pi_\ast^i\in \Pi^i,
\end{equation}
where ${\bm \theta}_k^i$ is the ``true'' kernel induced by $\pi_k$ and the next-layer V-function $\bar V^i$, \textit{i.e.},
\begin{equation*}
{\bm \phi}(s,a)^\trans {\bm \theta}_k^i=\E_{\bm a^{-i}\sim \pi_{k}^{-i}(\cdot \mid s)}\left [(\ell^i+\P \bar V^i)(s,\bm a)\right ],\quad \forall s\in \mS_h,a\in \mA^i.
\end{equation*}

By (almost) standard regret decomposition, the LHS of \Cref{eq:condition of gap} can be rewritten as
\begin{align}
&\quad \sum_{k=1}^K \left \langle \pi_k^i(\cdot\mid s)-\pi_\ast^i(\cdot\mid s), {\bm \phi}(s,a)^\trans {\bm \theta}_k^i\right \rangle=\underbrace{\sum_{k=1}^K \left \langle \pi_k^i(\cdot \mid s)-\pi_\ast^i(\cdot \mid s),\hat Q_k^i(s,\cdot)-B_k^i(s,\cdot)\right \rangle}_{\textsc{Reg-Term}}+\nonumber\\
&\quad \underbrace{\sum_{k=1}^K \left \langle \pi_k^i(\cdot \mid s)-\pi_\ast^i(\cdot \mid s),{\bm \phi}(s,a)^\trans \hat {\bm \theta}_k^i - \hat Q_k^i(s,a)\right \rangle}_{\textsc{Mag-Reduce}} + \underbrace{\sum_{k=1}^K \left \langle \pi_k^i(\cdot \mid s),{\bm \phi}(s,\cdot)^\trans ({\bm \theta}_k^i-\hat {\bm \theta}_k^i)\right \rangle}_{\textsc{Bias-1}}+ \nonumber\\
&\quad \underbrace{\sum_{k=1}^K \left \langle \pi_\ast^i(\cdot \mid s),{\bm \phi}(s,\cdot)^\trans (\hat {\bm \theta}_k^i-{\bm \theta}_k^i)\right \rangle}_{\textsc{Bias-2}}+ \underbrace{\sum_{k=1}^K \left \langle \pi_k^i(\cdot \mid s)-\pi_\ast^i(\cdot \mid s),B_k^i(s,\cdot)\right \rangle}_{\textsc{Bonus-1}\,(\pi_k^i)~-~\textsc{Bonus-2}\,(\pi_\ast^i)}. \label{eq:reg decomposition main text}
\end{align}

In the remaining part of this section, we overview the main steps in controlling these terms.

\subsubsection{Controlling \textsc{Reg-Term} via Magnitude-Reduced Estimator}
The \textsc{Reg-Term} is the regret \textit{w.r.t.} estimated losses fed into the EXP3 instance deployed on each state $s$, stated as follows:
\begin{equation*}
\textsc{Reg-Term}=\sum_{k=1}^K \sum_{a\in \mA^i} (\pi_k^i(a\mid s)-\pi_\ast^i(a\mid s)) \big (\hat Q_k^i(s,a)-B_k^i(s,a)\big ).
\end{equation*}

As we feed $\hat Q_k^i(s,a)-B_k^i(s,a)$ into the EXP3 procedure on each $s\in \mS_h$ (see \Cref{eq:exp3}), this term usually can be directly controlled using the EXP3 regret guarantee stated in \Cref{lem:exp3}, which roughly says
\begin{align*}
&\textsc{Reg-Term} \le \frac{\log \lvert \mA^i\rvert}{\eta} + \eta \sum_{k=1}^K \sum_{a\in \mA^i} \pi_k^i(a\mid s) \big (\hat Q_k^i(s,a)-B_k^i(s,a)\big )^2,\\
&\text{if }\hat Q_k^i(s,a)-B_k^i(s,a)\ge -\frac 1 \eta,~\forall k\in [K],a\in \mA^i.
\end{align*}

To verify $\hat Q_k^i(s,a)-B_k^i(s,a)\ge -\frac 1 \eta$, people usually control the estimated $\hat Q_k^i$ via \citep{luo2021policy}
\begin{equation}\label{eq:magnitude of Q function}
\lvert \bm \phi(s,a)^\trans \hat{\bm \theta}_k^i\rvert = \bm \phi(s,a)^\trans \hat \Sigma_{t,i}^\dagger \bm \phi(s_{k,h}^i,a_{k,h}^i) (\ell_{k,h}^i + \bar V^i(s_{k,h+1}^i)) \lesssim H \lVert \hat \Sigma_{t,i}^\dagger\rVert_2\le H \gamma^{-1}.
\end{equation}

While this works when $\gamma=\Otil(K^{-1/2})$, it becomes prohibitively large when setting $\gamma=\Otil(K^{-1})$.
Fortunately, thanks to the Magnitude-Reduced Estimator by \citet{dai2023refined}, we can ensure that $\hat Q_k^i-B_k^i\ge \hat m_k^i-B_k^i$ is bounded from below by $\hat m_k^i\ge -\Otil(K^{-1/2})$ \citep[Theorem 4.1]{dai2023refined}.
As $B_k^i$ is also small, we can apply the standard EXP3 guarantee in \Cref{lem:exp3} with $\eta=\Otil(K^{-1/2})$.

\subsubsection{Controlling \textsc{Bias-1} Term via Adaptive Freedman Inequality}
The \textsc{Bias-1} term $\sum_k \langle \pi_k^i(\cdot \mid s),{\bm \phi}(s,\cdot)^\trans ({\bm \theta}_k^i-\hat {\bm \theta}_k^i)\rangle$ corresponds to the cost of only knowing $\hat {\bm \theta}_k^i$ instead of ${\bm \theta}_k^i$ when estimating $Q_k^i$.
As $\hat \Sigma_{t,i}^\dagger$ is biased because of the $\gamma I$ regularizer in \Cref{eq:definition of Sigma hat}, $({\bm \theta}_k^i-\hat{\bm \theta}_k^i)$ can be further decomposed into $({\bm \theta}_k^i-\E[\hat{\bm \theta}_k^i])$ and $(\E[\hat{\bm \theta}_k^i]-\hat{\bm \theta}_k^i)$, namely the intrinsic bias and estimation error of $\hat{\bm \theta}_k^i$:
\begin{equation*}
\textsc{Bias-1}=\sum_{k=1}^K \E_{a\sim \pi_k^i(\cdot \mid s)} \left [{\bm \phi}(s,a)^\trans\right ] ({\bm \theta}_k^i-\hat {\bm \theta}_k^i)=\underbrace{\sum_{k=1}^K \left (\E_{a\sim \pi_k^i(\cdot \mid s)} \left [{\bm \phi}(s,a)^\trans\right ] {\bm \theta}_k^i-\mu_k\right )}_{\text{Intrinsic Bias}}+\underbrace{\sum_{k=1}^K (\mu_k-X_k)}_{\text{Estimation Error}},
\end{equation*}
where $X_k=\E_{a\sim \pi_k^i(\cdot \mid s)}\left [{\bm \phi}(s,a)^\trans\right ]\hat {\bm \theta}_k^i$ and $\mu_k=\E[X_k\mid \mathcal F_{k-1}]$; here, $(\mathcal F_k)_{k=0}^K$ is the natural filtration.

The intrinsic bias term is standard in analyzing (expected-)regret-minimization algorithms for linear MDPs -- see, \textit{e.g.}, Lemma D.2 of \citet{luo2021policy} -- and can thus be controlled analogously.

The estimation error term is a martingale. Consequently, regret-minimization papers in single-agent RL usually omit it by its zero-mean nature. However, it plays an important role here since $\Gap$ must be a high-probability upper bound. Existing high-probability results for adversarial linear bandits usually apply the famous Freedman inequality \citep{freedman1975tail}: Suppose that $\lvert X_k-\mu_k\rvert\le X$ \textit{a.s.}, then
\begin{equation*}
\sum_{k=1}^K (X_k-\mu_k)\lesssim \sqrt{\sum_{k=1}^K \E[(X_k-\mu_k)^2 \mid \mathcal F_{k-1}]} + X \log \frac 1\delta,\quad \text{with probability }1-\delta.
\end{equation*}

However, as we calculated in \Cref{eq:magnitude of Q function}, due to our choice of $\gamma=\Otil(K^{-1})$, $X$ must be of order $\Otil(HK)$. Hence, a direct application of the traditional Freedman inequality results in a factor of order $\Otil(K)$, which is unacceptable.

Fortunately, as our strengthened \Cref{lem:new main theorem} only requires a stochastic $\Gap$ with bounded expectation, we are allowed to have a data-dependent variant of the $X\log \frac 1\delta$ term. Inspired by the Adaptive Freedman Inequality proposed by \citet{lee2020bias} and improved by \citet{zimmert2022return}, we prove a variant of Freedman Inequality in \Cref{lem:more adaptive Freedman} which roughly reads
\begin{equation}
\sum_{k=1}^K (X_k-\mu_k)\lesssim \sqrt{\sum_{k=1}^K (\E[X_k^2\mid \mathcal F_{k-1}]+X_k^2)}\log \frac 1\delta,\quad \text{with probability }1-\delta.\label{eq:more adaptive Freedman main text}
\end{equation}

As both $\E[X_k^2\mid \mathcal F_{k-1}]$ and $X_k^2$ is known to the learner during execution, we can directly put them into $\textsc{Gap}_k^i(s)$ and ensure \Cref{eq:new gap requirement}. The more detailed calculation can be found in the proof of \Cref{lem:Bias-1}.

\subsubsection{Cancelling \textsc{Bias-2} Using \textsc{Bonus-2}}
While the \textsc{Bias-2} term $\sum_k \langle \pi_\ast^i(\cdot \mid s), {\bm \phi}(s,\cdot)^\trans (\hat{\bm \theta}_k^i-{\bm \theta}_k^i)\rangle$ looks almost the same as \textsc{Bias-1}, we cannot directly put the RHS of \Cref{eq:more adaptive Freedman main text} into $\Gap_k^i(s)$ as $\pi_\ast^i$ is unknown. Following the classical idea of designing bonuses, we can make use of the \textsc{Bonus-2} term, \textit{i.e.}, $\textsc{Bonus-2}=-\sum_{k=1}^K \langle \pi_\ast^i(\cdot \mid s), B_k^i(s,\cdot)\rangle$, to cancel \textsc{Bias-2}.

Although bonuses are standard in the literature, we shall remark again that our \textit{action-dependent bonus} $B_k^i(s,a)$ is novel. Following the notations in \Cref{sec:action-dependent bonus}, we would like to cover
\begin{equation*}
v_k^i(s,a)=\bm \phi(s,a)^\trans (\hat{\bm \theta}_k^i-\E[\hat{\bm \theta}_k^i])
\end{equation*}
using $B_k^i(s,a)$ such that \Cref{eq:general bonus cancellation} happens.
Because $\gamma\approx K^{-1}$, each martingale difference term $v_k^i(s,a)$ can be of order $\O(K)$ -- thus, directly picking $V^i=\max_{k,s,a} \lvert v_k^i(s,a)\rvert$ would make \textsc{Gap} too large, which means the classical Freedman inequality approach described in \Cref{sec:action-dependent bonus} again fails.

To tackle this issue, as motivated in \Cref{sec:action-dependent bonus}, we introduce an action-dependent $V^i(s,a)=\max_k \lvert v_k^i(s,a)\rvert$ and average it into the $K$ episodes, resulting in the bonus definition of (the $\beta_1$-related term corresponds to the $\sqrt{\E[(v_k^i(s,a))^2\mid \mathcal F_{k-1}]}$ part, while the $\beta_2$-related term corresponds to the $\frac 1K V^i(s,a)$ part in \Cref{eq:action-dependent bonus definition})
\begin{equation*}
B_k^i(s,a)=\beta_1 \lVert \bm \phi(s,a)\rVert_{\hat \Sigma_{t,i}^\dagger}^2+\beta_2 \sum_{j=1}^d \bm \phi(s,a)[j] \times \sup_{(s',a')\in \mS_h\times \mA^i} (\hat \Sigma_{t,i}^\dagger \bm \phi(s',a'))[j].
\end{equation*}

Thus, the \textsc{Bias-2} term is indeed cancelled by \textsc{Bonus-2}. Moreover, let us discuss a little bit why this approach is more favorable when proving \Cref{lem:main theorem CCE-Approx second half main text}: Regarding the $\sum_{k=1}^K \E_{a\sim \pi_k^i(\cdot \mid s)}[B_k^i(s,a)]$ term in \Cref{eq:general bonus cancellation}, we no longer need to suffer from $V^i=\Otil(K)$, but only $\frac 1K\sum_k \E_{\pi_k^i} [V^i(s,a)]$ instead. As (see \Cref{lem:Bonus-1 Expectation})
\begin{equation*}
V^i(s,a)=\sum_{j=1}^d \bm \phi(s,a)[j]\times \sup_{(s',a')\in \mS_h\times \mA^i} (\hat \Sigma_{t,i}^\dagger \bm \phi(s',a'))[j]\le \lVert \bm \phi(s,a)\rVert_{\hat \Sigma_{t,i}^\dagger} \times \sqrt{\lVert \hat \Sigma_{t,i}^\dagger\rVert_2},
\end{equation*}
we have $\frac 1K\sum_k \E_{\pi_k^i} [V^i(s,a)]\lesssim \sqrt{\langle \E_{\frac 1K\sum_k \pi_k^i}[\bm \phi \bm \phi^\trans],\hat \Sigma_{t,i}^\dagger\rangle \times \lVert \hat \Sigma_{t,i}\rVert_2}$. Using the arguments presented shortly in \Cref{sec:expectation of Gap main text}, we can see this term is indeed nicely bounded.

To conclude, \textit{action-dependent bonuses} allow us to cancel a random variable that can exhibit extreme values but has small expectations. We expect this technique to be of independent interest.

\subsubsection{Directly Putting \textsc{Bonus-1} into \textsc{Gap}}
The \textsc{Bonus-1} term is easy to handle in designing $\textsc{Gap}_k^i$ as $\pi_k^i$ and $B_k^i$ are both known. Still, to ensure \Cref{eq:new expectation of Gap}, we need to control $\frac 1K\sum_k \E_{\pi_k^i} [V^i(s,a)]$ -- the same arguments from the previous paragraph then work. 

\subsubsection{Controlling \textsc{Mag-Reduce} Like \textsc{Bias-1} and \textsc{Bias-2}}
This term is a unique challenge in our paper, which is due to the Magnitude-Reduced Estimator by \citet{dai2023refined}: Since the $\hat Q_k^i$ in \Cref{eq:definition of Q hat} differs from $\bm \phi^\trans \hat {\bm \theta}_k^i$, the following term shows up:
\begin{equation*}
\textsc{Mag-Reduce}=\sum_{k=1}^K \sum_{a\in \mA^i} \left (\pi_k^i(a\mid s)-\pi_\ast^i(a\mid s)\right )\left (\bm \phi(s,a)^\trans \hat {\bm \theta}_k^i - \hat Q_k^i(s,a)\right ).
\end{equation*}

However, this term vanishes in the original paper because $\hat Q_k^i$ is unbiased (see \Cref{lem:magnitude reduced}) and \citet{dai2023refined} studied expected-regret-minimization.
Fortunately, because $(s_{\kappa}^{\text{mag}},a_{\kappa,i}^{\text{mag}})$ in $\hat m_k^i(s,a)$ (see \Cref{eq:definition of m hat}) and $(s_{k,h},a_{k,h}^i)$ in $\hat Q_k^i(s,a)$ (see \Cref{eq:definition of Q hat}) are \textit{i.i.d.}, we can decompose $\textsc{Mag-Reduce}$ into the sum of a few martingales and apply the Adaptive Freedman Inequality to each of them.
Informally, we observe that the resulting concentrations is smaller than those of \textsc{Bias-1} and \textsc{Bias-2}, and thus \textsc{Mag-Reduce} is automatically controlled -- more detailed arguments are presented in the proof of \Cref{lem:mag-reduce}.

\subsection{Controlling the Sum of $\E[\Gap]$'s via Potential Function}\label{sec:expectation of Gap main text}
The potential function $\Psi_{t,h}^i$ is defined in \Cref{eq:definition of potential}, presented below for the ease of readers:
\begin{equation*}
\Psi_{t,h}^i=\left .\left (\sum_{\tau=1}^t \lVert \bm \phi(\tilde s_{\tau,h},\tilde a_{\tau,h}^i)\rVert_{\hat \Sigma_{\tau,i}^\dagger}^2\right )\middle / \left (64 \log \frac{8 m H T}{\delta}\right )\right ..
\end{equation*}

Below we briefly explain why it ensures \Cref{lem:main theorem CCE-Approx second half main text}. After taking expectation to the randomness in $\Gap_t^i(s)$, we would get (omitting all dependencies on $d$ and $H$; see \Cref{lem:expectation of Gap} for formal statements and calculations)
\begin{equation*}
\E_{\Gap}[\Gap_t^i(s)]\lesssim \frac{1}{\sqrt t} \E_{a\sim \tilde \pi_t^i(s)}\big [\lVert \bm \phi(s,a)\rVert_{\hat \Sigma_{t,i}^\dagger}^2\big ],\quad \forall s\in \mS_h,i\in [m],t\in [T].
\end{equation*}

Therefore, the definition of $\Psi_{t,h}^i$ ensures that if $\Psi_{t,h}^i$ is not so different from some $\Psi_{t_0,h}^i$, we can pretend that $\E_{\Gap}[\Gap_{t_0}^i(s)]$ is also similar to $\E_{\Gap}[\Gap_t^i(s)]$, \textit{i.e.}, setting $\Gap_t\gets \Gap_{t_0}$ will not violate the above inequality by a lot. Hence, suppose that the ``lazy'' update mechanism re-uses $\tilde \pi_t$ for $n_t$ times (\textit{i.e.}, \Cref{line:condition} is not violated from epoch $(t+1)$ until $(t+n_t-1)$ and the $\tilde \pi_t$'s and $\Gap_t$'s will be the same). We roughly have
\begin{align}
\text{\Cref{eq:new expectation of Gap}}&=\sum_{t=1}^T \sum_{i=1}^m \E_{s\sim_h \bar \pi_t} \left [\E_{\textsc{Gap}}\left [\Gap_t^i(s)\right ]\right ] \lesssim \sqrt T \sum_{i=1}^m \sum_{t=1}^T \E_{s\sim_h \tilde \pi_t} \left [\E_{a\sim \tilde \pi_t^i(s)}\left [\lVert \bm \phi(s,a) \rVert_{\hat \Sigma_{t,i}^\dagger}^2\right ]\right ] \nonumber \\
&=\sqrt T\sum_{i=1}^m \sum_{t=1}^T \langle \text{Cov}_h^i(\tilde \pi_t), \hat \Sigma_{t,i}^\dagger\rangle = \sqrt T\sum_{i=1}^m \sum_{t=1}^T n_t \langle \text{Cov}_h^i(\tilde \pi_t), \hat \Sigma_{t,i}^\dagger\rangle,\label{eq:matrix summation lemma main text}
\end{align}
where $\text{Cov}_h^i(\tilde \pi_t)$ is the covariance matrix of agent $i$ in layer $h$ when following $\tilde \pi_t$, \textit{i.e.},
\begin{equation*}
\text{Cov}_h^i(\tilde \pi_t)=\E_{s\sim_h \tilde \pi_t}\left [\E_{a\sim \tilde \pi_t^i(s)}[\bm \phi(s,a)\bm \phi(s,a)^\trans]\right ].
\end{equation*}

By the definition of $\bar \pi_t=\frac 1t\sum_{s=0}^{t-1} \tilde \pi_s$, we would roughly have $\hat \Sigma_{t,i}^\dagger \approx (\sum_{s=0}^{t-1} n_s \text{Cov}_h^i(\tilde \pi_s)+\gamma I)^{-1}$. Therefore,
\begin{equation*}
\sum_{t=1}^T n_t \langle \text{Cov}_h^i(\tilde \pi_t), \hat \Sigma_{t,i}^\dagger\rangle\approx \sum_{t=1}^T \left \langle n_t \text{Cov}_h^i(\tilde \pi_t), \left (\sum_{s<t} n_s \text{Cov}_h^i(\tilde \pi_s)\right )^{-1}\right \rangle.
\end{equation*}

Recall that the scalar version of this sum will give $\sum_t x_t / (\sum_{s<t} x_s)=\Otil(\sum_t x_t)$ if $x_t$ is bounded, one can imagine that this sum should also be small.
Indeed, from Lemma 11 of \citet{zanette2022stabilizing}, we can conclude that
\begin{equation*}
\sum_{t=1}^T n_t \langle \text{Cov}_h^i(\tilde \pi_t), \hat \Sigma_{t,i}^\dagger\rangle\lesssim \Otil\left (\log \det \left (\sum_t n_t \text{Cov}_h^i(\tilde \pi_t)\right )\right )=\Otil(d),
\end{equation*}
and thus $\text{\Cref{eq:new expectation of Gap}}=\text{\Cref{eq:matrix summation lemma main text}}\lesssim \Otil(m\sqrt{T})$ (again, omitting all dependencies on $d$ and $H$). \Cref{lem:main theorem CCE-Approx second half main text} then follows.

\section{Conclusion}

In this paper, we consider multi-player general-sum Markov Games with independent linear function approximations.
By enhancing the \texttt{AVLPR} framework recently proposed by \citet{wang2023breaking} with a \textit{data-dependent} pessimistic gap estimation, proposing novel \textit{action-dependent bonuses}, and incorporating several state-of-the-art techniques from the recent advances in the adversarial RL literature \citep{zimmert2022return,dai2023refined,liu2023bypassing}, we give the first algorithm that \textit{i)} bypasses the curse of multi-agents, \textit{ii)} attains optimal convergence rate, and \textit{iii)} avoids polynomial dependencies on the number of actions.
We \textit{a)} design data-dependent pessimistic sub-optimality gap estimations (\Cref{sec:framework}), and \textit{b)} propose an action-dependent bonus technique to cover extreme estimation errors (\Cref{sec:action-dependent bonus}), which can be of independent interest.

\subsection*{Acknowledgement}
We thank Haipeng Luo (University of Southern California), Chen-Yu Wei (University of Virginia), and Zihan Zhang (University of Washington) for their insightful discussions.
We thank \citet{fan2024rl} for pointing out a mistake in the original proof of \Cref{lem:Gap concentration}.
We greatly acknowledge the anonymous reviewers for their comments.
SSD acknowledges the support of NSF IIS 2110170, NSF DMS 2134106, NSF CCF 2212261, NSF
IIS 2143493, NSF CCF 2019844, NSF IIS 2229881. 

\bibliography{references}

\onecolumn
\newpage
\appendix
\renewcommand{\appendixpagename}{\centering \LARGE Supplementary Materials}
\appendixpage

\startcontents[section]
\printcontents[section]{l}{1}{\setcounter{tocdepth}{2}}

\section{Analysis of the Improved \textsc{AVLPR} Framework}
\label{sec:appendix main theorem}

This section proves the main theorem of our Improved \texttt{AVLPR} framework, restated as follows.
\begin{theorem}[Main Theorem of Improved \texttt{AVLPR}; Restatement of \Cref{lem:new main theorem}]\label{lem:new main theorem appendix}
Suppose that
\begin{enumerate}
    \item (Per-state no-regret) $(\tilde \pi,\mathsc{Gap})=\mathsc{CCE-Approx}_h(\bar \pi, \bar V, K)$ ensures the following \textit{w.p.} $1-\delta$:
    \begin{align}
    \max_{\pi_\ast^i\in \Pi_h^i}\left \{\left (\E_{a\sim \tilde \pi}-\E_{a\sim \pi_\ast^i\times \tilde \pi^{-i}}\right )\left [\big (\ell^i+\P_{h+1}\bar V^i\big )(s,a)\right ]\right \} \le \Gap^i(s),\quad \forall i\in [m], s\in \mS_h,\label{eq:new gap requirement appendix}
    \end{align}
    where $\Gap^i(s)$ is a \textbf{\color{violet}random variable} whose randomness comes from the environment (when generating the trajectories), the agents (when playing the policies), and internal randomness.
    \item (Optimistic V-function) $\hat V=\mathsc{V-Approx}_h(\bar \pi,\tilde \pi,\bar V,\Gap,K)$ ensures the following \textit{w.p.} $1-\delta$:
    \begin{align}
    \hat V^i(s)\in \bigg [\min\bigg \{&\E_{a\sim \tilde \pi}\left [\big (\ell^i+\P\bar V^i\big )(s,a)\right ]+\phantom{2}\Gap^i(s),H-h+1\bigg \},\nonumber\\
    &\E_{a\sim \tilde \pi}\left [\big (\ell^i+\P\bar V^i\big )(s,a)\right ]+2\Gap^i(s)\bigg ],\quad \forall i\in [m], s\in \mS_h.\label{eq:optimistic V-function appendix}
    \end{align}
    \item (Pigeon-hole condition \& Potential Function)
    The potential function $\Psi_{t,h}^i$ ensures that \Cref{line:condition} in \Cref{alg:framework} is violated for at most $d_{\text{replay}} \log T$ times. Moreover, there exists a deterministic $L$ ensuring the following \textit{w.p.} $1-\delta$:
    \begin{equation}\label{eq:new expectation of Gap appendix}
    \sum_{t=1}^T \sum_{i=1}^m \E_{s\sim_h \tilde \pi_t} \left [{\color{violet}\E_{\Gap}[\Gap_t^i(s)]}\right ]\le m \sqrt{LT \log^2 \frac T\delta},\quad \forall h\in [H],
    \end{equation}
    where the expectation is taken \textit{w.r.t.} the randomness in \textsc{Gap}.
\end{enumerate}
Then with probability $1-\Otil(\delta)$, the sum of \textsc{CCE-Regret} over all agents satisfies
\begin{equation*}
\sum_{i=1}^m \sum_{t=1}^T \left (V_{\tilde \pi_t}^i(s_1)-V_{\dagger,\tilde \pi_t^{-i}}^i(s_1)\right )=\Otil(mH\sqrt{LT}),
\end{equation*}
where $\tilde \pi_t$ is the policy for the $t$-th epoch (which is the policy with smallest \textsc{Gap}; see \Cref{eq:definition of r*}).

Suppose that one call to $\textsc{CCE-Approx}_h$ will cost $\Gamma_1 K$ samples ($K$ is the last parameter to the subroutine), and, similarly, one call to $\textsc{V-Approx}_h$ will cost $\Gamma_2$ samples. Then by picking $T=\Otil(m^2H^2 L / \epsilon^2)$, the roll-out policy is an $\epsilon$-CCE and the sample complexity is bounded by
\begin{equation*}
\Otil\left (\max\{\Gamma_1,\Gamma_2\}\times m^2 H^3 L \epsilon^{-2}\times d_{\text{replay}}\log T\right ).
\end{equation*}
\end{theorem}
\begin{proof}
The proof mostly inherits the proof of Theorem 18 from \citep{wang2023breaking}; all the main differences are sketched in the proof sketch of \Cref{lem:new main theorem} (in the main text).

Let $\mI$ be the iterations where \Cref{line:condition} is violated.
For any $t\in \mI$ and $h=H,H-1,\ldots,1$, we first pass a ``next-layer'' V-function $\bar V^i$ to $\textsc{CCE-Approx}_h$ and then calculate the ``current-layer'' V-function via $\textsc{V-Approx}_h$.
By \Cref{eq:new gap requirement appendix} of Condition 1 and \Cref{eq:optimistic V-function appendix} of Condition 2, for any ``next-layer'' $\bar V^i$,
\begin{align*}
&\quad \sum_{i=1}^m \min_{\pi_\ast^i\in \triangle(\mA^i)}\left \{\E_{\bm a\sim \pi_\ast^i\times \tilde \pi_{r^\ast(s)}^{-i}(\cdot \mid s)}\left [\big (\ell^i+\P \bar V^i\big )(s,\bm a)\right ]\right \}\\
&\ge \sum_{i=1}^m \E_{\bm a\sim \tilde \pi_{r^\ast(s)}(\cdot \mid s)}\left [\big (\ell^i+\P \bar V^i\big )(s,\bm a)\right ]-\sum_{i=1}^m \Gap_{r^\ast(s)}^i(s)\\
&\ge \sum_{i=1}^m \hat V^i(s),\quad \forall s\in \mS_h.
\end{align*}

Thus, by induction over all $h=H,H-1,\ldots,1$, we know that for all $t\in \mI$ and $h\in [H]$,
\begin{equation*}
\sum_{i=1}^m \hat V_t^i(s)\le \sum_{i=1}^m V_{\dagger,\pi_t^{-i}}^i(s),\quad \forall s\in \mS.
\end{equation*}
In words, this indicates that our $\hat V_t^i$ is indeed an optimistic estimation of the best-response V-function.
Again applying induction by invoking the other part of \Cref{eq:optimistic V-function appendix}, we know for all $t\in \mI$,
\begin{equation*}
\sum_{i=1}^m \hat V_t^i(s)\ge \sum_{i=1}^m V_{\tilde \pi_t}^i(s)-2\sum_{i=1}^m \sum_{\happa=h}^H \E_{s'\sim_\happa \tilde \pi_t\mid s_h = s} \left [\Gap_t^i(s')\right ],\quad \forall s\in \mS.
\end{equation*}

Putting these two inequalities together, the following holds for all $t\in \mI$:
\begin{equation*}
\sum_{i=1}^m \left (V_{\tilde \pi_t}^i(s)-V_{\dagger,\tilde \pi_t^{-i}}^i(s)\right )\le 2\sum_{i=1}^m \sum_{\happa=h}^H \E_{s'\sim_\happa \tilde \pi_t\mid s_h=s} \left [\Gap_t^i(s')\right ],\quad \forall s\in \mS.
\end{equation*}

Hence, by denoting $I_t\in \mI$ as the policy where the $t$-th epoch is using (\textit{i.e.}, the last time where \Cref{line:condition} is violated). By above calculation, we can write the sum of \textsc{CCE-Regret} over all agents as
\begin{align*}
\sum_{i=1}^m \sum_{t=1}^T \left (V_{\tilde \pi_t}^i(s_1)-V_{\dagger,\tilde \pi_t^{-i}}^i(s_1)\right )
=\sum_{i=1}^m \sum_{t=1}^T \left (V_{\tilde \pi_{I_t}}^i(s_1)-V_{\dagger,\tilde \pi_{I_t}^{-i}}^i(s_1)\right )
\le 2\sum_{i=1}^m \sum_{t=1}^T \sum_{\happa=1}^H \E_{s'\sim_\happa \tilde \pi_{I_t}} \left [\Gap_{I_t}^i(s')\right ].
\end{align*}

From \Cref{lem:Gap concentration}, we conclude that the following holds \textit{w.p.} $1-\Otil(\delta T H)$:
\begin{equation*}
\sum_{i=1}^m \Gap_t^i(s)\le 2\E_{\Gap}\left [\sum_{i=1}^m \Gap_t^i\right ],\quad \forall t\in \mI, h\in [H], s\in \mS_h.
\end{equation*}

Moreover, recall from \Cref{eq:new expectation of Gap appendix} that
\begin{equation*}
\sum_{t=1}^T \sum_{i=1}^m \E_{s\sim_h \tilde \pi_t} \left [{\E_{\Gap}[\Gap_t^i(s)]}\right ]\le m \sqrt{LT \log^2 \frac T\delta},\quad \forall h\in [H].
\end{equation*}

We then have
\begin{equation*}
\sum_{i=1}^m \sum_{t=1}^T \left (V_{\tilde \pi_t}^i(s_1)-V_{\dagger,\tilde \pi_t^{-i}}^i(s_1)\right )\le 4 m H \sqrt{LT \log^2 \frac T\delta},
\end{equation*}
which means that our choice of $T=\Otil(m^2H^2L/\epsilon^2)$ indeed roll-outs an $\epsilon$-CCE.

It only remains to calculate the total sample complexity. From Condition 3, $\lvert \mI\rvert\le d_{\text{replay}}\log T$. Thus, the total sample complexity is no more than
\begin{equation*}
T+\lvert \mI\rvert\times H\Otil(\Gamma_1T+\Gamma_2T)=\Otil(\max\{\Gamma_1,\Gamma_2\}\times m^2 H^3 L \epsilon^{-2}\times d_{\text{replay}}\log T),
\end{equation*}
as claimed.
\end{proof}

\begin{lemma}\label{lem:Gap concentration}
For any $t\in \mI$, \textit{w.p.} $1-\delta$, we have $\sum_{i=1}^m \Gap_t^i(s)\le 2\E[\sum_{i=1}^m \Gap_t^i(s)]$, $\forall s\in \mS_h$.
\end{lemma}
\begin{proof}
Denote the $R$ policy-\textsc{Gap} pairs yielded as $(\tilde \pi_1,\Gap_1),(\tilde \pi_2,\Gap_2),\ldots,(\tilde \pi_R,\Gap_R)$. By definition, $(\tilde \pi_t(s),\Gap_t(s))=(\tilde \pi_{r^\ast(s)}(s),\Gap_{r^\ast(s)}(s))$ where $r^\ast(s)$ is defined as (in \Cref{eq:definition of r*})
\begin{equation*}
r^\ast(s)=\argmin_{r\in [R]}\sum_{i=1}^m \Gap_r^i(s).
\end{equation*}

By Markov inequality, for each $r\in [R]$, $\Pr\left \{\sum_{i=1}^m \Gap_r^i(s)>2\sum_{i=1}^m \E[\Gap^i(s)]\right \}\le \frac{1}{2}$. Thus, as $R_2=\log \frac S\delta$ and $r^\ast(s)$ defined as the $r$ with smallest $\sum_{i=1}^m \Gap_r^i(s)$ in \Cref{eq:definition of r*}, we have $\sum_{i=1}^m \Gap_{r^\ast(s)}^i(s)\le 2\sum_{i=1}^m \E[\Gap^i(s)]$ with probability $1-\frac \delta S$.
Taking a union bound over all $s\in \mS_h$ gives our conclusion that $\sum_{i=1}^m \Gap_r^i(s)\le 2\sum_{i=1}^m \E[\Gap^i(s)]$ for all $s\in \mS_h$ \textit{w.p.} $1-\delta$.
\end{proof}

\section{Analysis of the Improved \textsc{CCE-Approx} Subroutine}
\label{sec:appendix CCE-Approx}

\begin{theorem}[\textsc{Gap} is \textit{w.h.p.} Pessimistic; Formal Version of \Cref{lem:main theorem CCE-Approx}]\label{lem:main theorem CCE-Approx appendix}
Suppose that $\eta=\Omega(\max\{\frac{\sqrt \gamma}{H},\frac \gamma {\beta_1+\beta_2}\})$, ${\beta_1}=\tilde \Omega(\frac{dH}{\sqrt K})$, $\beta_2=\tilde \Omega(\frac HK)$, and $\gamma=\Otil(\frac dK)$. For each execution of $(\tilde \pi,\Gap)=\textsc{CCE-Approx}_h(\bar \pi,\bar V,K)$,
\begin{equation}
\max_{\pi_\ast^i\in \Pi_h^i}\left \{\left (\E_{a\sim \tilde \pi}-\E_{a\sim \pi_\ast^i\times \tilde \pi^{-i}}\right )\left [\big (\ell^i+\P_{h+1}\bar V^i\big )(s,a)\right ]\right \} \le \Gap^i(s),\quad \forall i\in [m], s\in \mS_h,\label{eq:new gap requirement CCE-Approx}
\end{equation}
\textit{w.p.} $1-\O(\delta)$, where $\textsc{Gap}^i(s)$ is defined as (the names are in correspondence to those in \Cref{eq:reg decomposition main text})
\begin{align}
&\quad K\cdot \Gap^i(s)\triangleq \underbrace{\frac{\log \lvert \mA^i\rvert}{\eta}+2\eta \sum_{k=1}^K \sum_{a\in \mA^i} \pi_k^i(a\mid s) \hat Q_k^i(s,a)^2}_{\textsc{Reg-Term}}+ \nonumber\\
&\quad \underbrace{8\sqrt 2\sqrt{2dH^2\sum_{k=1}^K \left \lVert \E_{a\sim \pi_k^i(\cdot \mid s)}\left [{\bm \phi}(s,a)\right ] \right \rVert_{\hat \Sigma_{t,i}^\dagger}^2+\sum_{k=1}^K \left (\E_{a\sim \pi_k^i(\cdot \mid s)}[{\bm \phi}(s,a)^\trans]\hat {\bm \theta}_k^i\right )^2}\log \frac{4KH}{\gamma \delta}}_{\textsc{Bias-1}}+ \nonumber \\
&\quad \underbrace{\O\left (\frac d {\beta_1} \log \frac{dK}{\delta}\right )}_{\textsc{Bias-2}+\textsc{Bonus-2}}+\underbrace{\sum_{k=1}^K \E_{a\sim \pi_k^i(\cdot \mid s)}[B_k^i (s,a)]}_{\textsc{Bonus-1}},\quad \forall i\in [m],s\in \mS_h.\label{eq:definition of gap}
\end{align}
\end{theorem}
\begin{proof}
To make the proof easy to read, we first restate \Cref{eq:new gap requirement CCE-Approx} using the notations of \textsc{CCE-Approx}. 
Let ${\bm \theta}_k^i$ be the Q-function kernel induced by the next-layer V-function $\bar V^i$ and $\pi_k^{-i}$, \textit{i.e.},
\begin{equation}\label{eq:definition of theta}
{\bm \phi}(s,a^i)^\trans {\bm \theta}_k^i=Q_k^i(s,a)\triangleq \E_{\bm a^{-i}\sim \pi_{k}^{-i}(\cdot \mid s)}\left [(\ell^i+\P \bar V^i)(s,\bm a)\right ],\quad \forall s\in \mS_h,a^i\in \mA^i.
\end{equation}
Then $\ell_{k,h}^i+\bar V^i(s_{k,h+1}^i)$ is a sample with mean $\bm \phi(s,a^i)^\trans \bm \theta_k^i$.

Let $\pi_\ast^i$ be the best-response policy of agent $i$ when facing $\tilde \pi^{-i}$ (which is the average policy for $K$ episodes, \textit{i.e.}, $\tilde \pi^{-i}=\frac 1K\sum_{k=1}^K \pi_k^{-i}$).
We need to ensure the following with probability $1-\delta$:
\begin{equation*}
\Gap^i(s)\ge \sum_{a\in \mA^i} (\tilde \pi^i(a\mid s)-\pi_\ast^i(a\mid s)) \left ({\bm \phi}(s,a)^\trans {\bm \theta}_k^i\right ),\quad \forall i\in [m], s\in \mS_h,
\end{equation*}
while $\Gap$ is allowed to involve data-dependent quantities that are available during run-time.
By plugging in the definition of $\tilde \pi^{-i}$ and decomposing the right-handed-side, this inequality becomes
\begin{align}
K\cdot \Gap^i(s)&\ge \sum_{k=1}^K \sum_{a\in \mA^i} (\pi_k^i(a\mid s)-\pi_\ast^i(a\mid s)) \left ({\bm \phi}(s,a)^\trans {\bm \theta}_k^i\right )\nonumber\\
&=\underbrace{\sum_{k=1}^K \sum_{a\in \mA^i} (\pi_k^i(a\mid s)-\pi_\ast^i(a\mid s)) \left (\hat Q_k^i(s,a)-B_k^i(s,a)\right )}_{\mathsc{Reg-Term}}+\nonumber\\
&\quad \underbrace{\sum_{k=1}^K \sum_{a\in \mA^i} \pi_k^i(a\mid s) {\bm \phi}(s,a)^\trans ({\bm \theta}_k^i-\hat {\bm \theta}_k^i)}_{\mathsc{Bias-1}}+\underbrace{\sum_{k=1}^K \sum_{a\in \mA^i} \pi_\ast^i(a\mid s) {\bm \phi}(s,a)^\trans (\hat {\bm \theta}_k^i-{\bm \theta}_k^i)}_{\mathsc{Bias-2}}+\nonumber\\
&\quad \underbrace{\sum_{k=1}^K \sum_{a\in \mA^i} \pi_k^i(a\mid s) B_k^i(s,a)}_{\mathsc{Bonus-1}}-\underbrace{\sum_{k=1}^K \sum_{a\in \mA^i} \pi_\ast^i(a\mid s) B_k^i(s,a)}_{\mathsc{Bonus-2}}+\nonumber\\
&\quad \underbrace{\sum_{k=1}^K \sum_{a\in \mA^i}(\pi_k^i(a\mid s)-\pi_\ast^i(a\mid s))\left ({\bm \phi}(s,a)^\trans \hat {\bm \theta}_k^i - \hat Q_k^i(s,a)\right )}_{\mathsc{Mag-Reduce}},\label{eq:condition of gap appendix}
\end{align}
for all player $i\in [m]$ and state $s\in \mS_h$, with probability $1-\delta$.

In \Cref{lem:reg-term,lem:Bias-1,lem:Bias-2+Bonus,lem:bonus-1,lem:mag-reduce}, we control each term in the RHS of \Cref{eq:condition of gap appendix} and show that \Cref{eq:condition of gap appendix} indeed holds with probability $1-\O(\delta)$ when \textsc{Gap} is defined in \Cref{eq:definition of gap}.
\end{proof}

\subsection{Bounding \textsc{Reg-Term} via EXP3 Regret Guarantee}
In EXP3, one typical requirement is that the loss vector $\hat y_k$ fed into EXP3 should satisfy $\hat y_k(a)\ge -1/\eta$ (see \Cref{lem:exp3}).
To comply with the condition, people usually control the Q-estimate via $\lvert {\bm \phi}(s,a)^\trans \hat{\bm \theta}_k^i\rvert\lesssim \lVert \hat \Sigma_{t,i}^\dagger\rVert_2\le \gamma^{-1}$ \citep{luo2021policy} and set $\eta\approx \gamma$, suffering loss of order $\Otil(\gamma^{-1})$.

However, $\gamma^{-1}$ can be prohibitively large when setting $\gamma=\Otil(K^{-1})$, which \Cref{lem:Lem 14 of Liu et al} requires.
Fortunately, thanks to the Magnitude-Reduced Estimator by \citet{dai2023refined}, $\hat Q_k^i\ge \hat m_k^i$ (defined in \Cref{eq:definition of Q hat}) can be bounded from below by $-\Otil(K^{-1/2})$ and thus we can pick the standard learning rate of $\eta=\Otil(K^{-1/2})$.
The other component in $\hat y_k$ is the bonuses, which is $B_k^i(s,a)\triangleq \beta_1 \lVert \bm \phi(s,a)\rVert_{\hat \Sigma_{t,i}^\dagger}^2 + \beta_2 \sum_{j=1}^d \bm \phi(s,a)[j]\times (\sup_{(s',a')\in \mS_h\times \mA^i} \hat \Sigma_{t,i}^\dagger \bm \phi(s',a'))[j]$ for all $k$.

\begin{theorem}\label{lem:reg-term}
When $\eta=\Omega(\max\{\frac{\sqrt \gamma}{H},\frac \gamma {\beta_1+\beta_2}\})$, with probability $1-\delta$, for all $i\in [m]$ and $s\in \mS_h$:
\begin{equation*}
\textsc{Reg-Term}\le \frac{\log \lvert \mA^i\rvert}{\eta}+2\eta \sum_{k=1}^K \sum_{a\in \mA^i} \pi_k^i(a\mid s) \hat Q_k^i(s,a)^2+2\eta \sum_{k=1}^K \sum_{a\in \mA^i} \pi_k^i(a\mid s) B_k^i(s,a)^2.
\end{equation*}
\end{theorem}
\begin{proof}
The only thing we need to verify before invoking \Cref{lem:exp3} is that $\hat Q_k^i(s,a) - B_k^i(s,a)\ge -1/\eta$. We first show $\hat Q_k^i(s,a)\ge -1/\eta$.
Recall the definition of $\hat Q_k^i(s,a)$ in \Cref{eq:definition of Q hat}:
\begin{align*}
&\quad \hat Q_k^i(s,a)={\bm \phi}(s,a)^\trans \hat {\bm \theta}_k^i-H\left ({\bm \phi}(s,a)^\trans \hat \Sigma_{t,i}^\dagger {\bm \phi}(s_{k,h}^i,a_{k,h}^i)\right )_-+\hat m_k^i(s,a)\\
&={\bm \phi}(s,a)^\trans \hat \Sigma_{t,i}^\dagger ~ {\bm \phi}(s_{k,h}^i,a_{k,h}^i) ~ \left (\ell_{k,h}^i+\bar V^i(s_{k,h+1}^i)\right )-H\left ({\bm \phi}(s,a)^\trans \hat \Sigma_{t,i}^\dagger {\bm \phi}(s_{k,h}^i,a_{k,h}^i)\right )_-+\hat m_k^i(s,a).
\end{align*}

As $\lvert \ell_{k,h}^i+\bar V^i(s_{k,h+1}^i)\rvert\le H$, we know $\hat Q_k^i(s,a)\ge \hat m_k^i(s,a)$ as $x-(x)_-\ge 0$ holds for any $x\in \mathbb R$.
Thus, to lower bound $\hat Q_k^i(s,a)$, it suffices to lower bound $\hat m_k^i(s,a)$. Notice that
\begin{align*}
\left (\hat m_k^i(s,a)\right )^2
&=\left (\frac HK\sum_{\kappa=1}^K \left ({\bm \phi}(s,a)^\trans \hat \Sigma_{t,i}^\dagger {\bm \phi}(s_\kappa^\text{mag},\tilde a_{\kappa,i}^\text{mag})\right )_-\right )^2\\
&\overset{(a)}{\le} \frac{H^2}{K}\sum_{\kappa=1}^K \left ({\bm \phi}(s,a)^\trans \hat \Sigma_{t,i}^\dagger {\bm \phi}(s_\kappa^\text{mag},\tilde a_{\kappa,i}^\text{mag})\right )_-^2\\
&\overset{(b)}{\le} \frac{H^2}{K}\sum_{\kappa=1}^K \left ({\bm \phi}(s,a)^\trans \hat \Sigma_{t,i}^\dagger {\bm \phi}(s_\kappa^\text{mag},\tilde a_{\kappa,i}^\text{mag})\right )^2\\
&=H^2 {\bm \phi}(s,a)^\trans \hat \Sigma_{t,i}^\dagger \left (\frac 1K\sum_{\kappa=1}^K {\bm \phi}(s_\kappa^\text{mag},\tilde a_{\kappa,i}^\text{mag}){\bm \phi}(s_\kappa^\text{mag},\tilde a_{\kappa,i}^\text{mag})^\trans\right ) \hat \Sigma_{t,i}^\dagger{\bm \phi}(s,a).
\end{align*}
where (a) used Cauchy-Schwartz inequality and (b) used Jensen inequality and the fact that $(x)_-^2\le x^2$ for all $x\in \mathbb R$.
Let the true covariance of $\bar \pi$ of layer $h\in [H]$, and agent $i\in [m]$ be
\begin{equation}\label{eq:definition of covariance}
\Sigma_{t,i}=\E_{(s,a)\sim_h \bar \pi}\left [{\bm \phi}(s,a){\bm \phi}(s,a)^\trans \right ].
\end{equation}

The average matrix in the middle, namely
\begin{equation*}
\tilde \Sigma_{t,i}^{\text{mag}}\triangleq \frac 1K\sum_{\kappa=1}^K {\bm \phi}(s_\kappa^\text{mag},\tilde a_{\kappa,i}^\text{mag}){\bm \phi}(s_\kappa^\text{mag},\tilde a_{\kappa,i}^\text{mag})^\trans
\end{equation*}
is an empirical estimation of the true covariance $\Sigma_{t,i}$. Hence, by stochastic matrix concentration results stated as \Cref{lem:second corol of Lem A.4 Dai} of \Cref{sec:stochastic matrix concentration}, we know
\begin{equation*}
\tilde \Sigma_{t,i}^{\text{mag}}\preceq \frac 32 \Sigma_{t,i} + \frac 32 \frac{d}{K} \log \left (\frac{dK}{\delta}\right ) I \quad \text{with probability }1-\frac \delta K.
\end{equation*}

Meanwhile, the following matrix from \Cref{eq:definition of Sigma hat} (where we defined $\hat \Sigma_{t,i}^\dagger$) is yet another empirical estimation of $\Sigma_{t,i}$ that is independent to $\tilde \Sigma_{t,i}^{\text{mag}}$:
\begin{equation}\label{eq:definition of Sigma cov}
\tilde \Sigma_{t,i}^{\text{cov}}\triangleq \frac 1K\sum_{\kappa=1}^K {\bm \phi}(s_\kappa^\text{cov},\tilde a_{\kappa,i}^\text{cov}){\bm \phi}(s_\kappa^\text{cov},\tilde a_{\kappa,i}^\text{cov})^\trans.
\end{equation}

By similar arguments stated as \Cref{lem:Corol 10 of Liu et al} of \Cref{sec:stochastic matrix concentration}, we have
\begin{equation*}
\Sigma_{t,i}\preceq 2 \tilde \Sigma_{t,i}^{\text{cov}} + 3 \frac{d}{K} \log \left (\frac{dK}{\delta}\right ) I \quad \text{with probability }1-\frac \delta K.
\end{equation*}

Taking a union bound over all $k\in [K]$, the following holds for all $k$ with probability $1-2\delta$:
\begin{align*}
\left (\hat m_k^i(s,a)\right )^2&\le H^2 {\bm \phi}(s,a)^\trans \hat \Sigma_{t,i}^\dagger \tilde \Sigma_{t,i}^{\text{mag}} \hat \Sigma_{t,i}^\dagger {\bm \phi}(s,a)\\
&\le 3H^2 {\bm \phi}(s,a)^\trans \hat \Sigma_{t,i}^\dagger \left (\tilde \Sigma_{t,i}^{\text{cov}}+\gamma I\right ) \hat \Sigma_{t,i}^\dagger {\bm \phi}(s,a)\\
&= 3H^2 {\bm \phi}(s,a)^\trans \hat \Sigma_{t,i}^\dagger {\bm \phi}(s,a)\le 3 H^2 / \gamma.
\end{align*}

Therefore, we have $\hat Q_k^i(s,a)\ge - \hat m_k^i(s,a) \ge - 2 H / \sqrt \gamma$ with probability $1-\delta$, which is at least $-1/\eta$ with our choice of $\eta$.

For the bonus term $B_k^i$, we consider the two parts related to $\beta_1$ and $\beta_2$ separatedly. For the $\beta_1$-term, we have the following upper bound since $\hat \Sigma_{t,i}^\dagger=(\tilde \Sigma_{t,i}^{\text{cov}}+\gamma I)^{-1}\preceq \gamma^{-1} I$:
\begin{align*}
{\beta_1} \lVert {\bm \phi}(s,a)\rVert_{\hat \Sigma_{t,i}^\dagger}^{2}\le {\beta_1} \left \lVert \hat \Sigma_{t,i}^\dagger\right \rVert_2\le \frac{{\beta_1}}{\gamma},
\end{align*}

For the $\beta_2$-related term, notice that for any $j\in [d]$, we have
\begin{equation*}
(\bm \phi(s,a))[j]\times \sup_{(s',a')\in \mS_h\times \mA^i} (\hat \Sigma_{t,i}^\dagger \bm \phi(s',a'))[j]=\sup_{(s',a')\in \mS_h\times \mA^i} \left ((\bm \phi(s,a))^\trans \bm e_j \bm e_j^\trans \hat \Sigma_{t,i}^\dagger \bm \phi(s',a') \right ),
\end{equation*}
where $\bm e_j\in \mathbb R^d$ is the one-hot vector at the $j$-th coordinate. By Cauchy-Schwartz, this is further bounded by $\lVert (\bm \phi(s,a))^\trans \bm e_j \bm e_j^\trans\rVert_{\hat \Sigma_{t,i}^\dagger}\times \sup_{(s',a')\in \mS_h\times \mA^i} \lVert \bm \phi(s',a')\rVert_{\hat \Sigma_{t,i}^\dagger}$.
We also have $\lVert \bm \phi(s',a')\rVert_{\hat \Sigma_{t,i}^\dagger}\le \lVert \bm \phi(s',a')\rVert_2\times \sqrt{\lVert \hat \Sigma_{t,i}^\dagger\rVert_2}\le \sqrt{\gamma^{-1}}$ for any $(s',a')\in \mS_h\times \mA^i$. Thus
\begin{align}
&\quad \sum_{j=1}^d (\bm \phi(s,a))[j]\times \sup_{(s',a')\in \mS_h\times \mA^i} (\hat \Sigma_{t,i}^\dagger \bm \phi(s',a'))[j]
\le \sqrt{\gamma^{-1}} \sum_{j=1}^d \lVert (\bm \phi(s,a))^\trans \bm e_j \bm e_j^\trans\rVert_{\hat \Sigma_{t,i}^\dagger} \nonumber\\
&= \sqrt{\gamma^{-1}} \left \lVert \sum_{j=1}^d (\bm \phi(s,a))^\trans \bm e_j \bm e_j^\trans\right \rVert_{\hat \Sigma_{t,i}^\dagger}=\sqrt{\gamma^{-1}} \lVert \bm \phi(s,a)\rVert_{\hat \Sigma_{t,i}^\dagger}. \label{eq:sum of sup in bonus}
\end{align}

So the $\beta_2$-related term is controlled by $\beta_2 \gamma^{-1}$, which means $B_k^i(s,a)\le (\beta_1+\beta_2) \gamma^{-1}$.

Thus the condition that $\hat y_k(a)\ge -1/\eta$ holds once $\eta^{-1} \le \sqrt{\frac{3H^2}{\gamma}}+\frac{\beta_1+\beta_2} \gamma$, \textit{i.e.}, $\eta=\Omega(\max\{\frac{\sqrt \gamma}{H},\frac \gamma {\beta_1+\beta_2}\})$. Applying the EXP3 regret bound (\Cref{lem:exp3}) gives
\begin{align*}
\mathsc{Reg-Term}&\le \frac{\log \lvert \mA^i\rvert}{\eta}+\eta \sum_{k=1}^K \sum_{a\in \mA^i} \pi_k^i(a\mid s) \left (\hat Q_k^i(s,a)^2-B_k^i(s,a)\right )^2\\
&\le \frac{\log \lvert \mA^i\rvert}{\eta}+2\eta \sum_{k=1}^K \sum_{a\in \mA^i} \pi_k^i(a\mid s) \hat Q_k^i(s,a)^2+2 \eta \sum_{k=1}^K \sum_{a\in \mA^i} \pi_k^i(a\mid s) B_k^i(s,a)^2\\
&\le \frac{\log \lvert \mA^i\rvert}{\eta}+2\eta \sum_{k=1}^K \sum_{a\in \mA^i} \pi_k^i(a\mid s) \hat Q_k^i(s,a)^2+2 \sum_{k=1}^K \sum_{a\in \mA^i} \pi_k^i(a\mid s) B_k^i(s,a),
\end{align*}
where the last step uses $\eta B_k^i(s,a)\le 1$. All these terms are available during run-time, so the algorithm can include them into $\Gap_t^i(s)$.
\end{proof}

\subsection{Bounding \textsc{Bias-1} via Adaptive Freedman Inequality}\label{sec:Bias-1}
\begin{theorem}\label{lem:Bias-1}
With probability $1-2\delta$, we have the following for all $i\in [m]$ and $s\in \mS_h$:
\begin{align*}
\textsc{Bias-1}&\le \sum_{k=1}^K \frac {\beta_1} 4 \left \lVert \E_{a\sim \pi_k^i(\cdot \mid s)}\left [{\bm \phi}(s,a)^\trans\right ]\right \rVert_{\hat \Sigma_{t,i}^\dagger}^2+\O\left (\frac d {\beta_1} \log \frac{dK}{\delta}\right )+\\
&\quad 8\sqrt 2\sqrt{2dH^2\sum_{k=1}^K \left \lVert \E_{a\sim \pi_k^i(\cdot \mid s)}\left [{\bm \phi}(s,a)\right ] \right \rVert_{\hat \Sigma_{t,i}^\dagger}^2+\sum_{k=1}^K \left (\E_{a\sim \pi_k^i(\cdot \mid s)}[{\bm \phi}(s,a)^\trans]\hat {\bm \theta}_k^i\right )^2}\log \frac{4KH}{\gamma \delta}.
\end{align*}
\end{theorem}
\begin{proof}
Let $\{X_k\}_{k=1}^K$ be a sequence of random variables adapted to filtration $(\mathcal F_k)_{k=0}^K$ where
\begin{equation*}
X_k=\E_{a\sim \pi_k^i(\cdot \mid s)}\left [{\bm \phi}(s,a)^\trans\right ]\hat {\bm \theta}_k^i,\quad \forall 1\le k\le K.
\end{equation*}

Let $\mu_k=\E[X_k\mid \mathcal F_{k-1}]$ be the conditional expectations. Then $\{X_k-\mu_k\}_{k=1}^K$ forms a martingale difference sequence.
We divide \textsc{Bias-1} into two parts, one for the intrinsic bias of $\hat {\bm \theta}_k^i$ (how $\E[\hat {\bm \theta}_k^i]$ differs from ${\bm \theta}_k^i$) and the other for the estimation error (how $\hat {\bm \theta}_k^i$ differs from $\E[\hat {\bm \theta}_k^i]$). Namely,
\begin{align}
&\quad \mathsc{Bias-1}
=\sum_{k=1}^K \E_{a\sim \pi_k^i(\cdot \mid s)} \left [{\bm \phi}(s,a)^\trans\right ] ({\bm \theta}_k^i-\hat {\bm \theta}_k^i) \nonumber\\
&=\sum_{k=1}^K \left (\E_{a\sim \pi_k^i(\cdot \mid s)} \left [{\bm \phi}(s,a)^\trans\right ] {\bm \theta}_k^i-\mu_k\right )+\sum_{k=1}^K (\mu_k-X_k). \label{eq:Bias-1 decomposition}
\end{align}

The first term is a standard term appearing in regret-minimization analyses of single-agent RL. In \Cref{lem:intrinsic bias}, we control it in analog to Lemma D.2 of \citet{luo2021policy}, but invoking the new covariance estimation analyses by \citet{liu2023bypassing} (which we restated in \Cref{sec:stochastic matrix concentration}).

The second term is the main obstacle stopping people from obtaining high-probability regret bounds for adversarial contextual linear bandits.
While we are also unable to provide a deterministic high-probability upper bound, thanks to our Improved \texttt{AVLPR} framework (see the discussions after \Cref{lem:new main theorem}), \textit{data-dependent} high-probability bounds are allowed. This is yielded in \Cref{lem:martingale} by developing a variant of the Adaptive Freedman Inequality proposed by \citet{lee2020bias} and improved by \citet{zimmert2022return} (the variant can be found in \Cref{sec:more adaptive Freedman}).
\end{proof}

\subsubsection{Controlling Intrinsic Bias}
\begin{lemma}\label{lem:intrinsic bias}
With probability $1-\delta$, for any $i\in [m]$ and $s\in \mS_h$, we have
\begin{equation*}
\sum_{k=1}^K \left (\E_{a\sim \pi_k^i(\cdot \mid s)} \left [{\bm \phi}(s,a)^\trans\right ] {\bm \theta}_k^i-\mu_k\right )= \sum_{k=1}^K \frac {\beta_1} 4 \left \lVert \E_{a\sim \pi_k^i(\cdot \mid s)}\left [{\bm \phi}(s,a)^\trans\right ]\right \rVert_{\hat \Sigma_{t,i}^\dagger}^2+\O\left (\frac d {\beta_1} \log \frac{dK}{\delta}\right ).
\end{equation*}
\end{lemma}
\begin{proof}
The conditional expectation $\mu_k$ can be directly calculated as
\begin{align*}
\mu_k=\E_{\textsc{Gap}}\left [X_k\mid \mathcal F_{k-1}\right ]&=\E_{a\sim \pi_k^i(\cdot \mid s)}\left [{\bm \phi}(s,a)^\trans\right ]\E_{\textsc{Gap}}\left [\hat \Sigma_{t,i}^\dagger~{\bm \phi}(s_{k,h}^i,a_{k,h}^i) ~ \left (\ell_{k,h}^i+\bar V^i(s_{k,h+1}^i)\right )\right ]\\
&\overset{(a)}{=}\E_{a\sim \pi_k^i(\cdot \mid s)}\left [{\bm \phi}(s,a)^\trans\right ] \hat \Sigma_{t,i}^\dagger \E_{\textsc{Gap}}\left [{\bm \phi}(s_{k,h}^i,a_{k,h}^i) {\bm \phi}(s_{k,h}^i,a_{k,h}^i)^\trans {\bm \theta}_k^i\right ]\\
&\overset{(b)}{=}\E_{a\sim \pi_k^i(\cdot \mid s)}\left [{\bm \phi}(s,a)^\trans\right ] \hat \Sigma_{t,i}^\dagger \Sigma_{t,i} {\bm \theta}_k^i,
\end{align*}
where (a) uses the independence between $\hat \Sigma_{t,i}^\dagger$ and the trajectory $(s_{k,h}^i,a_{k,h}^i)$, and (b) uses the definition of $\Sigma_{t,i}$ in \Cref{eq:definition of covariance}.

To handle the first term of \Cref{eq:Bias-1 decomposition}, we use Cauchy-Schwartz inequality, triangle inequality, and AM-GM inequality (the calculation follows Lemma D.2 of \citet{luo2021policy}).
\begin{align*}
\E_{a\sim \pi_k^i(\cdot \mid s)}\left [{\bm \phi}(s,a)^\trans\right ] {\bm \theta}_k^i-\mu_k
&=\E_{a\sim \pi_k^i(\cdot \mid s)}\left [{\bm \phi}(s,a)^\trans\right ] (I-\hat \Sigma_{t,i}^\dagger \Sigma_{t,i}) {\bm \theta}_k^i\\
&=\E_{a\sim \pi_k^i(\cdot \mid s)}\left [{\bm \phi}(s,a)^\trans\right ] \hat \Sigma_{t,i}^\dagger (\gamma I+\tilde \Sigma_{t,i}^{\text{cov}}-\Sigma_{t,i}) {\bm \theta}_k^i\\
&\le \left \lVert \E_{a\sim \pi_k^i(\cdot \mid s)}\left [{\bm \phi}(s,a)^\trans\right ]\right \rVert_{\hat \Sigma_{t,i}^\dagger} \times \left \lVert(\gamma I + \tilde \Sigma_{t,i}^{\text{cov}}-\Sigma_{t,i}) {\bm \theta}_k^i\right \rVert_{\hat \Sigma_{t,i}^\dagger}\\
&\le \frac {\beta_1} 4 \left \lVert \E_{a\sim \pi_k^i(\cdot \mid s)}\left [{\bm \phi}(s,a)^\trans\right ]\right \rVert_{\hat \Sigma_{t,i}^\dagger}^2 + \frac 2{\beta_1} \left \lVert (\gamma I + \tilde \Sigma_{t,i}^{\text{cov}}-\Sigma_{t,i}) {\bm \theta}_k^i\right \rVert_{\hat \Sigma_{t,i}^\dagger}^2,
\end{align*}
where $\tilde \Sigma_{t,i}^{\text{cov}}$ is defined in \Cref{eq:definition of Sigma cov} such that $\hat \Sigma_{t,i}^\dagger=(\tilde \Sigma_{t,i}^{\text{cov}}+\gamma I)^{-1}$.
The first term directly goes to $\Gap$ as it is available during run-time.
The second term can be controlled by the following inequality \citep[Lemma 14]{liu2023bypassing} which we include as \Cref{lem:Lem 14 of Liu et al}:
\begin{equation}\label{eq:Lem 14 of Liu et al}
\frac 2{\beta_1} \left \lVert (\gamma I + \tilde \Sigma_{t,i}^{\text{cov}}-\Sigma_{t,i}) {\bm \theta}_k^i\right \rVert_{\hat \Sigma_{t,i}^\dagger}^2
= \O\left (\frac 1 {\beta_1} \frac{d}{K} \log \frac{dK}{\delta}\right ),\quad \text{with probability }1-\frac \delta K,
\end{equation}
where we plugged in the definition of $\gamma$ that $\gamma=\frac{5d}{K}\log \frac{6d}{\delta}$.
Conditioning on the good events in \Cref{eq:Lem 14 of Liu et al} and taking a union bound over all $k\in [K]$, our conclusion follows.
\end{proof}

\subsubsection{Controlling Estimation Error}\label{sec:Bias-1 martingale}
\begin{lemma}\label{lem:martingale}
With probability $1-2\delta$, for all $i\in [m]$ and $s\in \mS_h$, we have
\begin{equation*}
\left \lvert \sum_{k=1}^K(X_k-\mu_k)\right \rvert\le 8\sqrt 2\sqrt{2dH^2\sum_{k=1}^K \left \lVert \E_{a\sim \pi_k^i(\cdot \mid s)}\left [{\bm \phi}(s,a)\right ] \right \rVert_{\hat \Sigma_{t,i}^\dagger}^2+\sum_{k=1}^K \left (\E_{a\sim \pi_k^i(\cdot \mid s)}[{\bm \phi}(s,a)^\trans]\hat {\bm \theta}_k^i\right )^2}\log \frac{4KH}{\gamma \delta}.
\end{equation*}
\end{lemma}
\begin{proof}
From the Adaptive Freedman Inequality (\Cref{lem:more adaptive Freedman} in \Cref{sec:more adaptive Freedman}), we have
\begin{equation}\label{eq:more adaptive Freedman}
\left \lvert \sum_{i=1}^n (X_i-\mu_i)\right \rvert\le 4\sqrt 2\sqrt{\sum_{i=1}^n \E[X_i^2\mid \mathcal F_{i-1}]+\sum_{i=1}^n X_i^2} \log \frac{C}{\delta},\quad \text{with probability }1-2\delta,
\end{equation}
where $C=2\sqrt 2\sqrt{\sum_{i=1}^n \E[X_i^2\mid \mathcal F_{i-1}]+\sum_{i=1}^n X_i^2}$. By definition of $X_i=\E_{a\sim \pi_k^i(\cdot \mid s)}[\bm \phi(s,a)^\trans]\hat{\bm \theta}_k^i$,
\begin{equation*}
X_i^2\le \lVert \bm \phi(s,a)\rVert_2^2 \times \lVert \hat{\bm \theta}_k^i\rVert_2^2\le \lVert \hat \Sigma_{t,i}^\dagger\rVert_2^2\times \lVert \bm \phi(s_{k,h}^i,a_{k,h}^i)\rVert_2^2\times \left \lvert \ell_{k,h}^i+\bar V^i (s_{k,h+1}^i)\right \rvert^2\le \gamma^{-2} H^2,
\end{equation*}
where we used $\hat \Sigma_{t,i}^\dagger \preceq \gamma^{-1}I$. Hence, $C\le 4 K H \gamma^{-1}$.

As $X_k^2$ is available during run-time, it only remains to control $\E[X_k^2\mid \mathcal F_{k-1}]$ to make \Cref{eq:more adaptive Freedman} calculable. By definition of $X_k=\E_{a\sim \pi_k^i(\cdot \mid s)}\left [{\bm \phi}(s,a)^\trans\right ] \hat {\bm \theta}_k^i$, we have
\begin{align*}
&\quad \E_{\textsc{Gap}}\left [X_k^2 \mid \mathcal F_{k-1}\right]=\E_{\textsc{Gap}}\left [\left (\E_{a\sim \pi_k^i(\cdot \mid s)}\left [{\bm \phi}(s,a)^\trans\right ] \hat {\bm \theta}_k^i \right )^2\right ]\\
&=\E_{a\sim \pi_k^i(\cdot \mid s)}\left [{\bm \phi}(s,a)^\trans\right ] \E_{\textsc{Gap}}\left [\hat {\bm \theta}_k^i (\hat {\bm \theta}_k^i)^\trans \right ] \E_{a\sim \pi_k^i(\cdot \mid s)}\left [{\bm \phi}(s,a)\right ].
\end{align*}

We focus on the expectation in the middle, \textit{i.e.}, $\E_{\textsc{Gap}}\left [\hat {\bm \theta}_k^i (\hat {\bm \theta}_k^i)^\trans \right ]$. Plugging in the definition of $\hat {\bm \theta}_k^i$:
\begin{align*}
\E_{\textsc{Gap}}\left [\hat {\bm \theta}_k^i (\hat {\bm \theta}_k^i)^\trans \right ]=\hat \Sigma_{t,i}^\dagger \E_{\textsc{Gap}}\left [ {\bm \phi}(s_{k,h}^i,a_{k,h}^i) {\bm \theta}_k^i ({\bm \theta}_k^i)^\trans {\bm \phi}(s_{k,h}^i,a_{k,h}^i)^\trans \right ]\hat \Sigma_{t,i}^\dagger\preceq dH^2 \hat \Sigma_{t,i}^\dagger \Sigma_{t,i} \hat \Sigma_{t,i}^\dagger.
\end{align*}

From \Cref{lem:Corol 10 of Liu et al}, $\Sigma_{t,i}\preceq 2 (\tilde \Sigma_{t,i}^{\text{cov}}+\gamma I)$ \textit{w.p.} $1-\frac \delta K$ when $\gamma \ge \frac{3d}{2K} \log \left (\frac{dK}\delta\right ) I$. Hence,
\begin{equation*}
\E_{\textsc{Gap}}\left [\hat {\bm \theta}_k^i (\hat {\bm \theta}_k^i)^\trans \right ]\preceq dH^2 \hat \Sigma_{t,i}^\dagger \Sigma_{t,i} \hat \Sigma_{t,i}^\dagger \preceq 2dH^2 \hat \Sigma_{t,i}^\dagger\quad \text{with probability }1-\frac \delta K.
\end{equation*}

Putting this into $\E_{\textsc{Gap}}\left [X_k^2 \mid \mathcal F_{k-1}\right]=\E_{a\sim \pi_k^i(\cdot \mid s)}\left [{\bm \phi}(s,a)^\trans\right ] \E_{\textsc{Gap}}\left [\hat {\bm \theta}_k^i (\hat {\bm \theta}_k^i)^\trans \right ] \E_{a\sim \pi_k^i(\cdot \mid s)}\left [{\bm \phi}(s,a)\right ]$ gives
\begin{equation}\label{eq:conditional variance}
\E_{\textsc{Gap}}\left [X_k^2 \mid \mathcal F_{k-1}\right]\le 2dH^2 \left \lVert \E_{a\sim \pi_k^i(\cdot \mid s)}\left [{\bm \phi}(s,a)\right ] \right \rVert_{\hat \Sigma_{t,i}^\dagger}^2.
\end{equation}

Our conclusion follows by combining \Cref{eq:conditional variance,eq:more adaptive Freedman}.
\end{proof}

\subsection{Cancelling \textsc{Bias-2} Using \textsc{Bonus-2}}
\textsc{Bias-2} looks pretty similar to \textsc{Bias-1}, except that we now have $\E_{a\sim \pi^i_{\color{blue} \ast}(\cdot \mid s)}\left [{\bm \phi}(s,a)^\trans \hat {\bm \theta}_k^i\right ]$ instead of $\E_{a\sim \pi^i_{\color{blue} k}(\cdot \mid s)}\left [{\bm \phi}(s,a)^\trans \hat {\bm \theta}_k^i\right ]$.
This subtle difference actually forbids us from handling \textsc{Bias-2} analogue to \textsc{Bias-1} as $\pi_\ast^i$ is unknown to the agent.
As we sketched in the main text, we also adopt the classical idea of using bonuses to cancel biases. However, as the maximum among $\E_{a\sim \pi_\ast^i(\cdot \mid s)}[\bm \phi(s,a)^\trans \hat {\bm \theta}_k^i]$ can be as large as $\lVert \hat{\Sigma}_{k,i}^\dagger\rVert_2\le \gamma^{-1}\approx \O(K)$, it can no longer be neglected like previous papers.

As mentioned in the main text, we use a state-action-wise bonus to cancel the maximum martingale difference term induced by Adaptive Freedman Inequality.
As \Cref{eq:condition of gap appendix} is linear in $\pi_\ast^i$, we only need to consider the $\pi_\ast^i(\cdot \mid s)$'s that are one-hot on some action $a_\ast^i\in \mA^i$. For notional simplicity, we abbreviate $\bm \phi_\ast^i=\bm \phi(s,a_\ast^i)$ when $s$ is clear from the context.

\begin{theorem}\label{lem:Bias-2+Bonus}
When ${\beta_1}=\tilde{\Omega}\left (\frac{dH}{\sqrt K}\right )$ and $\beta_2=\tilde \Omega(\frac HK)$, \textit{w.p.} $1-2\delta$, for all $i\in [m]$ and $s\in \mS_h$,
\begin{align*}
\textsc{Bias-2}+\textsc{Bonus-2}&=\O\left (\frac d {\beta_1} \log \frac{dK}{\delta}\right ).
\end{align*}
\end{theorem}
\begin{proof}
Imitating the analysis in \Cref{sec:Bias-1} but applying the original Adaptive Freedman Inequality (\Cref{lem:adaptive Freedman}) instead of our \Cref{lem:more adaptive Freedman} gives \Cref{lem:Bias-2}, \textit{i.e.},
\begin{equation}\label{eq:bias-2 decomposition}
\mathsc{Bias-2}\le \frac {\beta_1} 4 \lVert {\bm \phi}_\ast^i\rVert_{\hat \Sigma_{t,i}^\dagger}^2 + \O\left (\frac d {\beta_1} \log \frac{dK}{\delta}\right ) + 3\sqrt{2dH^2 \lVert {\bm \phi}_\ast^i\rVert_{K \hat \Sigma_{t,i}^\dagger}^2} \log \frac{4KH}{\gamma \delta} + 2\max_{k\in [K]} (\bm \phi_\ast^i)^\trans \hat {\bm \theta}_k^i \log \frac{4KH}{\gamma\delta}.
\end{equation}

By definition of \textsc{Bonus-2}, we have
\begin{align}
\mathsc{Bonus-2}&=\sum_{a\in \mA^i} \pi_\ast^i(a\mid s) \left (\sum_{k=1}^K {\beta_1} \lVert {\bm \phi}(s,a)\rVert_{\hat \Sigma_{t,i}^\dagger}^{2}+\beta_2 \sup_{(s',a')\in \mS_h\times \mA^i} \bm \phi(s,a)^\trans \hat \Sigma_{t,i}^\dagger \bm \phi(s',a')\right ) \nonumber\\
&={\beta_1} \lVert \bm \phi_\ast^i \rVert_{K \hat \Sigma_{t,i}^\dagger}^2 + K \beta_2 \sum_{j=1}^d \bm \phi_\ast^i[j] \times \sup_{(s',a')\in \mS_h\times \mA^i} (\hat \Sigma_{t,i}^\dagger \bm \phi(s',a'))[j]. \label{eq:negative bonus}
\end{align}

So we only need to control \Cref{eq:bias-2 decomposition} using $-(\text{\Cref{eq:negative bonus}})$.
The first term in \Cref{eq:bias-2 decomposition} is already contained in \Cref{eq:negative bonus}, while the second term is a constant.
For the third term, we would like to control it using the remaining $\frac 34 {\beta_1} \lVert \bm \phi_\ast^i\rVert_{K\hat \Sigma_{t,i}^\dagger}^2$, \textit{i.e.}, we show
\begin{equation*}
4\sqrt 2 \sqrt{2dH^2 \lVert {\bm \phi}_\ast^i\rVert_{K\hat \Sigma_{t,i}^\dagger}^2}\log \frac{4KH}{\gamma \delta}\le \frac 34 {\beta_1} \lVert \bm \phi_\ast^i\rVert_{K\hat \Sigma_{t,i}^\dagger}^2.
\end{equation*}
In other words, we would like to control $\frac{1024}{9}dH^2 \log^2 \frac{4H}{\gamma \delta}\lVert {\bm \phi}_\ast^i\rVert_{K\hat \Sigma_{t,i}^\dagger}^2$ by ${\beta_1}^2 \lVert {\bm \phi}_\ast^i\rVert_{K\hat \Sigma_{t,i}^\dagger}^4$. Equivalently,
\begin{equation*}
\frac{1024}{9}dH^2 \log^2 \frac{4H}{\gamma \delta}{\beta_1}^{-2}\le \lVert {\bm \phi}_\ast^i\rVert_{\hat \Sigma_{t,i}^\dagger}^2=K({\bm \phi}_\ast^i)^\trans \hat \Sigma_{t,i}^\dagger {\bm \phi}_\ast^i.
\end{equation*}

As $\hat \Sigma_{t,i}^\dagger=(\tilde \Sigma_{t,i}^{\text{cov}}+\gamma I)^{-1}\succeq (1+\gamma)^{-1} I$, $\lVert {\bm \phi}_\ast^i\rVert_{K\hat \Sigma_{t,i}^\dagger}^2\ge \frac{K}{1+\gamma}\lVert {\bm \phi}_\ast^i\rVert_2^2\ge\frac{K}{1+\gamma}\frac{1}{\sqrt d}$ (recall our assumption that $\lVert \bm \phi\rVert_2\ge \frac{1}{\sqrt d}$).
As $\gamma\le 1$, this inequality is ensured so long as
\begin{equation*}
\frac{1024}{9}dH^2 \log^2 \frac{4H}{\gamma \delta}\times 2\sqrt d\times {\beta_1}^{-2}\le K\Longleftarrow {\beta_1}\ge \frac{64 dH\log \frac{4KH}{\gamma \delta}}{3\sqrt K}=\tilde \Omega\left (\frac{dH}{\sqrt K}\right ).
\end{equation*}

For the last term, by definition of $\hat{\bm \theta}_k^i$, it's covered by the second part in \Cref{eq:negative bonus} once $K\beta_2\ge 2H\log \frac{4KH}{\gamma\delta}$ as $(\bm \phi_\ast^i)^\trans \hat {\bm \theta}_k^i\le H (\bm \phi_\ast^i)^\trans \hat \Sigma_{t,i}^\dagger \bm \phi(s_{k,h},a_{k,h}^i)$ where $(s_{k,h},a_{k,h}^i)\in \mS_h\times \mA^i$, which means $\sum_{j=1}^d \bm \phi_\ast^i[j] \times \sup_{(s',a')\in \mS_h\times \mA^i} (\hat \Sigma_{t,i}^\dagger \bm \phi(s',a'))[j]$ in \Cref{eq:negative bonus} covers $(\bm \phi_\ast^i)^\trans \hat \Sigma_{t,i}^\dagger \bm \phi(s_{k,h},a_{k,h}^i)$. Hence our conclusion follows given that $\beta_1=\tilde{\Omega}\left (\frac{dH}{\sqrt K}\right )$ and $\beta_2=\tilde{\Omega}(\frac HK)$.
\end{proof}

\begin{lemma}\label{lem:Bias-2}
With probability $1-2\delta$, for all $i\in [m]$ and $s\in \mS_h$, we have
\begin{align*}
\mathsc{Bias-2}&\le \frac {\beta_1} 4 \sum_{k=1}^K \lVert {\bm \phi}_\ast^i \rVert_{\hat \Sigma_{t,i}^\dagger}^2 + \O\left (\frac d {\beta_1} \log \frac{dK}{\delta}\right )+\\
&\quad 3\sqrt{2dH^2\sum_{k=1}^K \lVert {\bm \phi}_\ast^i\rVert_{\hat \Sigma_{t,i}^\dagger}^2}\log \frac{4KH}{\gamma\delta}+2\max_{k\in [K]} (\bm \phi_\ast^i)^\trans \hat {\bm \theta}_k^i \log \frac{4KH}{\gamma\delta}.
\end{align*}
\end{lemma}
\begin{proof}
Imitating \Cref{sec:Bias-1}, we also decompose $\textsc{Bias-2}$ into intrinsic bias and estimation error. Let $X_k=(\bm \phi_\ast^i)^\trans \hat{\bm \theta}_k^i$ and $\mu_k=\E[X_k\mid \mathcal F_{k-1}]$, we have
\begin{align*}
\textsc{Bias-2}=\sum_{k=1}^K (\bm \phi_\ast^i)^\trans (\hat{\bm \theta}_k^i - {\bm \theta}_k^i)=\sum_{k=1}^K \left (\mu_k - (\bm \phi_\ast^i)^\trans {\bm \theta}_k^i\right ) + \sum_{k=1}^K (X_k-\mu_k).
\end{align*}

The first term is the same as \Cref{lem:intrinsic bias}: Concluding that $\mu_k=({\bm \phi}_\ast^i)^\trans \hat \Sigma_{t,i}^\dagger \Sigma_{t,i} {\bm \theta_k^i}$ and then applying Cauchy-Schwartz inequality, triangle inequality, and AM-GM inequality gives
\begin{equation*}
\sum_{k=1}^K \left (\mu_k - (\bm \phi_\ast^i)^\trans {\bm \theta}_k^i\right )\le \frac {\beta_1} 4 \sum_{k=1}^K \lVert {\bm \phi}_\ast^i \rVert_{\hat \Sigma_{t,i}^\dagger}^2 + \O\left (\frac d {\beta_1} \log \frac{dK}{\delta}\right ),\quad \text{with probability }1-\delta.
\end{equation*}

For the second term, the proof is also similar to \Cref{lem:martingale}. The only difference is that instead of our \Cref{lem:more adaptive Freedman}, we now apply the original Adaptive Freedman Inequality (in \Cref{lem:adaptive Freedman}). We get
\begin{equation*}
\sum_{k=1}^K (X_k-\mu_k)\le 3\sqrt{\sum_{k=1}^K \E[X_k^2\mid \mathcal F_{k-1}]}\log \frac C\delta + 2\max_{k\in [K]} X_k \log \frac C \delta,\quad \text{with probability }1-\delta,
\end{equation*}
where $C=2\max \{1,\sqrt{\sum_{k=1}^K \E[X_k^2\mid \mathcal F_{k-1}]},\max_{k\in [K]}X_k\}$. Following the calculations in \Cref{lem:martingale}, $C$ is bounded by $4H \gamma^{-1}$ and $\E[X_k^2\mid \mathcal F_{k-1}]\le 2dH^2\lVert {\bm \phi}_\ast^i\rVert_{\hat \Sigma_{t,i}^\dagger}^2$. The maximum part is directly contained in our conclusion by noticing that $X_k=(\bm \phi_\ast^i)^\trans \hat{\bm \theta}_k^i$.
\end{proof}

\subsection{Putting \textsc{Bonus-1} into \textsc{Gap} Directly}
The two components in the \textsc{Bonus-1} term, namely $\pi_k^i$ and $B_k^i$, are both known during run-time. So we trivially have the following theorem:
\begin{theorem}\label{lem:bonus-1}
For all $i\in [m]$ and $s\in \mS_h$, we have
\begin{equation*}
\textsc{Bonus-1}\le \sum_{k=1}^K \sum_{a\in \mA^i} \pi_k^i(a\mid s) B_k^i(s,a).
\end{equation*}
\end{theorem}
\begin{proof}
This is the definition of $\textsc{Bonus-1}$.
\end{proof}

\subsection{Bounding \textsc{Mag-Reduce} via Martingale Properties}
\begin{theorem}\label{lem:mag-reduce}
\textsc{Mag-Reduce} is bounded by the sum of RHS of \Cref{lem:martingale,lem:Bias-2} \textit{w.p.} $1-\Otil(\delta)$.
\end{theorem}
\begin{proof}
By definition of $\hat Q_k^i(s,a)$ in \Cref{eq:definition of Q hat}, we have
\begin{align*}
&\quad \textsc{Mag-Reduce}=\sum_{k=1}^K \sum_{a\in \mA^i}(\pi_k^i(a\mid s)-\pi_\ast^i(a\mid s))\bigg ({\bm \phi}(s,a)^\trans \hat {\bm \theta}_k^i - \hat Q_k^i(s,a)\bigg )\\
&=\sum_{k=1}^K \sum_{a\in \mA^i}(\pi_k^i(a\mid s)-\pi_\ast^i(a\mid s))\bigg (H\left ({\bm \phi}(s,a)^\trans \hat \Sigma_{t,i}^\dagger {\bm \phi}(s_{k,h},a_{k,h}^i)\right )_--\frac HK \sum_{\kappa=1}^K \left ({\bm \phi}(s,a)^\trans \hat \Sigma_{t,i}^\dagger {\bm \phi}(s_{\kappa}^{\text{mag}},a_{\kappa,i}^{\text{mag}})\right )_-\bigg )
\end{align*}

As $(s_{k,h},a_{k,h}^i)$ and $(s_\kappa^{\text{mag}},a_{\kappa,i}^{\text{mag}})$ are both sampled from $\bar \pi$, all these $(\cdot)_-$'s are common mean. Thus, by telescoping, we can decompose \textsc{Mag-Reduce} into the sum of $K+1$ martingales.

It suffices to consider only one of them, for example,
\begin{align}
\sum_{k=1}^K \sum_{a\in \mA^i}(\pi_k^i(a\mid s)-\pi_\ast^i(a\mid s))H\left (\left ({\bm \phi}(s,a)^\trans \hat \Sigma_{t,i}^\dagger {\bm \phi}(s_{k,h},a_{k,h}^i)\right )_--\E_{(s',a')\sim_h^i \bar \pi}\bigg [\left ({\bm \phi}(s,a)^\trans \hat \Sigma_{t,i}^\dagger {\bm \phi}(s',a')\right )_-\bigg ]\right ), \label{eq:martingales in Mag-Reduce}
\end{align}
where $(s,a)\sim_h^i \pi$ means the $h$-th layer state and the $h$-th layer $i$-th agent action sampled from $\pi$.

Again, there are two components in \Cref{eq:martingales in Mag-Reduce}, one related to $\pi_k^i$ and the other related to $\pi_\ast^i$. Fortunately, they can be handled pretty similarly to what we did in \Cref{lem:Bias-1} and \Cref{lem:Bias-2+Bonus}: For the $\pi_k^i$ part, applying \Cref{lem:more adaptive Freedman} as in \Cref{sec:Bias-1 martingale}, we have
\begin{align*}
&\quad \sum_{k=1}^K \sum_{a\in \mA^i}\pi_k^i(a\mid s) H\left (\left ({\bm \phi}(s,a)^\trans \hat \Sigma_{t,i}^\dagger {\bm \phi}(s_{k,h},a_{k,h}^i)\right )_--\E_{(s',a')\sim_h^i \bar \pi}\bigg [\left ({\bm \phi}(s,a)^\trans \hat \Sigma_{t,i}^\dagger {\bm \phi}(s',a')\right )_-\bigg ]\right )\\
&\le \Otil\left (H\sqrt{\sum_{k=1}^K \E_{\textsc{Gap}}\left [\left (\E_{a\sim \pi_k^i(\cdot \mid s)}\left [\left ({\bm \phi}(s,a)^\trans \hat \Sigma_{t,i}^\dagger {\bm \phi}(s_{k,h}^i,a_{k,h}^i)\right )_-\right ]\right )^2\right ]}+\right .\\
&\quad \qquad \left .H\sqrt{\sum_{k=1}^K \left (\E_{a\sim \pi_k^i(\cdot \mid s)}\left [\left ({\bm \phi}(s,a)^\trans \hat \Sigma_{t,i}^\dagger {\bm \phi}(s_{k,h}^i,a_{k,h}^i)\right )_-\right ]\right )^2}\right ).
\end{align*}
Note that $\E[(\cdot)_-]^2\le \E[(\cdot)_-^2]\le \E[(\cdot)^2]$, it becomes identical to the conclusion of \Cref{lem:martingale}. Thus, this component only causes a $\Otil(1)$ contribution to the final \textsc{Gap} and can be neglected.

Similarly, for $\pi_\ast^i$, we apply \Cref{lem:adaptive Freedman} like we did in \Cref{lem:Bias-2}. We have
\begin{align*}
&\quad -\sum_{k=1}^K \sum_{a\in \mA^i} \pi_\ast^i(a\mid s)H\left (\left ({\bm \phi}(s,a)^\trans \hat \Sigma_{t,i}^\dagger {\bm \phi}(s_{k,h},a_{k,h}^i)\right )_--\E_{(s',a')\sim_h^i \bar \pi}\bigg [\left ({\bm \phi}(s,a)^\trans \hat \Sigma_{t,i}^\dagger {\bm \phi}(s',a')\right )_-\bigg ]\right )\\
&\le \Otil\left (H\sqrt{\sum_{k=1}^K \E_{\textsc{Gap}}\left [\left (\E_{a\sim \pi_\ast^i(\cdot \mid s)}\left [\left ({\bm \phi}(s,a)^\trans \hat \Sigma_{t,i}^\dagger {\bm \phi}(s_{k,h}^i,a_{k,h}^i)\right )_-\right ]\right )^2\right ]}+\right . \\
&\quad \qquad \left .H \max_{k\in [K]} \left (-\E_{a\sim \pi_\ast^i(\cdot \mid s)}\left [\left ({\bm \phi}(s,a)^\trans \hat \Sigma_{t,i}^\dagger {\bm \phi}(s_{k,h}^i,a_{k,h}^i)\right )_-\right ]\right )\right ).
\end{align*}
Again, we have $\E[(\cdot)_-]^2\le \E[(\cdot)_-^2]\le \E[(\cdot)^2]$ and also $\lvert -\E[(\cdot)_-]\rvert\le \E[(\cdot)]$. Thus, this part also produces the same result as \Cref{lem:Bias-2}.
\end{proof}

\section{Controlling the Expectation of \textsc{Gap} Using Potentials}\label{sec:expectation of Gap}
In this section, we verify \Cref{eq:new expectation of Gap appendix}, \textit{i.e.}, prove \Cref{lem:main theorem CCE-Approx second half main text}.
\begin{theorem}[\Cref{alg:linear case} Allows a Potential Function; Formal Version of \Cref{lem:main theorem CCE-Approx second half main text}]\label{lem:main theorem CCE-Approx second half}
With probability $1-\delta$, under the conditions of \Cref{lem:main theorem CCE-Approx appendix}, it is possible to give a tuning of \Cref{alg:linear case} such that
\begin{equation*}
\sum_{t=1}^T \sum_{i=1}^m \E_{s\sim_h \tilde \pi_t} \left [\E_{\textsc{Gap}}[\Gap_t^i(s)]\right ]=\Otil(m d^2 H \sqrt T),\quad \text{with probability }1-\Otil(\delta).
\end{equation*}

In other words, \Cref{eq:new expectation of Gap appendix} is ensured by picking $L=d^4 H^2$.
\end{theorem}
\begin{proof}
The proof is divided into two parts.
In \Cref{lem:expectation of Gap}, we calculate $\E_{s\sim_h \tilde \pi}[\E_{\Gap}[\Gap_t^i(s)]]$ for any roll-in policy $\tilde \pi$ (although we only use $\tilde \pi_t$ as the roll-in policy, it is unknown at the point when $\textsc{CCE-Approx}_h$ is executed, because we iterated $h=H,H-1,\ldots,1$).
Then, in \Cref{lem:summation of Gap}, we control their summation using the definition of potential functions in \Cref{eq:definition of potential} and also borrowing techniques from \citep{zanette2022stabilizing,cui2023breaking}.
\end{proof}

\subsection{Calculating the Expectation of \textsc{Gap} \textit{w.r.t.} Any $\tilde \pi$}
\begin{theorem}\label{lem:expectation of Gap}
Consider a single agent $i\in [m]$, epoch $t\in [T]$, and layer $h\in [H]$. For any outcome of $\textsc{CCE-Approx}_h$ with $K$ set as $t$. Then for any ``roll-in'' policy $\tilde \pi$ (which is chosen as $\tilde \pi_t$ in \Cref{eq:new expectation of Gap appendix}), under the conditions in \Cref{lem:main theorem CCE-Approx appendix}, \textit{i.e.}, setting $\eta=\Omega(\max\{\frac{\sqrt \gamma}{H},\frac \gamma {\beta_1+\beta_2}\})$, ${\beta_1}=\tilde \Omega(\frac{dH}{\sqrt K})$, $\beta_2=\tilde \Omega(\frac HK)$, and $\gamma=\Otil(\frac dK)$ for the execution of \textsc{CCE-Approx} with $K\gets t$, we have
\begin{align*}
&\quad \E_{s\sim_h \tilde \pi} \left [\E_{\Gap}[t\times \Gap_t^i(s)]\right ]\\
&=\Otil \Bigg (\eta^{-1}+\eta H^2 t \E_{s\sim_h \tilde \pi}\left [\E_{a\sim \tilde \pi_t^i(\cdot \mid s)}\left [\lVert \bm \phi(s,a)\rVert_{\hat \Sigma_{t,i}^\dagger}^2\right ]\right ]+\sqrt d H \sqrt{t\E_{s\sim_h \tilde \pi}\left [\E_{a\sim \tilde \pi_t^i(\cdot \mid s)}\left [\lVert \bm \phi(s,a)\rVert_{\hat \Sigma_{t,i}^\dagger}^2\right ]\right ]}\\
&\qquad \qquad \frac{d}{\beta_1} + \beta_1 t \E_{s\sim_h \tilde \pi}\left [\E_{a\sim \tilde \pi_t^i(\cdot \mid s)}\left [\lVert \bm \phi(s,a)\rVert_{\hat \Sigma_{t,i}^\dagger}^2\right ]\right ]+\beta_2 \sqrt{\gamma^{-1} t \E_{s\sim_h \tilde \pi}\left [\E_{a\sim \tilde \pi_t^i(\cdot \mid a)}\left [\lVert \bm \phi(s,a)\rVert_{\hat \Sigma_{t,i}^\dagger}^2\right ]\right ]} \Bigg ).
\end{align*}
Note that, although in \Cref{alg:framework} we mixed up $\Gap$'s from different $r$'s, all $\Gap$'s are \textit{i.i.d.} samples from the same distribution and thus $\E[\Gap^i(s)]$ does not depend on the choice of $r^\ast(s)$ in \Cref{eq:definition of r*}.
\end{theorem}
\begin{proof}
Recall the definition of $\Gap_t^i(s)$ from \Cref{eq:definition of gap}, we have the following decomposition:
\begin{align}
&\quad \E_{s\sim_h \tilde \pi}\left [\E_{\textsc{Gap}}\left [K\times \Gap_t^i(s)\right ]\right ]= 
\underbrace{\frac{\log \lvert \mA^i\rvert}{\eta}+2\eta \E_{s\sim_h \tilde \pi}\left [\E_{\textsc{Gap}}\left [\sum_{k=1}^K \sum_{a\in \mA^i} \pi_k^i(a\mid s) \hat Q_k^i(s,a)^2\right ]\right ]}_{\textsc{Reg-Term}}+ \nonumber\\
&\quad \underbrace{8\sqrt 2\E_{s\sim_h \tilde \pi}\left [\E_{\textsc{Gap}}\left [\sqrt{2dH^2\sum_{k=1}^K \left \lVert \E_{a\sim \pi_k^i(\cdot \mid s)}\left [{\bm \phi}(s,a)\right ] \right \rVert_{\hat \Sigma_{t,i}^\dagger}^2+\sum_{k=1}^K \left (\E_{a\sim \pi_k^i(\cdot \mid s)}[{\bm \phi}(s,a)^\trans]\hat {\bm \theta}_k^i\right )^2}\log \frac{4KH}{\gamma \delta}\right ]\right ]}_{\textsc{Bias-1}}+ \nonumber \\
&\quad \underbrace{\O\left (\frac d {\beta_1} \log \frac{dK}{\delta}\right )}_{\textsc{Bias-2}+\textsc{Bonus-2}}+ \nonumber\\
&\quad \underbrace{\E_{s\sim_h \tilde \pi}\left [\E_{\textsc{Gap}}\left [\sum_{k=1}^K \E_{a\sim \pi_k^i(\cdot \mid s)}\left [{\beta_1} \lVert \bm \phi(s,a)\rVert_{\hat \Sigma_{t,i}^\dagger}^2 + \beta_2 \sum_{j=1}^d \bm \phi(s,a)[j] \times \sup_{(s',a')\in \mS_h\times \mA^i} (\hat \Sigma_{t,i}^\dagger \bm \phi(s',a'))[j]\right ]\right ]\right ]}_{\textsc{Bonus-1}}.\label{eq:Gap expectation}
\end{align}

We then go by these terms one by one in \Cref{lem:Reg-Term expectation,lem:Bias-1 Expectation,lem:Bias-2 Expectation,lem:Bonus-1 Expectation} and give the conclusion.
\end{proof}

\begin{lemma}\label{lem:Reg-Term expectation}
Consider the \textsc{Reg-Term} part in \Cref{eq:Gap expectation}, we have
\begin{equation*}
\frac{\log \lvert \mA^i\rvert}{\eta}+2\eta \E_{s\sim_h \tilde \pi}\left [\E_{\textsc{Gap}}\left [\sum_{k=1}^K \sum_{a\in \mA^i} \pi_k^i(a\mid s) \hat Q_k^i(s,a)^2\right ]\right ]\le \Otil\left (\eta^{-1}+\eta H^2 t \E_{(s,a)\sim_h^i \tilde \pi}\left [\lVert \bm \phi(s,a)\rVert_{\hat \Sigma_{t,i}^\dagger}^2\right ]\right ),
\end{equation*}
where $\sim_h^i$ stands for the state-action pair that agent $i$ observes in layer $h$.
\end{lemma}
\begin{proof}
The first component is already a constant and only contributes $\Otil(\eta^{-1})$ to the final bound. For the second component, we invoke the property of the Magnitude-Reduced Estimator $\hat Q_k^i(s,a)$ by \citet{dai2023refined} (which we summarize as \Cref{lem:magnitude reduced}) and conclude $\E[\hat Q_k^i(s,a)^2]=\O(\E[(\bm \phi(s,a) \hat{\bm \theta}_k^i)^2])$, where the expectation is only taken \textit{w.r.t.} the randomness in \Cref{eq:definition of Q hat}. Thus
\begin{align*}
&\quad 2\eta \E_{s\sim_h \tilde \pi}\left [\E_{\textsc{Gap}}\left [\sum_{k=1}^K \sum_{a\in \mA^i} \pi_k^i(a\mid s) \hat Q_k^i(s,a)^2\right ]\right ]
=2\eta \E_{s\sim_h \tilde \pi}\left [\E_{\textsc{Gap}}\left [\sum_{k=1}^K \sum_{a\in \mA^i} \pi_k^i(a\mid s) \left ({\bm \phi}(s,a)^\trans \hat {\bm \theta}_k^i\right )^2\right ]\right ]\\
&\le 2\eta H^2 \sum_{k=1}^K \E_{s\sim_h \tilde \pi}\left [\E_{a \sim \pi_k^i(\cdot \mid s)} \left [ \E_{\textsc{Gap}}\left [{\bm \phi}(s_{k,h}^i,a_{k,h}^i)^\trans \hat \Sigma_{t,i}^\dagger {\bm \phi}(s,a) {\bm \phi}(s,a)^\trans \hat \Sigma_{t,i}^\dagger {\bm \phi}(s_{k,h}^i,a_{k,h}^i) \right ]\right ]\right ]\\
&\overset{(a)}{=}2\eta H^2 \sum_{k=1}^K \E_{s\sim_h \tilde \pi}\left [\E_{a \sim \pi_k^i(\cdot \mid s)} \left [ \langle \hat \Sigma_{t,i}^\dagger \Sigma_{t,i} \hat \Sigma_{t,i}^\dagger, {\bm \phi}(s,a) {\bm \phi}(s,a)^\trans\rangle \right ]\right ]\\
&\overset{(b)}{\le} 2\eta H^2 K \E_{s\sim_h \tilde \pi}\left [\E_{a\sim \tilde \pi_t^i(\cdot \mid s)}\left [\lVert \bm \phi(s,a)\rVert_{\hat \Sigma_{t,i}^\dagger}^2\right ]\right ],
\end{align*}
where (a) uses the fact that $(s_{k,h}^i,a_{k,h}^i)$ are sampled from $\bar \pi_t$ (recall the definition of $\hat \Sigma_{t,i}^\dagger$ in \Cref{eq:definition of Sigma hat}), and (b) uses \Cref{lem:Corol 10 of Liu et al} from \citet{liu2023bypassing} (which gives $\Sigma_{t,i}\preceq (\hat \Sigma_{t,i}^\dagger)^{-1}$) together with the fact that $\tilde \pi_t\triangleq \frac 1K \sum_{k=1}^K \pi_k$ (for those states in $\mS_h$). Recall the configuration that $K=t$ in \textsc{CCE-Approx}, the above quantity is $\Otil\left (\eta H^2 t \E_{s\sim_h \tilde \pi}\left [\E_{a\sim \tilde \pi_t^i(\cdot \mid s)}\left [\lVert \bm \phi(s,a)\rVert_{\hat \Sigma_{t,i}^\dagger}^2\right ]\right ]\right )$.
\end{proof}

\begin{lemma}\label{lem:Bias-1 Expectation}
Consider the \textsc{Bias-1} part in \Cref{eq:Gap expectation}, we have
\begin{align*}
&\quad 8\sqrt 2\E_{s\sim_h \tilde \pi}\left [\E_{\textsc{Gap}}\left [\sqrt{2dH^2\sum_{k=1}^K \left \lVert \E_{a\sim \pi_k^i(\cdot \mid s)}\left [{\bm \phi}(s,a)\right ] \right \rVert_{\hat \Sigma_{t,i}^\dagger}^2+\sum_{k=1}^K \left (\E_{a\sim \pi_k^i(\cdot \mid s)}[{\bm \phi}(s,a)^\trans]\hat {\bm \theta}_k^i\right )^2}\log \frac{4KH}{\gamma \delta}\right ]\right ]\\
&=\Otil\left (\sqrt d H \sqrt{t\E_{s\sim_h \tilde \pi}\left [\E_{a\sim \tilde \pi_t^i(\cdot \mid s)}\left [\lVert \bm \phi(s,a)\rVert_{\hat \Sigma_{t,i}^\dagger}^2\right ]\right ]}\right ).
\end{align*}
\end{lemma}
\begin{proof}
As $\sqrt{X+Y}\le \sqrt X+\sqrt Y$, we can write (ignoring constants and logarithmic factors)
\begin{align*}
\textsc{Bias-1}=\Otil\Bigg (&\sqrt dH \E_{s\sim_h \tilde \pi}\left [\E_{\textsc{Gap}}\left [\sqrt{\sum_{k=1}^K \left \lVert \E_{a\sim \pi_k^i(\cdot \mid s)}[\bm \phi(s,a)] \right \rVert_{\hat \Sigma_{t,i}^\dagger}^2}\right ]\right ]+\\
&\E_{s\sim_h \tilde \pi}\left [\E_{\textsc{Gap}}\left [ \sqrt{\sum_{k=1}^K \left ( \E_{a\sim \pi_k^i(\cdot \mid s)}[\bm \phi(s,a)]^\trans \hat{\bm \theta}_k^i \right )^2}\right ]\right ]\Bigg ).
\end{align*}

According to the calculations in \Cref{lem:martingale} on $\E[X_k^2\mid \mathcal F_{k-1}]$, the second term is exactly bounded by the first one. Utilizing the fact that $\E[\sqrt X]\le \sqrt{\E[X]}$, we have
\begin{align*}
\textsc{Bias-1}&\le \Otil\left (\sqrt dH\sqrt{\E_{s\sim_h \tilde \pi}\left [\E_{\textsc{Gap}}\left [\sum_{k=1}^K \left \lVert \E_{a\sim \pi_k^i(\cdot \mid s)}[\bm \phi(s,a)] \right \rVert_{\hat \Sigma_{t,i}^\dagger}^2\right ]\right ]}\right )\\
&=\Otil\left (\sqrt dH\sqrt{\sum_{k=1}^K \left \langle \E_{s\sim_h \tilde \pi}\left [ \E_{a\sim \pi_k^i(\cdot \mid s)}[\bm \phi(s,a)\bm \phi(s,a)^\trans]\right ],\hat \Sigma_{t,i}^\dagger\right \rangle}\right )\\
&=\Otil\left (\sqrt dH\sqrt{t \E_{s\sim_h \tilde \pi}\left [\E_{a\sim \tilde \pi_t^i(\cdot \mid s)}\left [\lVert \bm \phi(s,a)\rVert_{\hat \Sigma_{t,i}^\dagger}^2\right ]\right ]}\right ),
\end{align*}
where the last line again uses the configuration $K=t$ and the definition of $\tilde \pi_t$.
\end{proof}

\begin{lemma}\label{lem:Bias-2 Expectation}
The \textsc{Bias-2} + \textsc{Bonus-2} part in \Cref{eq:Gap expectation} is of order $\Otil(d{\beta_1}^{-1})$.
\end{lemma}
\begin{proof}
This part is already a constant in \Cref{eq:new expectation of Gap appendix}
\end{proof}

\begin{lemma}\label{lem:Bonus-1 Expectation}
The \textsc{Bonus-1} term in \Cref{eq:Gap expectation} is bounded by
\begin{align*}
&\quad \E_{s\sim_h \tilde \pi}\left [\E_{\textsc{Gap}}\left [\sum_{k=1}^K \E_{a\sim \pi_k^i(\cdot \mid s)}\left [{\beta_1} \lVert \bm \phi(s,a)\rVert_{\hat \Sigma_{t,i}^\dagger}^2 + \beta_2 \sum_{j=1}^d \bm \phi(s,a)[j] \times \sup_{(s',a')\in \mS_h\times \mA^i} (\hat \Sigma_{t,i}^\dagger \bm \phi(s',a'))[j]\right ]\right ]\right ]\\
&\le \beta_1 t \E_{s\sim_h \tilde \pi}\left [\E_{a\sim \tilde \pi_t^i(\cdot \mid s)}\left [\lVert \bm \phi(s,a)\rVert_{\hat \Sigma_{t,i}^\dagger}^2\right ]\right ]+\beta_2 \sqrt{\gamma^{-1} t \E_{s\sim_h \tilde \pi}\left [\E_{a\sim \tilde \pi_t^i(\cdot \mid a)}\left [\lVert \bm \phi(s,a)\rVert_{\hat \Sigma_{t,i}^\dagger}^2\right ]\right ]}.
\end{align*}
\end{lemma}
\begin{proof}
For the $\beta_1$-part, we have
\begin{align*}
&\quad \E_{s\sim_h \tilde \pi}\left [\E_{\textsc{Gap}}\left [\sum_{k=1} \E_{a\sim \pi_k^i(\cdot \mid s)}\left [{\beta_1} \lVert \bm \phi(s,a)\rVert_{\hat \Sigma_{t,i}^\dagger}^2\right ]\right ]\right ] \\
&={\beta_1} \E_{\textsc{Gap}}\left [\sum_{k=1}^K \left \langle \E_{s\sim_h \tilde \pi}\left [\E_{a\sim \pi_k^i(\cdot \mid s)}\left [\bm \phi(s,a) \bm \phi(s,a)^\trans\right ],\hat \Sigma_{t,i}^\dagger\right ]\right \rangle\right ] \\
&=\beta_1 t \E_{s\sim_h \tilde \pi}\left [\E_{a\sim \tilde \pi_t^i(\cdot \mid s)}\left [\lVert \bm \phi(s,a)\rVert_{\hat \Sigma_{t,i}^\dagger}^2\right ]\right ],
\end{align*}
which directly becomes the first part in the conclusion (where, again, we used the configuration that $K=t$ and also the definition of $\tilde \pi_t$).

For the $\beta_2$-part, from the calculations in \Cref{eq:sum of sup in bonus}, we know
\begin{equation*}
\sum_{j=1}^d (\bm \phi(s,a))[j]\times \sup_{(s',a')\in \mS_h\times \mA^i} (\hat \Sigma_{t,i}^\dagger \bm \phi(s',a'))[j]\le \sqrt{\gamma^{-1}} \lVert \bm \phi(s,a)\rVert_{\hat \Sigma_{t,i}^\dagger}.
\end{equation*}

Utilizing Cauchy-Schwartz inequality and the fact that $\E[\sqrt X]\le \sqrt{\E[X]}$, we can get
\begin{align*}
&\quad \E_{s\sim_h \tilde \pi}\left [\E_{\textsc{Gap}}\left [\sum_{k=1}^K \E_{a\sim \pi_k^i(\cdot \mid s)}\left [\beta_2 \sum_{j=1}^d \bm \phi(s,a)[j] \times \sup_{(s',a')\in \mS_h\times \mA^i} (\hat \Sigma_{t,i}^\dagger \bm \phi(s',a'))[j]\right ]\right ]\right ]\\
&\le \beta_2 \sqrt{\gamma^{-1}} \sqrt{K \E_{s\sim_h \tilde \pi}\left [\E_{\textsc{Gap}}\left [\sum_{k=1}^K \E_{a\sim \pi_k^i(\cdot \mid s)}\left [\lVert \bm \phi(s,a)\rVert_{\hat \Sigma_{t,i}^\dagger}^2\right ]\right ]\right ]}\\
&=\beta_2 t \sqrt{\gamma^{-1} \E_{s\sim_h \tilde \pi}\left [\E_{a\sim \tilde \pi_t^i(\cdot \mid a)}\left [\lVert \bm \phi(s,a)\rVert_{\hat \Sigma_{t,i}^\dagger}^2\right ]\right ]}.
\end{align*}

Putting two parts together gives our conclusion.
\end{proof}

\subsection{Summing Up $\E[\Gap]$'s Using Potentials}
\begin{theorem}\label{lem:summation of Gap}
Under the conditions of \Cref{lem:main theorem CCE-Approx appendix}, \textit{i.e.}, setting $\eta=\Omega(\max\{\frac{\sqrt \gamma}{H},\frac \gamma {\beta_1+\beta_2}\})$, ${\beta_1}=\tilde \Omega(\frac{dH}{\sqrt K})$, $\beta_2=\tilde \Omega(\frac HK)$, and $\gamma=\Otil(\frac dK)$ for $\textsc{CCE-Approx}$ executions with parameter $K$ (which is set to $t$ in each epoch $t$ where \Cref{line:condition} is violated), we have
\begin{equation*}
\sum_{t=1}^T \sum_{i=1}^m \E_{s\sim_h \tilde \pi_t} \left [\E_{\textsc{Gap}}[\Gap_t^i(s)]\right ]=\Otil(m d^2 H \sqrt T),\quad \text{with probability }1-\Otil(\delta).
\end{equation*}
\end{theorem}
\begin{proof}
For any $t$, let the last time \Cref{line:condition} that was violated be $I_t$. Then $\tilde \pi_t=\tilde \pi_{I_t}$ and $\hat \Sigma_{t,i}^\dagger=\hat \Sigma_{I_t,i}^\dagger$ by \Cref{line:condition}. For any $\tau$, we denote $n_\tau$ as the number of indices such that $I_t=\tau$. 
Throughout the proof, we use $\eta_t,\beta_{1,t},\beta_{2,t},\gamma_t$ to denote the $\eta,\beta_1,\beta_2,\gamma$ used by \textsc{CCE-Approx} in the $t$-th epoch, respectively. Recall the conclusion from \Cref{lem:expectation of Gap} (note that the LHS of \Cref{lem:expectation of Gap} is $\E[t\times \Gap_t^i(s)]$),
\begin{align*}
&\quad \sum_{t=1}^T \sum_{i=1}^m \E_{s\sim_h \tilde \pi_t} \left [\E_{\textsc{Gap}}[\Gap_t^i(s)]\right ]=\sum_{t=1}^T \sum_{i=1}^m \E_{s\sim_h \tilde \pi_{I_t}} \left [\E_{\textsc{Gap}}[\Gap_{I_t}^i(s)]\right ]\\
&=\sum_{t=1}^T \sum_{i=1}^m \Otil \Bigg (\frac{\eta_{I_t}^{-1}}{I_t}+\frac{d\beta_{1,I_t}^{-1}}{I_t}+\left (\eta_{I_t} H^2+\beta_{1,I_t}\right ) \E_{s\sim_h \tilde \pi_{I_t}}\left [\E_{a\sim \tilde \pi_{I_t}^i(\cdot \mid s)}\left [\lVert \bm \phi(s,a)\rVert_{\hat \Sigma_{I_t,i}^\dagger}^2\right ]\right ]+\\
&\qquad \qquad \qquad \left (H\sqrt{\frac{d}{I_t}}+\beta_{2,I_t} \sqrt{\gamma_{I_t}^{-1}}\right ) \sqrt{\E_{s\sim_h \tilde \pi_{I_t}}\left [\E_{a\sim \tilde \pi_{I_t}^i(\cdot \mid s)}\left [\lVert \bm \phi(s,a)\rVert_{\hat \Sigma_{I_t,i}^\dagger}^2\right ]\right ]} \Bigg ).
\end{align*}

We focus on a single agent $i\in [m]$.
For notational simplicity, we denote $(s,a)\sim_h^i \pi$ as the state-action pair that agent $i\in [m]$ observes in layer $h\in [H]$.
The first two terms are bounded by $\Otil(\sum_{t=1}^T (\eta_t^{-1} /t + d \beta_{1,t}^{-1} / t))$.
For the third term, we replace the coefficients with a sup:
\begin{align*}
&\quad \sum_{t=1}^T \Otil\left (\eta_{I_t} H^2 + \beta_{1,I_t} \right ) \E_{s\sim_h \tilde \pi_{I_t}}\left [\E_{a\sim \tilde \pi_{I_t}^i(\cdot \mid s)}\left [\lVert \bm \phi(s,a)\rVert_{\hat \Sigma_{{I_t},i}^\dagger}^2\right ]\right ]\\
&\le \sup_{t\in [T]} \left (t\eta_t H^2 + t \beta_{1,t}\right ) \Otil\left (\sum_{t=1}^T n_t \E_{(s,a)\sim_h^i \tilde \pi_t} \left [\frac 1t \lVert \bm \phi(s,a)\rVert_{\hat \Sigma_{t,i}^\dagger}^2\right ]\right ),
\end{align*}
where, for simplicity, we define $n_t$ as zero for those $t$ where \Cref{line:condition} isn't violated. Similarly, for the last term, we replace the coefficients with a sup and then apply Cauchy-Schwartz. We get
\begin{align*}
&\quad \sum_{t=1}^T \Otil\left (H\sqrt{\frac{d}{I_t}}+\beta_{2,I_t} \sqrt{\gamma_{I_t}^{-1}}\right ) \sqrt{\E_{s\sim_h \tilde \pi_{I_t}}\left [\E_{a\sim \tilde \pi_{I_t}^i(\cdot \mid s)}\left [\lVert \bm \phi(s,a)\rVert_{\hat \Sigma_{I_t,i}^\dagger}^2\right ]\right ]}\\
&\le \sup_{t\in [T]} \left (\sqrt dH + \beta_{2,t} \sqrt{\gamma_t^{-1} t}\right )\sum_{t=1}^T \sqrt{\E_{s\sim_h \tilde \pi_{I_t}}\left [\E_{a\sim \tilde \pi_{I_t}^i(\cdot \mid s)}\left [\lVert \bm \phi(s,a)\rVert_{\hat \Sigma_{I_t,i}^\dagger}^2\right ]\right ]}\\
&\le \sup_{t\in [T]} \left (\sqrt dH + \beta_{2,t} \sqrt{\gamma_t^{-1} t}\right ) \sqrt T \sqrt{\sum_{t=1}^T n_t \E_{(s,a)\sim_h^i \tilde \pi_t} \left [\frac 1t \lVert \bm \phi(s,a)\rVert_{\hat \Sigma_{t,i}^\dagger}^2\right ]}.
\end{align*}

Thus the only thing we need to do before concluding the proof is to show that
\begin{equation}\label{eq:matrix version summation lemma}
\sum_{t=1}^T n_t \E_{(s,a)\sim_h^i \tilde \pi_t} \left [\frac 1t \lVert \bm \phi(s,a)\rVert_{\hat \Sigma_{t,i}^\dagger}^2\right ]\text{ is small}.
\end{equation}

Indeed, this quantity is of order $\Otil(d)$ for all $i\in [m]$ and $h\in [H]$ \textit{w.p.} $1-\Otil(\delta)$: In \Cref{lem:potential sum}, we imitate Lemma 10 of \citet{cui2023breaking} and conclude that \Cref{eq:matrix version summation lemma} is of order $\Otil(d)$ for any fixed $i\in [m]$ and $h\in [H]$ \textit{w.p.} $1-\delta/(2mH)$; taking a union bound then gives the aforementioned fact.

To summarize, we have
\begin{align*}
&\quad \sum_{t=1}^T \sum_{i=1}^m \E_{s\sim_h \tilde \pi_t} \left [\E_{\Gap}[\Gap_t^i(s)]\right ]\\
&=\Otil\left (\sum_{t=1}^T \left (\frac{\eta_t^{-1}}{t}+d\frac{\beta_{1,t}^{-1}}{t}\right )+\sup_{t\in [T]}\left (t\eta_t H^2 + t\beta_{1,t}\right )d + \sup_{t\in [T]}\left (\sqrt dH + \beta_{2,t}\sqrt{\gamma_t^{-1} t}\right )\sqrt{dT}\right ),
\end{align*}
and we need to ensure that $\eta_t=\Omega(\max\{\frac{\sqrt \gamma_t}{H},\frac{\gamma_t}{\beta_{1,t}+\beta_{2,t}}\})$, ${\beta_{1,t}}=\tilde \Omega(\frac{dH}{\sqrt t})$, $\beta_{2,t}=\tilde \Omega(\frac Ht)$, and $\gamma_t=\Otil(\frac dt)$. Setting $\eta_t=\tilde{\Theta}(\frac{\sqrt d}{\sqrt t H})$, ${\beta_{1,t}}=\tilde \Omega(\frac{dH}{\sqrt t})$, $\beta_{2,t}=\tilde \Omega(\frac Ht)$, and $\gamma_t=\Otil(\frac dt)$ gives
\begin{align*}
&\quad \sum_{t=1}^T \sum_{i=1}^m \E_{s\sim_h \tilde \pi_t} \left [\E_{\Gap}[\Gap_t^i(s)]\right ]\\
&\le m\times \Otil\left (\sum_{t=1}^T \left (\frac{H}{\sqrt{dt}}+\frac{1}{H\sqrt t}\right )+(\sqrt d H \sqrt T + d H \sqrt T) d + (\sqrt dH + H/\sqrt d)\sqrt{dT}\right )\\
&=\Otil(m d^2 H \sqrt T),
\end{align*}
as claimed.
\end{proof}

\begin{lemma}\label{lem:potential sum}
For any agent $i\in [m]$ and layer $h\in [H]$, \textit{w.p.} $1-\delta/(2mH)$, we have
\begin{equation*}
\sum_{t=1}^T n_t \E_{(s,a)\sim_h^i \tilde \pi_t} \left [\frac 1t \lVert \bm \phi(s,a)\rVert_{\hat \Sigma_{t,i}^\dagger}^2\right ]=\Otil(d).
\end{equation*}
\end{lemma}
\begin{proof}
For a $t$ where \Cref{line:condition} is violated, recall the definition of $\hat \Sigma_{t,i}^\dagger$ in \Cref{eq:definition of Sigma hat} and the definition of $\Sigma_{t,i}^{\text{cov}}$ in \Cref{eq:definition of Sigma cov}, we know $\hat \Sigma_{t,i}^\dagger=(\Sigma_{t,i}^{\text{cov}}+\gamma I)^{-1}$. Also recall the choice of $\gamma=\tilde{\Theta}(\frac dt)$ in \Cref{alg:linear case} and \Cref{lem:Corol 10 of Liu et al} by \citet{liu2023bypassing} that says $\Sigma_{t,i}\preceq 2\tilde \Sigma_{t,i}^{\text{cov}}+\tilde{\Theta}(\frac dt) I$, we have
\begin{equation*}
\frac 1t \hat \Sigma_{t,i}^\dagger = \left (t \Sigma_{t,i}^{\text{cov}} + t \gamma I\right )^{-1} \preceq \left (t \Sigma_{t,i} + \tilde{\Theta}(d) I\right )^{-1}=\left (\sum_{s=0}^{t-1} \E_{(s,a)\sim_h^i \tilde \pi_s} [\bm \phi(s,a) \bm \phi(s,a)^\trans] + \tilde \Theta(d)I\right )^{-1}.
\end{equation*}

Let $\text{Cov}_h^i(\pi)$ be the true covariance of $\pi$, \textit{i.e.}, $\E_{(s,a)\sim_h^i \pi}[\bm \phi(s,a) \bm \phi(s,a)^\trans]$. \Cref{eq:matrix version summation lemma} becomes
\begin{equation*}
\text{\Cref{eq:matrix version summation lemma}}\le \sum_{t=1}^T \left \langle n_t \text{Cov}_h^i (\tilde \pi_t),\left (\sum_{s=0}^{t-1} \text{Cov}_h^i (\tilde \pi_s) + \tilde{\Theta}(d)I\right )^{-1}\right \rangle.
\end{equation*}

Note that $\sum_{s=0}^{t-1}\text{Cov}_h^i (\tilde \pi_s)=\sum_{s=0}^{t-1} n_s \text{Cov}_h^i (\tilde \pi_s)$ for any $t$ where \Cref{line:condition} is violated, thus \Cref{eq:matrix version summation lemma} can be viewed as a matrix version of $\sum_{t=I_t} (X_t / (\sum_{s=0}^{t-1} X_s))$ where $X_t=n_t \text{Cov}_h^i (\tilde \pi_t)$. We use the following lemma by \citet[Lemma 11]{zanette2022stabilizing}:
\begin{lemma}[Lemma 11 by \citet{zanette2022stabilizing}]\label{lem:Lemma 11 of Zanette et al}
For a random vector $\bm \phi \in \mathbb R^d$, scalar $\alpha>0$, and PSD matrix $\Sigma$, suppose that $\alpha \E[\lVert \bm \phi\rVert_{\Sigma^{-1}}^2]\le L$ where $L\ge e-1$, then
\begin{equation*}
\alpha \E[\lVert \bm \phi\rVert_{\Sigma^{-1}}^2] \le L \log \frac{\det(\Sigma+\alpha \E[\bm \phi \bm \phi^\trans])}{\det(\Sigma)}\le \alpha L \E[\lVert \bm \phi\rVert_{\Sigma^{-1}}^2].
\end{equation*}
\end{lemma}

In our case, we apply \Cref{lem:Lemma 11 of Zanette et al} to all $t=I_t$ with $\alpha=n_t$, $\Sigma=\sum_{s=0}^{t-1}\text{Cov}_h^i (\tilde \pi_s)+\tilde{\Theta}(d)I$, and the distribution as $\bm \phi(s,a)$ with $(s,a)\sim_h^i \tilde \pi_t$. We first calculate $\alpha \E[\lVert \bm \phi\rVert_{\Sigma^{-1}}^2]$, which is bounded by our potential construction and is in analog to Lemma 4 of \citet{cui2023breaking}:
\begin{lemma}\label{lem:relative concentration in potential}
For any $t\in [T]$ where \Cref{line:condition} is violated, we have the following \textit{w.p.} $1-\delta/(2 m H T)$:
\begin{equation*}
n_{t} \langle \Cov_h^i(\tilde \pi_{t}), \hat \Sigma_{t,i}^\dagger\rangle=n_{t} \E_{(s,a)\sim_h^i \tilde \pi_{t}} \left [\lVert \bm \phi(s,a)\rVert_{\hat \Sigma_{t,i}^\dagger}^2\right ]=\Otil(1).
\end{equation*}
\end{lemma}

The proof of \Cref{lem:relative concentration in potential} is presented shortly after.

Therefore, when defining $L=\max_{t\in [T]} n_{t} \langle \Cov_h^i(\tilde \pi_{t}), \hat \Sigma_{t,i}^\dagger\rangle$, we can conclude from \Cref{lem:Lemma 11 of Zanette et al}:
\begin{equation*}
\left \langle n_t \text{Cov}_h^i (\tilde \pi_t),\left (\sum_{s=0}^{t-1} \text{Cov}_h^i (\tilde \pi_s) + \tilde{\Theta}(d)I\right )^{-1}\right \rangle\le L \log \frac{\det(\sum_{s=0}^{t-1} \text{Cov}_h^i (\tilde \pi_s) + \tilde{\Theta}(d)I + n_t \text{Cov}_h^i (\tilde \pi_t))}{\det(\sum_{s=0}^{t-1} \text{Cov}_h^i (\tilde \pi_s) + \tilde{\Theta}(d)I)}.
\end{equation*}

By telescoping and the fact that $\sum_{s=0}^{t-1} \text{Cov}_h^i (\tilde \pi_s)=\sum_{s<t,s=I_s} n_s \text{Cov}_h^i (\tilde \pi_s)$, we have
\begin{equation*}
\sum_{t=1}^T \left \langle n_t \text{Cov}_h^i (\tilde \pi_t),\left (\sum_{s=0}^{t-1} \text{Cov}_h^i (\tilde \pi_s) + \tilde{\Theta}(d)I\right )^{-1}\right \rangle\le L \log \frac{\det(\sum_{t=0}^T \text{Cov}_h^i (\tilde \pi_t) + \tilde{\Theta}(d)I)}{\det(\tilde{\Theta}(d)I)},
\end{equation*}
which is bounded by $\Otil(d)$ as $L=\Otil(1)$ from \Cref{lem:relative concentration in potential}, $\log \det(\sum_{t=0}^T \text{Cov}_h^i (\tilde \pi_t) + \tilde{\Theta}(d)I)\le d \log \frac{\text{Tr}(\sum_{t=0}^T \text{Cov}_h^i (\tilde \pi_t) + \tilde{\Theta}(d)I)}{d}=\Otil(d\log (d+\frac{T}{d}))=\Otil(d)$, and $\log \det(\tilde{\Theta}(d)I)\ge 0$.
\end{proof}

\begin{proof}[Proof of \Cref{lem:relative concentration in potential}]
Recall the potential function definition in \Cref{eq:definition of potential}, which says $64 \log \frac{8 m H T}{\delta} \cdot \Psi_{t,h}^i=\sum_{\tau=1}^t \lVert \bm \phi(\tilde s_{\tau,h},\tilde a_{\tau,h}^i)\rVert_{\hat \Sigma_{\tau,i}^\dagger}^2$.
Suppose that the first time after $t$ that \Cref{line:condition} is violated again is in epoch $t_1$, then $\Psi_{t_1,h}^i$ is the first potential to reach $\Psi_{t,h}^i+1$, which means
\begin{equation*}
\sum_{\tau=t}^{t_1-1} \lVert \bm \phi(\tilde s_{\tau,h},\tilde a_{\tau,h}^i)\rVert_{\hat \Sigma_{\tau,i}^\dagger}^2 \le 64 \log \frac{8 m H T^2}{\delta},\quad \forall h\in [H].
\end{equation*}

Fix a single $h\in [H]$. Recall that for all epochs where \Cref{line:condition} is not violated, we directly adopt the previous $\tilde \pi_t$ and $\hat \Sigma_{t,i}^\dagger$. Hence, $\hat \Sigma_{\tau,i}^\dagger\equiv \hat \Sigma_{t_0,i}^\dagger$ for all $\tau\in [t,t_1)$, and all such $(\tilde s_{\tau,h},\tilde a_{\tau,h}^i)$'s are also drawn from the same $\tilde \pi_{t}$. From \Cref{lem:relative concentration} (which is Lemma 48 of \citet{cui2023breaking}), we know that
\begin{align*}
&\quad \langle \Cov_h^i(\tilde \pi_{t}), \hat \Sigma_{t,i}^\dagger\rangle = \E_{(s,a)\sim_h^i \tilde \pi_{t}} \left [\lVert \bm \phi(s,a)\rVert_{\hat \Sigma_{t,i}^\dagger}^2\right ]\\
&\le 2 \frac{1}{t_1-t} \sum_{\tau=t}^{t_1-1} \lVert \bm \phi(\tilde s_{\tau,h},\tilde a_{\tau,h}^i)\rVert_{\hat \Sigma_{\tau,i}^\dagger}^2 \le \frac{128}{t-t_0} \log \frac{8 m H T^2}{\delta},
\end{align*}
\textit{w.p.} $1-\delta/(2mHT)$. 
Hiding all logarithmic factors gives our conclusion.
\end{proof}

\section{Proof of the Resulting Sample Complexity (\Cref{lem:sample complexity})}
\label{sec:appendix overall}
\begin{proof}[Proof of \Cref{lem:sample complexity}]
We verify the conditions in \Cref{lem:new main theorem}.
The first condition, \textit{i.e.}, \Cref{eq:new gap requirement appendix}, is formalized as \Cref{lem:main theorem CCE-Approx appendix}, whose proof is in \Cref{sec:appendix CCE-Approx}.
The second condition, \textit{i.e.}, \Cref{eq:optimistic V-function appendix}, is postponed to \Cref{sec:V-approx}.
The third condition, \textit{i.e.}, \Cref{eq:new expectation of Gap appendix}, is justified in \Cref{lem:main theorem CCE-Approx second half}.

Then we can invoke the conclusion from \Cref{lem:new main theorem} to conclude a sample complexity of
\begin{align*}
&\quad \Otil\left (\max\{\Gamma_1,\Gamma_2\}\times m^2 H^3 L \epsilon^{-2}\times d_{\text{replay}}\log T\right )\\
&=\Otil(m\times m^2H^3\times d^4H^2 \epsilon^{-2}\times dmH)=\Otil(m^4 d^5 H^6 \epsilon^{-2}),
\end{align*}
where $\Gamma_1=\Gamma_2=m$ (see \Cref{alg:linear case,alg:V-approx}), $L=d^4 H^2$ (from \Cref{lem:main theorem CCE-Approx second half}), and $d_{\text{replay}}=dmH$ (from Section E.5 of \citet{wang2023breaking}).
\end{proof}

\subsection{\textsc{V-Approx} Procedure by \citet{wang2023breaking} and Its Guarantee}\label{sec:V-approx}
In this section, we introduce the \textsc{V-Approx} procedure by \citet{wang2023breaking}. In \Cref{alg:V-approx}, we present the algorithm for linear Markov Games (\textit{i.e.}, the \textsc{Optimistic-Regress} procedure in their Algorithm 3 is replaced by that in Section E.1).

\begin{algorithm}[!t]
\caption{$\textsc{V-Approx}_h$ Subroutine for Independent Linear Markov Games \citep{wang2023breaking}}
\label{alg:V-approx}
\begin{algorithmic}[1]
\REQUIRE{Previous-layer policy $\bar \pi$, current-layer policy $\tilde \pi$, next-layer V-function $\bar V$, sub-optimality gap upper-bound $\Gap$, epoch length $K$. Set regularization factor $\lambda=\O(\frac dK \log \frac{dK}{\delta})$.}
\FOR{$i=1,2,\ldots,m$}
\FOR{$k=1,2,\ldots,K$}
\STATE For each agent $j\ne i$, execute $\bar \pi\circ \tilde \pi$; for agent $i$, execute $\bar \pi$ (\textit{i.e.}, the same as \Cref{alg:linear case}). Record the trajectory as $\{(s_{k,\happa},a_{k,\happa}^i,\hat \ell_{k,\happa}^i)\}_{\happa=1}^H$ where $\hat \ell_{k,h}^i\triangleq \hat \ell_{k,\happa}^i+\bar V^i(s_{k,h+1}^i)$.
\STATE Calculate $\hat{\bm \theta}_h^i=\argmin_{\bm \theta} \frac 1K \sum_{k=1}^K (\langle \bm \phi(s_{k,h},a_{k,h}^i),\bm \theta\rangle - \ell_{k,h}^i)+\lambda \lVert \bm \theta\rVert_2^2$.
\STATE For each $s\in \mS_h$ and $a\in \mA^i$, set $\bar Q^i(s,a)=\min \{\langle \bm \phi(s,a),\hat{\bm \theta}_h^i\rangle+\frac 32 \textsc{Gap}^i(s),H-h+1\}$.
\STATE For each $s\in \mS_h$, set $\bar V^i(s)=\sum_{a\in \mA^i} \tilde \pi^i(a\mid s) \bar Q^i(s,a)$.
\ENDFOR
\ENDFOR
\RETURN $\{\bar V^i(\mS_h)\}_{i\in [m]}$.
\end{algorithmic}
\end{algorithm}

Similar to Section E.3 of \citet{wang2023breaking}, we have the following lemma:
\begin{lemma}
Consider using \Cref{alg:V-approx} with some roll-in policy $\bar \pi$ and episode length $K$. Let $(\tilde \pi,\Gap)=\textsc{CCE-Approx}_h(\bar \pi,\bar V,K)$. Then \Cref{alg:V-approx} with parameters $(\bar \pi,\tilde \pi,\bar V,\Gap,K)$ ensures that it's output $\bar V(\mS_h)$ satisfies the following with probability $1-\delta$:
\begin{align*}
\bar V^i(s)\in \bigg [\min\bigg \{&\E_{a\sim \tilde \pi}\left [\big (\ell^i+\P_{h+1}\bar V^i\big )(s,a)\right ]+\phantom{2}\Gap^i(s),H-h+1\bigg \},\nonumber\\
    &\E_{a\sim \tilde \pi}\left [\big (\ell^i+\P_{h+1}\bar V^i\big )(s,a)\right ]+2\Gap^i(s)\bigg ],\quad \forall i\in [m], s\in \mS_h.
\end{align*}
\end{lemma}
\begin{proof}
Consider a fixed $i\in [m]$ and $h\in [H]$. By \Cref{assumption}, there exists $\bm \theta_h^i$ such that
\begin{equation*}
\E[\hat \ell_{k,h}^i\mid s_{k,h},a_{k,h}^i]=\langle \bm \phi(s_{k,h},a_{k,h}^i),\bm \theta_h^i\rangle,\quad \forall k\in [K].
\end{equation*}

Let $\tilde \Sigma_{i,h}^{\text{reg}}=\frac 1K\sum_{k=1}^K \bm \phi(s_{k,h},a_{k,h}^i) \bm \phi(s_{k,h},a_{k,h}^i)^\trans + \lambda I$ and $\xi_{k,h}^i = \hat \ell_{k,h}^i - \E[\hat \ell_{k,h}^i \mid s_{k,h},a_{k,h}^i]$.
From Lemma 21 of \citet{wang2023breaking}, we know when $\lambda=\O(\frac dK \log \frac{dK}{\delta})$, with probability $1-\delta$,
\begin{equation}\label{lem:Lem 21 of Wang et al}
\left \lVert \sum_{k=1}^K \bm \phi(s_{k,h},a_{k,h}^i) \xi_{k,h}^i\right \rVert_{(\bar \Sigma_{i,h})^{-1}}=\Otil(\sqrt{dH^2 K}),
\end{equation}
where $\bar \Sigma_{i,h}=\E_{(s,a)\sim_h^i \bar \pi}[\bm \phi(s,a) \bm \phi(s,a)^\trans]$ (which is also the expectation of $\bm \phi(s_{k,h},a_{k,h}^i) \bm \phi(s_{k,h},a_{k,h}^i)^\trans$).
We also conclude from Lemma 22 of \citet{wang2023breaking} that $\tilde \Sigma_{i,h}^{\text{reg}}\succeq \frac 12(\bar \Sigma_{i,h}+\lambda I)$. Therefore, for any $s\in \mS_h$ and $a\in \mA^i$, we have
\begin{align*}
&\quad \left \lvert \bm \phi(s,a)^\trans \hat{\bm \theta}^i - \E[\hat \ell_{k,h}^i\mid s,a]\right \rvert\\
&=\left \lvert \bm \phi(s,a)^\trans (\tilde \Sigma_{i,h}^{\text{reg}})^{-1}\left (\frac 1K \sum_{k=1}^K \bm \phi(s_{k,h},a_{k,h}^i) \bm \phi(s_{k,h},a_{k,h}^i)^\trans - \tilde \Sigma_{i,h}^{\text{reg}}\right ) \bm \theta_{k,h}^i\right \rvert + \\
&\quad \frac 1K \left \lvert \bm \phi(s,a)^\trans (\tilde \Sigma_{i,h}^{\text{reg}})^{-1} \sum_{k=1}^K \bm \phi(s_{k,h},a_{k,h}^i) \xi_{k,h}^i\right \rvert\\
&\le \lVert \bm \phi(s,a)\rVert_{(\tilde \Sigma_{i,h}^{\text{reg}})^{-1}} \times \left (\Otil\left (\sqrt{\frac dK}\right )+\Otil\left (\sqrt{\frac{dH^2}{K}}\right)\right ),
\end{align*}
where the last step uses both \Cref{lem:Lem 14 of Liu et al} and \Cref{lem:Lem 21 of Wang et al}.
Noticing that $\hat \Sigma_{t,i}^\dagger$ is also an empirical covariance of policy $\bar \pi$, we can conclude that $\lVert \bm \phi(s,a)\rVert_{(\tilde \Sigma_{i,h}^{\text{reg}})^{-1}}=\O(\lVert \bm \phi(s,a)\rVert_{(\bar \Sigma_{i,h}+\lambda I)^{-1}})=\O(\lVert \bm \phi(s,a)\rVert_{\hat \Sigma_{t,i}^\dagger})$ from \Cref{lem:Corol 10 of Liu et al,lem:second corol of Lem A.4 Dai}. Consequently, the error in $\bar V^i(s)$ is no more than
\begin{align*}
&\quad \sum_{a\in \mA^i} \tilde \pi^i(a\mid s)\left \lvert \bm \phi(s,a)^\trans \hat{\bm \theta}^i - \E[\hat \ell_{k,h}^i\mid s,a]\right \rvert
=\Otil\left (\sqrt{\frac{dH^2}{K}} \sum_{a\in \mA^i} \tilde \pi^i(a\mid s) \lVert \bm \phi(s,a)\rVert_{\hat \Sigma_{t,i}^\dagger}\right )\\
&=\Otil\left (\sqrt{\frac{dH^2}{K}} \frac 1K \sum_{k=1}^K \sqrt{\left \lVert \E_{a\sim \pi_k^i(\cdot \mid s)}[\bm \phi(s,a)]\right \rVert_{\hat \Sigma_{t,i}^\dagger}^2}\right )
\le \Otil\left (\frac 1K 
\sqrt{dH^2 \sum_{k=1}^K \left \lVert \E_{a\sim \pi_k^i(\cdot \mid s)}[\bm \phi(s,a)]\right \rVert_{\hat \Sigma_{t,i}^\dagger}^2}\right ),
\end{align*}
where the last step used Cauchy-Schwartz.
Recall the definition of $\textsc{Gap}$ from \Cref{eq:definition of gap}, we can conclude that the total estimation error of \textsc{V-Approx} is no more than $\frac 12 \Gap_t^i(s)$. Our claim then follows from imitating the remaining arguments of \citet[Section E.3]{wang2023breaking}.
\end{proof}

\section{Auxiliary Lemmas}

\subsection{Magnitude-Reduced Estimators}\label{sec:magnitude reduced}
The following lemma characterizes the Magnitude-Reduced Estimator \citep{dai2023refined}.
\begin{lemma}[Magnitude-Reduced Estimators \citep{dai2023refined}]\label{lem:magnitude reduced}
For a random variable $Z$, its magnitude-reduced estimator $\hat Z\triangleq Z-(Z)_-+\E[(Z)_-]$ where $(Z)_-\triangleq \min\{Z,0\}$ satisfies
\begin{equation*}
\E[\hat Z]=\E[Z],\quad \E[(\hat Z)^2]\le 6\E[Z^2],\quad \hat Z\ge \E[(Z)_-].
\end{equation*}
\end{lemma}
\begin{proof}
The first conclusion follows from $\E[\hat Z]=\E[Z]-\E[(Z)_-]+\E[\E[(Z)_-]]=\E[Z]$.

For the second conclusion, by definition of $\hat Z$, $\E[(\hat Z)^2]=\E[(Z-(Z)_-+\E[(Z)_-])^2]\le 2(\E[Z^2]+\E[(Z)_-^2]+\E[(Z)_-]^2)$. As $(Z)_-^2=\min\{Z,0\}^2\le Z^2$, $\E[(Z)_-^2]\le \E[Z^2]$. By Jensen's inequality, $\E[(Z)_-]^2\le \E[(Z)_-^2]\le \E[Z^2]$. Hence, we arrive at the conclusion that $\E[(\hat Z)^2]\le 6\E[Z^2]$.

The last inequality follows from the fact that $Z-(Z)_-$ is $0$ if $Z<0$ and $Z$ if $Z\ge 0$. Therefore, $\hat Z=Z-(Z)_-+\E[(Z)_-]\ge \E[(Z)_-]$, as desired.
\end{proof}

\subsection{Stochastic Matrix Concentration}\label{sec:stochastic matrix concentration}
We then present some stochastic matrix concentration results.
\begin{lemma}[{Lemma A.4 by \citet{dai2023refined}}]\label{lem:Lem A.4 of Dai et al}
If $H_1,H_2,\ldots,H_n$ are i.i.d. $d$-dimensional PSD matrices such that for all $i$: i) $\E[H_i]=H$, ii) $H_i\preceq I$ a.s., and iii) $H\succeq \frac{1}{dn}\left (\log \frac d\delta\right ) I$, then
\begin{equation*}
-\sqrt{\frac dn\log \frac d\delta}H^{1/2}\preceq \frac 1n \sum_{i=1}^n H_i-H \preceq \sqrt{\frac dn\log \frac d\delta}H^{1/2}\quad \text{with probability }1-2\delta.
\end{equation*}
\end{lemma}

The first corollary of \Cref{lem:Lem A.4 of Dai et al} is derived by \citet[Corollary 10]{liu2023bypassing}.
\begin{corollary}\label{lem:Corol 10 of Liu et al}
If $H_1,H_2,\ldots,H_n$ are i.i.d. $d$-dimensional PSD matrices such that for all $i$: i) $\E[H_i]=H$, and ii) $H_i\preceq cI$ a.s. for some positive constant $c>0$, then
\begin{equation*}
H\preceq \frac 2 n \sum_{i=1}^n H_i+3c \cdot \frac dn \log \left (\frac d\delta\right ) I\quad \text{with probability }1-\delta.
\end{equation*}
\end{corollary}

The second corollary is the opposite direction of \Cref{lem:Corol 10 of Liu et al}.
\begin{corollary}\label{lem:second corol of Lem A.4 Dai}
For $H_1,H_2,\ldots,H_n$ i.i.d. PSD with expectation $H$ such that $H_i\preceq cI$ a.s. for all $i$,
\begin{equation*}
\frac 1n \sum_{i=1}^n H_i \preceq \frac 32H +3c \cdot \frac{d}{2n} \log \left (\frac d\delta\right ) I\quad \text{with probability }1-\delta.
\end{equation*}
\end{corollary}
\begin{proof}
The proof mostly follows from Corollary 10 of \citet{liu2023bypassing}. Using the fact that $H^{1/2}\preceq \frac k2 H + \frac{1}{2k}$ for any $k>0$, we know from \Cref{lem:Lem A.4 of Dai et al} that under the same conditions of \Cref{lem:Lem A.4 of Dai et al},
\begin{equation*}
\frac 1n \sum_{i=1}^n H_i-H\preceq \sqrt{\frac dn \log \frac d\delta} H^{1/2} \preceq \frac 12 H + \frac{d}{2n} \left (\log \frac d\delta\right ) I\quad \text{with probability }1-\delta.
\end{equation*}

In other words,
\begin{equation}\label{eq:second corol of Lem A.4 Dai}
\frac 1n \sum_{i=1}^n H_i \preceq \frac 32 H + \frac{d}{2n} \left (\log \frac d\delta\right ) I \quad \text{with probability }1-\delta.
\end{equation}

Now we show \Cref{lem:second corol of Lem A.4 Dai}. For the case where $\frac dn \log \frac d\delta \le 1$, define $\tilde H_i=\frac{1}{2c} H_i + \frac{d}{2n}\left (\log \frac d\delta\right ) I$.
Then $\tilde H_i\preceq \frac{1}{2c} c I + \frac{d}{2n}\left (\log \frac d\delta\right ) I \preceq I$.
Moreover, $\tilde H=\E[\tilde H_i]=\frac{1}{2c} H + \frac{d}{2n}\left (\log \frac d\delta\right ) I$ also ensures $\tilde H\succeq \frac{1}{dn}\left (\log \frac d\delta\right ) I$.
Hence, applying \Cref{eq:second corol of Lem A.4 Dai} to $\tilde H_1,\tilde H_2,\ldots,\tilde H_n$ gives
\begin{equation*}
\frac 1n \sum_{i=1}^n \tilde H_i\preceq \frac 32 \tilde H + \frac{d}{2n} \left (\log \frac d\delta\right ) I\quad \text{with probability }1-\delta.
\end{equation*}

By the definitions of $\tilde H_i$ and $\tilde H$, we further have the following, which shows our claim:
\begin{equation*}
\frac 1n \sum_{i=1}^n H_i\preceq \frac 32 \tilde H + 3c \cdot \frac{d}{2n} \left (\log \frac d\delta\right ) I\quad \text{with probability }1-\delta.
\end{equation*}

The case of $\frac dn \log \frac d\delta > 1$ is trivial because $H_i\preceq cI \preceq \frac 32 c \left (\frac dn \log \frac d\delta \right )I$.
\end{proof}

A key fact from the above corollaries is Lemma 14 of \citet{liu2023bypassing}, which we include below.
\begin{theorem}[Lemma 14 of \citet{liu2023bypassing}]\label{lem:Lem 14 of Liu et al}
For a $d$-dimensional distribution $\mathcal D$, let $\bm \phi_1,\bm \phi_2,\ldots,\bm \phi_n$ be i.i.d. samples from $\mathcal D$. Define $\Sigma=\E_{\bm \phi\sim \mathcal D}[\bm \phi \bm \phi^\trans]$ and $\tilde \Sigma=\frac 1n\sum_{i=1}^n \bm \phi_i \bm \phi_i^\trans$. If $\hat \Sigma^\dagger=(\tilde \Sigma+\gamma I)^{-1}$ where $\gamma=5\frac dn \log \frac{6d}{\delta}$, then with probability $1-\delta$, we have
\begin{equation*}
\left \lVert (\gamma I + \tilde \Sigma-\Sigma) \bm \theta\right \rVert_{\hat \Sigma^\dagger}^2
= \O\left (\frac{d}{n} \log \frac{d}{\delta}\right ),\quad \forall \lVert \bm \theta\rVert_2\le 1.
\end{equation*}
\end{theorem}

\subsection{Adaptive Freedman Inequality}\label{sec:more adaptive Freedman}
In this paper, we will make use of the Adaptive Freedman Inequality proposed by \citet[Theorem 2.2]{lee2020bias} and improved by \citet[Theorem 9]{zimmert2022return}, stated as follows.
\begin{lemma}[{Theorem 9 of \citet{zimmert2022return}}]\label{lem:adaptive Freedman original}
For a sequence of martingale differences $\{X_i\}_{i=1}^n$ adapted to the filtration $(\mathcal F_i)_{i=0}^n$, suppose that $\E[\lvert X_i\rvert\mid \mathcal F_{i-1}]<\infty$ a.s. Then
\begin{equation*}
\sum_{i=1}^n X_i \le 3\sqrt{\sum_{i=1}^n \E[X_i^2\mid \mathcal F_{i-1}] \log \frac C\delta}+2 \max_{i=1,2,\ldots,n} X_i \log \frac C\delta,\quad \text{with probability }1-\delta,
\end{equation*}
where
\begin{equation*}
C=2\max \left \{1, \sqrt{\sum_{i=1}^n \E[X_i^2\mid \mathcal F_{i-1}]},\max_{i=1,2,\ldots,n} X_i\right \}.
\end{equation*}
\end{lemma}

A direct corollary of \Cref{lem:adaptive Freedman original} is the following lemma:
\begin{lemma}\label{lem:adaptive Freedman}
For a sequence of random variables $\{X_i\}_{i=1}^n$ adapted to the filtration $(\mathcal F_i)_{i=0}^n$, let the conditional expectation of $X_i$ be $\mu_i\triangleq \E[X_i\mid \mathcal F_{i-1}]$. Suppose that $\E[\lvert X_i\rvert\mid \mathcal F_{i-1}]<\infty$ a.s. Then
\begin{equation*}
\sum_{i=1}^n (X_i-\mu_i) \le 3\sqrt{\sum_{i=1}^n \E[X_i^2\mid \mathcal F_{i-1}]} \log \frac C\delta+2 \max_{i=1,2,\ldots,n} X_i \log \frac C\delta,\quad \text{with probability }1-\delta,
\end{equation*}
where
\begin{equation*}
C=2\max \left \{1, \sqrt{\sum_{i=1}^n \E[X_i^2\mid \mathcal F_{i-1}]},\max_{i=1,2,\ldots,n} X_i\right \}.
\end{equation*}
\end{lemma}
\begin{proof}
We apply \Cref{lem:adaptive Freedman original} to the martingale difference sequence $\{X_i-\mu_i\}_{i=1}^n$, giving
\begin{equation*}
\sum_{i=1}^n (X_i-\mu_i)\le 3\sqrt{\sum_{i=1}^n \E[(X_i-\mu)^2\mid \mathcal F_{i-1}] \log \frac C\delta}+2 \max_{i=1,2,\ldots,n} (X_i-\mu_i) \log \frac C\delta,\quad \text{with probability }1-\delta.
\end{equation*}

It is clear that $\max (X_i-\mu_i)\le \max X_i+\sqrt{\sum_{i=1}^n \mu_i^2}$. Furthermore, we know that $\E[(X_i-\mu_i)^2\mid \mathcal F_{i-1}]=\E[X_i^2\mid \mathcal F_{i-1}]-\mu_i^2$. Putting these two parts together gives our conclusion.
\end{proof}

We also give the following variant of \Cref{lem:adaptive Freedman}:
\begin{lemma}\label{lem:more adaptive Freedman}
For a sequence of random variables $\{X_i\}_{i=1}^n$ adapted to the filtration $(\mathcal F_i)_{i=0}^n$, let the conditional expectation of $X_i$ be $\mu_i\triangleq \E[X_i\mid \mathcal F_{i-1}]$. Suppose that $\E[\lvert X_i\rvert\mid \mathcal F_{i-1}]<\infty$ a.s. Then
\begin{equation*}
\left \lvert \sum_{i=1}^n (X_i-\mu_i)\right \rvert \le 8\sqrt 2\sqrt{\sum_{i=1}^n \E[X_i^2\mid \mathcal F_{i-1}]+\sum_{i=1}^n X_i^2} \log \frac{C}{\delta},\quad \text{with probability }1-\delta,
\end{equation*}
where
\begin{equation*}
C=2\sqrt 2\sqrt{\sum_{i=1}^n \E[X_i^2\mid \mathcal F_{i-1}]+\sum_{i=1}^n X_i^2}.
\end{equation*}
\end{lemma}
\begin{proof}
Applying \Cref{lem:adaptive Freedman original} to the martingale difference sequence $\{X_i-\mu_i\}_{i=1}^n$, the following inequality holds with probability $1-\delta$:
\begin{equation*}
\sum_{i=1}^n (X_i-\mu_i)\le 3\sqrt{\sum_{i=1}^n \E[(X_i-\mu_i)^2\mid \mathcal F_{i-1}]\log \frac{C'}{\delta}}+2\max_{i=1,2,\ldots,n}\{X_i-\mu_i\}\log \frac{C'}{\delta},
\end{equation*}
where
\begin{equation*}
C'=2\max\left \{1,\sqrt{\sum_{i=1}^n \E[(X_i-\mu_i)^2\mid \mathcal F_{i-1}]},\max_{i=1,2,\ldots,n}\{X_i-\mu_i\}\right \}.
\end{equation*}

Utilizing the fact that $(X_i-\mu_i)^2\le 2(X_i^2+\mu_i^2)$ for all $i$, we have
\begin{equation*}
\max_{i=1,2,\ldots,n}\{X_i-\mu_i\}\le \sqrt{2\sum_{i=1}^n (X_i^2+\mu_i^2)}.
\end{equation*}

Hence, we can write
\begin{align*}
\sum_{i=1}^n (X_i-\mu_i)
&\le 4\sqrt 2\sqrt{\sum_{i=1}^n \E[(X_i-\mu_i)^2\mid \mathcal F_{i-1}]+\sum_{i=1}^n (X_i^2+\mu_i^2)}\log \frac{C'}{\delta}\\
&=4\sqrt 2\sqrt{\sum_{i=1}^n \left (\E[X_i^2\mid \mathcal F_{i-1}]-\mu_i^2\right )+\sum_{i=1}^n (X_i^2+\mu_i^2)}\log \frac{C'}{\delta}\\
&=4\sqrt 2\sqrt{\sum_{i=1}^n \E[X_i^2\mid \mathcal F_{i-1}]+\sum_{i=1}^n X_i^2}\log \frac{C'}{\delta}.
\end{align*}
Similarly, $C'\le 2\sqrt 2\sqrt{\sum_{i=1}^n \E[X_i^2\mid \mathcal F_{i-1}]+\sum_{i=1}^n X_i^2}=C$. 
By exactly the same arguments, the same inequality also holds for $\sum_{i=1}^n (\mu_i-X_i)$. Our conclusion then follows.
\end{proof}

\subsection{Relative Concentration Bounds}
\begin{lemma}[Lemma 48 by \citet{cui2023breaking}]\label{lem:relative concentration}
Let $X_1,X_2,\ldots$ be i.i.d. random variables supported in $[0,1]$ and let $\hat S_n=\frac 1n \sum_{i=1}^n X_i$. Let $n$ be the stopping time that $n=\min_n\{n\mid \sum_{i=1}^n X_i \ge 64 \log (4 n_{\max})/\delta\}$. Suppose that $n\le n_{\max}$, then w.p. $1-\delta$, $\frac 12 \hat S_n\le \E[X]\le \frac 32 \hat S_n$.
\end{lemma}

\subsection{EXP3 Regret Guarantee}\label{sec:exp3}
The following lemma is a classical result for the EXP3 algorithm. For completeness, we also include the proof by \citet[Lemma C.1]{dai2023refined} here.
\begin{lemma}\label{lem:exp3}
Let $x_0,x_1,x_2,\ldots,x_T\in \mathbb R^A$ be defined as
\begin{equation*}
x_{t+1,i}=\left .\left (x_{t,i}\exp(-\eta c_{t,i})\right )\middle /\left (\sum_{i'=1}^A x_{t,i'}\exp(-\eta c_{t,i'})\right )\right .,\quad \forall 0\le t<T,
\end{equation*}
where $c_t\in \mathbb R^A$ is the loss corresponding to the $t$-th iteration. Suppose that $\eta c_{t,i}\ge -1$ for all $t\in [T]$ and $i\in [A]$. Then
\begin{equation*}
\sum_{t=1}^T \langle x_t-y,c_t\rangle\le \frac{\log A}{\eta}+\eta \sum_{t=1}^T \sum_{i=1}^A x_{t,i}c_{t,i}^2
\end{equation*}
holds for any distribution $y\in \triangle([A])$ when $x_0=(\frac 1A,\frac 1A,\ldots,\frac 1A)$.
\end{lemma}
\begin{proof}
By linearity, it suffices to prove the inequality for all one-hot $y$'s. Without loss of generality, let $y=\bm{1}_{i^\ast}$ where $i^\ast\in [A]$. Define $C_{t,i}=\sum_{t'=1}^t c_{t',i}$ as the prefix sum of $c_{t,i}$. Let
\begin{equation*}
    \Phi_t=\frac 1\eta \ln \left (\sum_{i=1}^A \exp\left (-\eta C_{t,i}\right )\right ),
\end{equation*}

then by definition of $x_t$, we have
\begin{align*}
    \Phi_t-\Phi_{t-1}&=\frac 1\eta \ln \left (\frac{\sum_{i=1}^A \exp(-\eta C_{t,i})}{\sum_{i=1}^A \exp(-\eta C_{t-1,i})}\right )
    =\frac 1\eta \ln \left (\sum_{i=1}^A x_{t,i} \exp(-\eta c_{t,i})\right )\\
    &\overset{(a)}{\le} \frac 1\eta \ln \left (\sum_{i=1}^A x_{t,i} (1-\eta c_{t,i}+\eta^2 c_{t,i}^2)\right )
    =\frac 1\eta \ln \left (1-\eta \langle x_t,c_t\rangle+\eta^2 \sum_{i=1}^A x_{t,i}c_{t,i}^2\right )\\
    &\overset{(b)}{\le}-\langle x_t,c_t\rangle+\eta \sum_{i=1}^A x_{t,i}c_{t,i}^2,
\end{align*}

where (a) used $\exp(-x)\le 1-x+x^2$ for all $x\ge -1$ and (b) used $\ln(1+x)\le x$ (again for all $x\ge -1$). Therefore, summing over $t=1,2,\ldots,T$ gives
\begin{align*}
    \sum_{t=1}^T \langle x_t,c_t\rangle&\le \Phi_0-\Phi_T+\eta \sum_{t=1}^T \sum_{i=1}^A x_{t,i}c_{t,i}^2\\
    &\le \frac{\ln N}{\eta}-\frac 1\eta \ln \left (\exp(-\eta C_{T,i^\ast})\right )+\eta \sum_{t=1}^T \sum_{i=1}^N p_t(i)\ell_t^2(i)\\
    &\le \frac{\ln A}{\eta}+L_T(i^\ast)+\eta \sum_{t=1}^T \sum_{i=1}^A x_{t,i}c_{t,i}^2.
\end{align*}

Moving $C_{t,i^\ast}$ to the LHS then shows the inequality for $y=\bm{1}_{i^\ast}$. The result then extends to all $y\in \triangle([A])$ by linearity.
\end{proof}

\end{document}